\documentclass{article}

\usepackage[preprint]{neurips_2026}

\usepackage[utf8]{inputenc} 
\usepackage[T1]{fontenc}    
\usepackage{hyperref}       
\usepackage{url}            
\usepackage{booktabs}       
\usepackage{amsfonts}       
\usepackage{nicefrac}       
\usepackage{microtype}      
\usepackage{xcolor}         

\usepackage{algorithm}
\usepackage{algorithmic}
\usepackage{amsmath}
\usepackage{amssymb}
\usepackage{mathtools,multicol}
\usepackage{amsthm}
\usepackage{booktabs}
\usepackage{multirow}
\usepackage{enumitem}
\usepackage{makecell}
\usepackage[table]{xcolor}
\usepackage[most]{tcolorbox}
\usepackage{adjustbox}
\usepackage{graphicx}
\usepackage{xspace}
\usepackage{wrapfig}

\newtheorem{proposition}{Proposition}
\newtheorem{remark}{Remark}

\newcommand{\method}{{\textbf{MPU}}\@\xspace}
\newcommand{\ie}{i.e.,\xspace}
\newcommand{\eg}{e.g.,\xspace}

 \definecolor{mpucolor1}{HTML}
 {BA684F}
  \definecolor{mpucolor1b}{HTML}
 {000000}
\definecolor{mpucolor2}{HTML}{C76F54}
\definecolor{mpucolor3}{HTML}{FBE9C6}
\definecolor{mpucolor4}{HTML}{F5DEB2}
\definecolor{mpucolor5}{HTML}{FCF7F6}
\definecolor{mpucolor1_bak}{HTML}{BA684F}

\newtcolorbox{pmt}[2][]{
  colbacktitle = mpucolor1_bak,
  colframe = mpucolor1_bak,
  colback = mpucolor5,
  coltitle = white,
  fonttitle=\bfseries,
  title={Prompt: #2},
  #1 
}

\title{\method: Towards Secure and Privacy-Preserving\\Knowledge Unlearning for Large Language Models}

\author{%
    \textbf{Tiantong Wang\textsuperscript{1,2}}\quad
    \textbf{Xinyu Yan\textsuperscript{1,2}}\quad
    \textbf{Tiantong Wu\textsuperscript{1,2,\thanks{Corresponding author.}}}\\
    \textbf{Yurong Hao\textsuperscript{1}}\quad
    \textbf{Pengjun Xie\textsuperscript{3}}\quad
    \textbf{Wei Yang Bryan Lim\textsuperscript{1}}
    \\
    \textsuperscript{1}College of Computing and Data Science, Nanyang Technological University \\
    \textsuperscript{2}Alibaba-NTU Global e-Sustainability CorpLab (ANGEL) \\
    \textsuperscript{3}Tongyi Lab, Alibaba Group
}

\begin{document}

\maketitle

\begin{abstract}
    Machine unlearning for large language models often faces a privacy dilemma in which strict constraints prohibit sharing either the server's parameters or the client's forget set.
    To address this dual non-disclosure constraint, we propose \textcolor{mpucolor1b}{\method}, an algorithm-agnostic privacy-preserving \textcolor{mpucolor1b}{\textbf{\underline{M}}}ultiple \textcolor{mpucolor1b}{\textbf{\underline{P}}}erturbed Copies \textcolor{mpucolor1b}{\textbf{\underline{U}}}nlearning framework that primarily introduces two server-side modules: \textbf{Pre-Process} for randomized copy generation and \textbf{Post-Process} for update aggregation.
    In Pre-Process, the server distributes multiple perturbed and reparameterized model instances, allowing the client to execute unlearning locally on its private forget set without accessing the server's exact original parameters. 
    After local unlearning, the server performs Post-Process by inverting the reparameterization and aggregating updates with a harmonic denoising procedure to alleviate the impact of perturbation. Experiments with seven unlearning algorithms show that \method achieves comparable unlearning performance to noise-free baselines, with most algorithms' average degradation well below 1\% up to 10\% noise, and can even outperform the noise-free baseline for some algorithms under 1\% noise. Code is available at \url{https://github.com/Tristan0318/MPU}.
\end{abstract}

\section{Introduction}

As large language models (LLMs) continue to advance, their tendency to memorize and reproduce training data has raised serious concerns about privacy, safety, and intellectual property. These concerns motivate an urgent need for \textbf{machine unlearning}~\citep{cao2015towards}, whose goal is to selectively remove undesired data, knowledge, or behaviors from a trained model while preserving its general utility for normal tasks. Unlearning is particularly challenging for modern LLMs, as full retraining is prohibitively expensive, and deletion requests may arrive continuously over the model’s lifecycle. As a result, a growing body of research has investigated approaches ranging from training-time strategies such as sharding or slicing~\citep{bourtoule2021machine}, post-hoc model editing and selective forgetting~\citep{golatkar2020eternal}, influence-based approximations~\citep{koh2017understanding}, and foundational formulations of deletion guarantees~\citep{ginart2019making}.

However, many real-world deployments impose an additional constraint that is often overlooked: the data to be forgotten may belong confidentially to a client and must remain local, while the deployed model is proprietary to the server. 
This creates a central tension in \textbf{server-client unlearning}: the server requires an update that removes the effect of a client-local forget set, but \textit{\textbf{(i)}} the client should not disclose raw data (or fine-grained sufficient statistics) to the server, and \textit{\textbf{(ii)}} the server may prefer not to reveal its exact current model parameters to the client. 
Consequently, this setting calls for a restricted server-client unlearning framework that enables effective forgetting without direct data sharing, and without exposing the server’s exact parameters.

Existing unlearning approaches do not directly address this challenge.
Training-time methods based on SISA-style sharding or slicing can reduce the deletion cost by retraining only affected subsets, but require maintaining specific training structures and retaining per-shard state to support subsequent retraining upon deletion requests~\citep{bourtoule2021machine}. 
Post-hoc techniques such as selective forgetting can lower the cost of removing particular classes or examples, but typically assume that the entity performing unlearning has direct access to the model and relevant data distributions~\citep{golatkar2020eternal}.
Influence-function-based approximations offer a principled lens on example influence. However, they can be computationally demanding for LLMs and often rely on second-order information that is difficult to obtain robustly at scale~\citep{koh2017understanding}. 
In federated settings, the right to be forgotten has been explored through reconstructing an unlearned model using server-side training histories~\citep{liu2020federated} or by coordinating efficient retraining while keeping data local~\citep{liu2022right}. 
Recent work further highlights that federated unlearning methods can exhibit substantial trade-offs between effectiveness and efficiency across scenarios~\citep{zhang2025oblivionis}.
Overall, many existing solutions rely on substantial server-side state or centralized access to training records, exposing the server's exact current model to clients.

In this paper, we propose \textcolor{mpucolor1b}\method, a privacy-preserving \textcolor{mpucolor1b}{\textbf{\underline{M}}}ultiple \textcolor{mpucolor1b}{\textbf{\underline{P}}}erturbed Copies \textcolor{mpucolor1b}{\textbf{\underline{U}}}nlearning framework tailored to server--client deployments.
Our key idea is to let the server publish perturbed model instances to clients, with the perturbation designed to be self-canceling during server-side aggregation.
Specifically, at each communication round, instead of broadcasting the exact model parameters, the server releases $m\ge2$ copies that are \textit{\textbf{(i)}} perturbed by structured noise and \textit{\textbf{(ii)}} transformed by an invertible, data-independent, function-preserving reparameterization sampled from parameter symmetries. Starting from each published copy, the client runs a local unlearning routine on its private forget set and returns copy-wise updates.
The server then inverts the reparameterizations and aggregates the returned updates using harmonic weights, which cancel the first-order error term introduced by noise. 
As a result, \method yields a server-side update that matches the noise-free unlearning step, while keeping forget set local and obscuring the server's exact model parameters by communicating only perturbed, symmetry-transformed copies. The key contributions are:
\begin{itemize}
    \item \textbf{Dual Non-Disclosure Unlearning Framework.}
    We propose a server-client parameter-unlearning framework where the client keeps the forget set local (sharing neither raw data nor fine-grained sufficient statistics/distribution), while the server avoids disclosing its exact current parameters by communicating perturbed model copies. To our knowledge, this is the first solution to the dual non-disclosure setting without relying on auxiliary statistics, such as surrogate data.

    \item \textbf{Invertible and Secure Function-Preserving Reparameterizations.}
    We generalize invertible, data-independent, function-preserving reparameterizations to modern Transformer architectures, including RoPE-style positional mechanisms, enabling symmetry-based reparameterizations for LLMs (\eg Meta’s Llama family of models). We furthermore prove the security level in theory that breaking the reparameterization key is NP-hard.

    \item \textbf{Theoretical Guarantees for First-Order Noise Cancellation.}
    We provide theoretical guarantees that, under our structured noise injection and harmonic aggregation, the first-order error induced by noise is eliminated after aggregation, resulting in a server update that is consistent with the noise-free unlearning step.

\end{itemize}

\section{Related Work}
\label{sec:related}

\paragraph{LLM Unlearning.}
LLMs acquire vast amounts of knowledge during pre-training, which may include sensitive or otherwise undesired information~\citep{qiu2025survey}. Machine unlearning aims to enable models to ``forget'' specific pieces of knowledge while maintaining performance on remaining data. Recent work has therefore focused on selective unlearning, \ie suppressing undesired outputs for a designated forget set. Text-based strategies include Gradient Ascent fine-tuning that maximizes cross-entropy loss on forget samples~\citep{jang2023knowledge,yao2024large}, preference-inspired objectives (\eg NPO and SimNPO) that constrain updates with a reference model and length-normalized rewards~\citep{zhang2024negative,fan2024simplicity}, and substitute-response training that learns safe alternative answers to forget queries~\citep{maini2024tofu, mekala2025alternate}.

Beyond text-level objectives, distribution-based approaches drive the model's output distribution toward a target distribution aligned with unlearning goals~\citep{liu2025rethinking, wang2025balancing}. Meanwhile, activation-based methods intervene on internal representations rather than only on outputs, for example, by perturbing hidden states for harmful inputs toward random or refusal-like directions~\citep{shen2025lunar}. Rather than weighting multiple losses, recent works such as NGDiff and MolLM~\citep{jin2025unlearning, pan2025multi} formulate the combination as a multi-task problem, normalizing gradients and computing common descent directions to better trade off forgetting target knowledge against retaining overall utility.
However, many existing LLM unlearning algorithms are studied in settings where the unlearning party can access the model parameters and optimize using the forget set. In this work, we focus on a more constrained server-client interaction rule, in which the forget set remains client-local, and the server does not expose its exact parameters.

\paragraph{Direct Model Merging.}
Our framework aggregates multi-copy updates computed from multiple models, which relates to direct model merging and the empirical linearity of parameter updates in the fine-tuning paradigm. 
In this paradigm, models adapted to different tasks from a shared pretrained checkpoint often exhibit cross-task linearity, whereby their weights and feature spaces can be combined through linear interpolation, enabling direct model merging without additional retraining and substantially reducing computational overhead~\citep{zhou2024emergence}. A widely used strategy is weight averaging, where parameters of similarly initialized fine-tuned models are averaged. Model Soups~\citep{wortsman2022model} further demonstrated that such averaging can improve accuracy and out-of-distribution robustness compared to individual models. Alternatively, task arithmetic~\citep{ilharco2022editing} 
operates on task vectors to compose or edit specific model capabilities. 

Recently, \citet{zhou2024metagpt} extended task arithmetic to LLMs, formulating it as an optimization problem that exploits local linearity and near-orthogonality of task updates. \citet{wang2026m} quantified the merging capability of multiple models by output discrepancy, and theoretically proved the plausibility of merging task-vectors given the same pre-trained model. The above works motivate our design of aggregating different perturbed copies by weighted merging, to obtain efficient privacy-aware unlearning.

\section{Proposed Method:~\method}
\label{sec:method}

\begin{figure}[!htp]
    \centering
    \includegraphics[width=\linewidth]{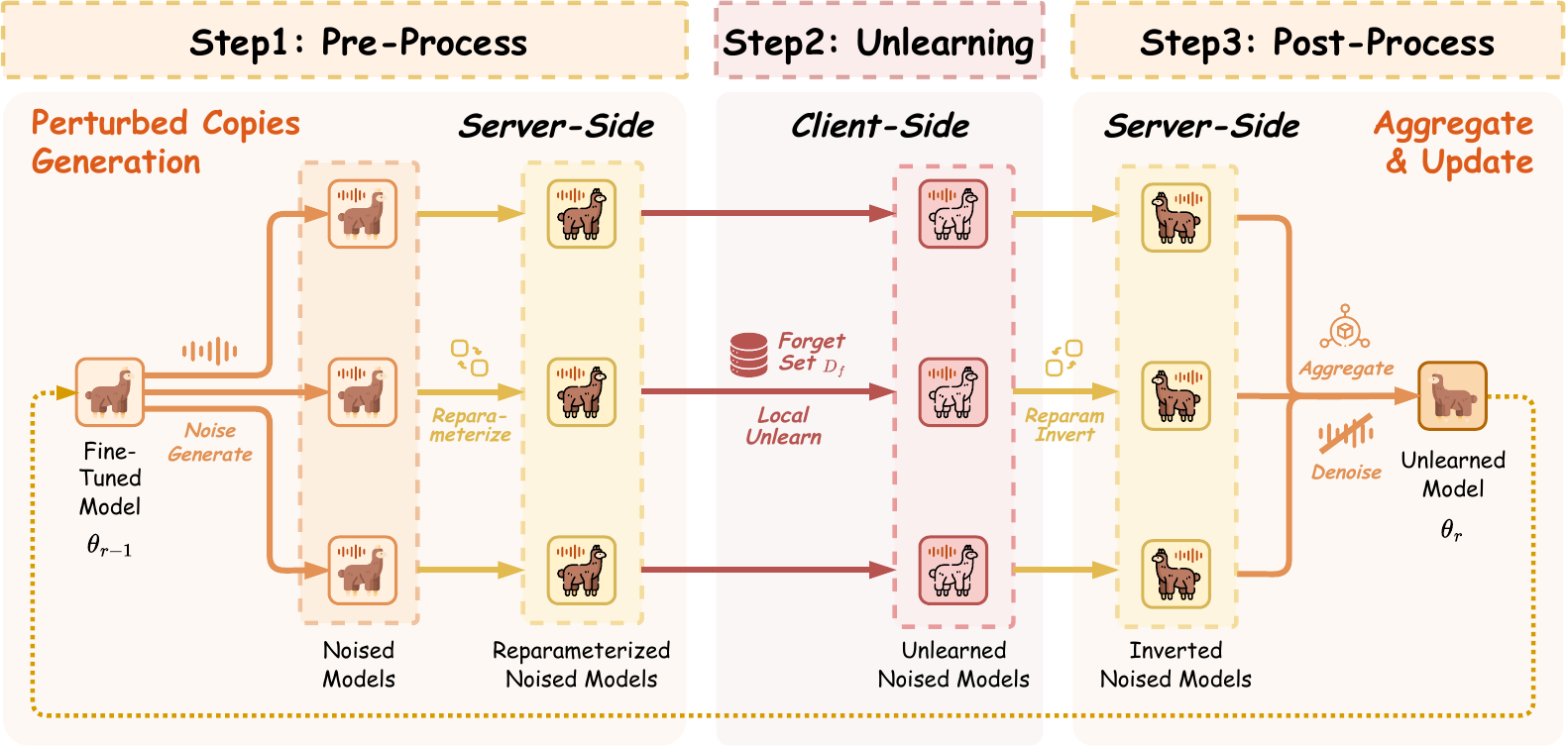}
    \caption{Overview of the proposed \method framework across communication rounds. The server generates perturbed, reparameterized model copies from $\theta_{r-1}$, clients unlearn on $\mathcal{D}_f$, and the server inverts the reparameterization and aggregates updates to obtain $\theta_{r}$.}
    \label{fig:liuchengtu}
    \vspace{-2.5mm}
\end{figure}

\subsection{Algorithm Overview}
\label{sec:pum-setting}

We consider a server-client unlearning framework operating over $R$ communication rounds. Each round consists of three sequential stages. First, the server generates and distributes $m$ perturbed copies of the current global model to the client. Second, the client performs local unlearning on these perturbed models using its private dataset. Third, the server collects the resulting local unlearning updates and aggregates them via a harmonic denoising mechanism, producing the updated global model for the next round.
\begin{remark}
    In Algorithm~\ref{alg:pum}, only the results of the following two operations are public
    \begin{equation}
        \theta^{(k,r)}_{\mathrm{pub}} \gets T_{k,r}\!\big(\theta_{r-1}+\epsilon_k^{(r)}\big)
    \end{equation}
    \begin{equation}
        \Delta^{(k,r)} \gets \textsc{Unlearn}\!\big(\theta^{(k,r)}_{\mathrm{pub}},\mathcal{D}_f\big)
    \end{equation}
    
    and all other operations are kept private to the server.
\end{remark}

\begin{wrapfigure}{R}{0.6\textwidth}
\begin{minipage}{0.6\textwidth}
\vspace{-0.85cm}
\begin{algorithm}[H]
\caption{\textcolor{mpucolor1b}\method}
\label{alg:pum}
\begin{algorithmic}[1]
\REQUIRE $\theta_0$, $R$, $L$,  $m\ge2$, $\{\sigma_\ell\}_{\ell=1}^{L}$, $\{\alpha_k>0\}_{k=1}^m$, $\eta_{\mathrm{srv}}$, $\{s_r\}_{r=1}^R$, $\{t_r\}_{r=1}^R$
\FOR{$r \in [R]$}
     \FOR {all $k\in[m]$ and  $\ell \in [L]$}
        \STATE Generate $\{\epsilon^{0,(r)}_{k,\ell}\}$
         using Eq.~\eqref{eq:pum-noise-construct}
     \ENDFOR
  \FOR {$k \in [m]$}
    \STATE $\epsilon_{k,\ell}^{(r)} \gets \alpha_k\,\epsilon_{k,\ell}^{0,(r)}$
    \STATE $\epsilon_k^{(r)}:= \text{stack}_\ell\big(\epsilon^{(r)}_{k,\ell}\big)$
  \ENDFOR
  \STATE $S_0\gets 0$;\; $S_1\gets \mathbf{0}\in\mathbb{R}^{d}$
  \FOR{$k \in [m]$}
    \STATE $T_{k,r}\gets \textsc{SampleReparam}(t_r,k)$
    \STATE Client Receives: $\theta^{(k,r)}_{\mathrm{pub}} \gets T_{k,r}\!\big(\theta_{r-1}+\epsilon_k^{(r)}\big)$
    \STATE Client Updates: $\Delta^{(k,r)} \gets \textsc{Unlearn}\!\big(\theta^{(k,r)}_{\mathrm{pub}},\mathcal{D}_f\big)$
    \STATE $\widehat{\Delta}^{(k,r)} \gets T_{k,r}^{-1}\!\big(\Delta^{(k,r)}\big)$
    \STATE
    $S_0\gets S_0+\alpha_k^{-1}$;\quad $S_1\gets S_1+\alpha_k^{-1}\,\widehat{\Delta}^{(k,r)}$
  \ENDFOR
  \STATE $\bar\Delta^{(r)} \gets S_1 / S_0$
  \STATE $\theta_r \gets \theta_{r-1} + \eta_{\mathrm{srv}}\,\bar\Delta^{(r)}$
\ENDFOR
\STATE Return $\theta_R$
\end{algorithmic}
\end{algorithm}
\vspace{-1.5cm}
\end{minipage}
\end{wrapfigure}

\subsection{Pre-Process: Perturbed Copies Generation}
\label{sec:pum-perturb}

In each round, the server perturbs the current model before transmitting it to the client. The generation of perturbed copies consists of two components: \textit{\textbf{(i)}} structured noise injection and \textit{\textbf{(ii)}} an invertible, function-preserving reparameterization. Noise injection mitigates privacy leakage risks during inference. Meanwhile, the reparameterization shifts the original parameter space, preventing the client from reconstructing the original parameters even when multiple perturbed copies are accessed.

\subsubsection{Privacy Guarantee: Noise Generation}
\label{sec:pum-noise}

Noise injection is a common and effective defense against inference attacks. In \method, noise is generated independently per block $\ell$ with scale $\sigma_\ell$. We set $\sigma_\ell$ based on a reference \textbf{task vector} $v_\ell$, defined as the difference between the current model parameters and a public reference model parameter. Typically, the public reference model refers to a released pretrained model before fine-tuning. This design is motivated by the fact that the task vector captures the parameter update induced by fine-tuning, which encodes sensitive information about the underlying private data, which requires protection. We define $\sigma_{\ell}$ as:
\begin{equation}
\sigma_\ell=\kappa\cdot \mathrm{RMS}(v_\ell),
\quad
\mathrm{RMS}(v_\ell):=
\sqrt{\frac{1}{d_\ell}\sum_{t=1}^{d_\ell} v_\ell(t)^2},
\end{equation}
where $\kappa>0$ is a controllable noise-level hyperparameter.

For each round $r \in [R]$ and block $\ell$, the server draws i.i.d.\ Gaussian vectors
$z_{k,\ell}^{(r)} \sim \mathcal{N}\!\left(0, \sigma_\ell^2 I_{d_\ell}\right)$
for $k \in [m]$ from seeds $s_r$ and $\ell$, subtracts their mean, and rescales the resulting vectors so that each perturbed copy maintains the prescribed marginal variance. The resulting base noise vectors satisfy a block-wise zero-sum constraint:
\begin{equation}
\epsilon^{0,(r)}_{k,\ell}
=
\sqrt{\frac{m}{m-1}}\,
\Big(z_{k,\ell}^{(r)}-\bar z_\ell^{(r)}\Big)
\Rightarrow
\sum_{k=1}^m \epsilon^{0,(r)}_{k,\ell}\equiv 0,
\label{eq:pum-noise-construct}
\end{equation}
where $\bar z_\ell^{(r)}=\frac{1}{m}\sum_{k=1}^m z_{k,\ell}^{(r)}$. Intuitively, the zero-sum structure forces the $m$ copy noises to lie in an $(m-1)$-dimensional subspace, which is the key algebraic property enabling noise cancellation during the server's aggregation.

We then apply a per-copy positive scaling $\alpha_k>0$:
\begin{equation}
\epsilon^{(r)}_{k,\ell} = \alpha_k\,\epsilon^{0,(r)}_{k,\ell},
\quad
\epsilon_k^{(r)} := \text{stack}_\ell\big(\epsilon^{(r)}_{k,\ell}\big)\in\mathbb{R}^d.
\end{equation}
where the stack operator combines the layer-level noises $\epsilon^{(r)}_{k,\ell}$ together to a model-level noise $\epsilon_k^{(r)}$.
The $\alpha_k$'s introduce heterogeneous noise magnitudes across copies, work as a secondary protection, avoiding parameter reconstruction even if reparameterization information is leaked.
Finally, we add the noise $\epsilon_k^{(r)}$ to the original model parameters to obtain $m$ first-stage perturbed models prior to reparameterization.

\paragraph{Formal Privacy Guarantee by Differential Privacy}
In standard noise-injection formal privacy mechanisms such as Differential Privacy (DP)~\citep{dwork2006calibrating}, one typically specifies a privacy budget first (\eg $(\varepsilon,\delta)$), and then computes the noise level according to a chosen DP mechanism and its accounting method.
In practice, the resulting noise scale can vary across different DP mechanisms, and is influenced by hyperparameters used in DP mechanisms (i.e., sensitivity, clipping bound in Gaussian Mechanism~\citep{dwork2006ourgaussian}, search range in Renyi-DP Mechanism~\citep{Mironov2017}). For consistency and avoiding additional hyperparameters' influence, we therefore use a fixed noise level directly in \method. See Appendix~\ref{app:mpu_dp_combination} for a detailed discussion on possible DP combinations to achieve a formal privacy guarantee.

\subsubsection{Security Guarantee: Reparameterization}
\label{sec:pum-reparam}

To prevent client-side reconstruction of the original model from multiple copies when scalars $\{\alpha_k\}_{k=1}^{m}$ are leaked, since a small number of scalar keys alone provides insufficient security, the server applies a reparameterization to each block prior to releasing the perturbed models. We extend previous work on neural network functional invariance~\citep{kuurkova1994functionally} to derive functionally invariant reparameterization on modern Transformer model architectures.

Let $\Theta$ denote the parameter tuple of an attention or feed-forward network (FFN) block, and let
$f_{\Theta} : \mathbb{R}^{d_{\mathrm{model}}} \rightarrow \mathbb{R}^{d_{\mathrm{model}}}$ denote the forward function induced by $\Theta$.
We consider a family of \emph{reparameterizations} acting on the parameter space, defined as invertible mappings
$T : \Theta \mapsto T(\Theta)$ equipped with explicit inverses $T^{-1}$.

For each copy $k$ and round $r$, the server samples a distinct data-independent reparameterization
\begin{equation}
T_{k,r} \gets \textsc{SampleReparam}(t_r, k),
\end{equation}
where $T_{k,r}$ is dependent on solely seed $t_r$ and $k$, which is \emph{function-preserving}, i.e.,
\begin{equation}
\label{eq:func-preserve}
f_{T_{k,r}(\Theta)}(x) = f_{\Theta}(x), \quad \forall x \in \mathbb{R}^{d_{\mathrm{model}}}.
\end{equation}
Equivalently, letting $\mathcal{G}$ denote the set of all such invertible function-preserving maps, $\mathcal{G}$ forms a \emph{parameter-symmetry group} under composition. Each $T_{k,r} \in \mathcal{G}$ acts on $\Theta$ while leaving the realized function invariant.(See Eq.~\ref{eq:mpu-attn-forward},~\ref{eq:ffn_reparam} in Appendix for the detailed operation rule of $T$)

\paragraph{Feed-Forward Network Reparameterization} For feed-forward network reparameterization, we adopt the conventional MLP channel permutation reparameterization. We give the formal reparameterization steps in Appendix~\ref {app:mpu-ffn}

\paragraph{Attention Reparameterization}
More generally, Attention blocks admit a rich parameter-symmetry group that includes continuous attention symmetries that leave the attention computation invariant.
We derive detailed reparameterization constructions and invariance proofs for the full Transformer architecture in Appendix~\ref{app:pum-reparam}.

Specifically, for RoPE-based~\citep{Su2021RoPE} models such as Llama-3~\citep{grattafiori2024llama} series models, we restrict attention transformations to those that commute with the RoPE operators, ensuring that the reparameterization preserves functional equivalence for both RoPE and non-RoPE network structures. Concrete RoPE-aware reparameterization is presented in Appendix.~\ref{app:mpu-rope}.

\paragraph{Optimization Trajectory Invariance}
The reparameterization is not only function-preserving, but also induces an invariant optimization (learning) trajectory.
Specifically, suppose $\mathcal{A}$ is a deterministic learning algorithm (\eg stochastic gradient descent with fixed mini-batch ordering). Then the learning dynamics satisfy Eq.~\ref{eq: traj_inv}:
\begin{equation}
\label{eq: traj_inv}
\mathcal{A}(\theta) = T^{-1}\!\left(\mathcal{A}(T(\theta))\right),
\end{equation}
indicating that optimization commutes with the reparameterization, so the local unlearning process will not be affected by reparameterization.
Meanwhile, we also provide the detailed analysis of trajectory invariance in Appendix~\ref{app:mpu-reparam-trajectory}.

\paragraph{Theoretical Security Guarantee of Reparameterization}
We formally show that the alignment task for breaking the reparameterization key is mathematically equivalent to the Multi-Dimensional Assignment Problem (MDAP) for discrete Feed-Forward Network permutations, and the Generalized Orthogonal Procrustes Problem (GOPP) for continuous Attention symmetries. Both alignment problems are computationally prohibitive: they require $\mathcal{O}(d^3)$ operations for $m=2$ and are strictly \textbf{NP-hard} for $m \ge 3$. Consequently, this spatial misalignment mathematically prevents the client from discovering cross-copy correlations, effectively blocking reconstruction of the server's parameters. We provide formal theoretical analysis in Appendix~\ref{app:security_hardness}.

\subsection{Client-Side Local Unlearning}
Upon receiving the reparameterized model parameters, the client performs local unlearning using its private data, employing standard unlearning algorithms such as GradAscent, NPO, and DPO. 
Notably, \method is algorithm-agnostic and can be integrated with any parameter-based unlearning methods.

\subsection{Performance Guarantee: Update Aggregation}
\label{sec:pum-aggregation}

After the client returns the local unlearning update $\Delta^{(k,r)}$, the server maps it back to the original parameter coordinates using the explicit inverse reparameterization:
\begin{equation}
\widehat{\Delta}^{(k,r)} \;\gets\; T_{k,r}^{-1}\!\big(\Delta^{(k,r)}\big),
\end{equation}
so that all copy-wise updates are expressed in a common parameterization prior to aggregation. 
Note that $T_{k,r}^{-1}$ is efficiently computable by transpose due to the property of orthogonal reparameterization matrices.

After inversion, the server aggregates the $m$ returned updates using harmonic aggregation to cancel the first-order noise:
\begin{equation}
\bar{\Delta}^{(r)}
=
\frac{\sum_{k=1}^m \alpha_k^{-1}\,\widehat{\Delta}^{(k,r)}}
     {\sum_{k=1}^m \alpha_k^{-1}}.
\end{equation}
Here, we briefly explain why the aggregation cancels the first-order noise. Let $\Delta^\star(\theta)$ denote the ideal (noise-free) unlearning displacement, and let
$J$ be the Jacobian of $\Delta^\star$ evaluated at $\theta_{r-1}$. 
Under a local linearization assumption, the inverted update first-order approximately satisfies
\begin{equation}
\widehat{\Delta}^{(k,r)} 
\;\approx\; 
\Delta^\star(\theta_{r-1}) + J\,\epsilon_k^{(r)}
=
\Delta^\star(\theta_{r-1}) + J\,\alpha_k\,\epsilon_k^{0,(r)},
\end{equation}
where $\epsilon_k^{0,(r)}$ is the stacked noise before scaling.
Substituting this expression into the harmonic average, the injected term becomes
\begin{equation}
\frac{\sum_{k=1}^m \alpha_k^{-1}\,J\,\alpha_k\,\epsilon_k^{0,(r)}}{\sum_{k=1}^m\alpha_k^{-1}}
=
\frac{J \sum_{k=1}^m \epsilon_k^{0,(r)}}{\sum_{k=1}^m\alpha_k^{-1}}
=
0,
\end{equation}
where the last equality follows from the block-wise zero-sum property in
Eq.~\ref{eq:pum-noise-construct}. 
Therefore, the aggregation eliminates the correlated first-order noise error, while requiring only the scalar coefficients $\{\alpha_k\}$ rather than storing the full noise parameters.
A more detailed error analysis for higher orders is provided in Appendix~\ref{app:mpu-aggregation}.

At the end of each round, the server applies the aggregated update with step size $\eta_{\mathrm{srv}}$:
\begin{equation}
\theta_r \gets \theta_{r-1} + \eta_{\mathrm{srv}}\,\bar{\Delta}^{(r)}.
\end{equation}
\subsection{Memory-Efficiency}
\label{sec:pum-streaming}

Although Algorithm~\ref{alg:pum} conceptually adopts $m$ published copies per round, the server and client do not need to store all $m$ perturbed models (nor the $m$ returned updates) in memory. This observation follows from the fact that the harmonic
aggregation coefficient depends solely on $\alpha_k$'s.
Consequently, the sufficient statistics for each round reduce to the two accumulators:
\begin{equation}
S_0 \;=\; \sum_{k=1}^m \alpha_k^{-1},
\quad
S_1 \;=\; \sum_{k=1}^m \alpha_k^{-1}\,\widehat{\Delta}^{(k,r)} \in \mathbb{R}^d,
\end{equation}
after which the update is given by $\bar{\Delta}^{(r)} = S_1 / S_0$.

Concretely, the server can implement each round in a streaming manner.
For $k = 1, \dots, m$, the server publishes only a single perturbed copy
$\theta_{\mathrm{pub}}^{(k,r)} = T_{k,r}(\theta_{r-1} + \epsilon_k^{(r)})$, receives the
corresponding client update $\Delta^{(k,r)}$, inverts it as
$\widehat{\Delta}^{(k,r)} = T_{k,r}^{-1}(\Delta^{(k,r)})$, and updates the accumulators:
\begin{equation}
S_0 \leftarrow S_0 + \alpha_k^{-1},
\quad
S_1 \leftarrow S_1 + \alpha_k^{-1}\,\widehat{\Delta}^{(k,r)}.
\end{equation}
At no point does the server need to store all $m$ models or all $m$ updates. The peak server-side memory footprint (beyond the base parameters $\theta_{r-1}$) is therefore dominated by a single $d$-dimensional accumulator $S_1$ and the currently processed copy, yielding an $O(d)$ memory requirement rather than $O(md)$. The same streaming procedure applies on the client side, where the client processes one published copy at a time, avoiding the need to store $m$ models simultaneously.

\subsection{Computational Overhead}
We provide the derivation of computational overhead in Appendix~\ref{app:overhead}. The arithmetic intensity is 
\begin{equation}
    A = \frac{6N + 4d_hN_{attn}}{9Nb}
\end{equation}
which is only 4\% for a modern A100 GPU's arithmetic intensity ridge point.
The server latency is memory-bound as
\begin{equation}
    T_{server} = \frac{9mNb}{C_{mem}}
\end{equation}
which takes around 0.13 seconds for an A100 GPU. Thus, the computational overhead is negligible.
\section{Experiments}
\label{sec:experiments}

\subsection{Experimental Setup}
To evaluate the effectiveness of the proposed \method framework when coupled with diverse unlearning algorithms, we design and conduct a comprehensive set of experiments.

\subsubsection{Models and Benchmark} 
We conduct experiments using five representative base models: \textbf{Llama-3.2-1B-Instruct}, \textbf{Llama-3.2-3B-Instruct}, \textbf{Llama-3.1-8B-Instruct}~\citep{grattafiori2024llama}, \textbf{Qwen2.5-1.5B-Instruct} and \textbf{Qwen2.5-3B-Instruct}~\citep{qwen2024qwen2}. Models are evaluated on the widely adopted \textbf{TOFU}~\citep{maini2024tofu}, and \textbf{MUSE}~\citep{shi2024muse} benchmarks, following established experimental setups in prior works~\citep{wang2024llm, openunlearning2025}, using full-finetuning by default and LoRA~\citep{hu2022lora} for \textbf{Llama-3.1-8B model}. The evaluation metrics are detailed in Appendix~\ref{app:Metrics}.

\subsubsection{Unlearning Algorithms}
To contextualize our results, we benchmark \method{} against representative unlearning objectives spanning distinct algorithmic paradigms, organized by the primary mechanism used to suppress information associated with the forget set (Appendix~\ref{Unlearning}): \textit{\textbf{(i)}} Loss-Reversal, First-Order Unlearning: \textbf{GradAscent}~\citep{jang2023knowledge} and \textbf{GradDiff}~\citep{liu2022continual}.
\textit{\textbf{(ii)}}  Preference-Style, Bounded Objectives: \textbf{DPO}~\citep{rafailov2023direct}, \textbf{NPO}~\citep{zhang2024negative}, and \textbf{SimNPO}~\citep{fan2024simplicity}.
\textit{\textbf{(iii)}} Distribution Shaping via Self-Distillation: \textbf{UnDIAL}~\citep{dong2025undial}. \textit{\textbf{(iv)}} Loss Reweighting for Targeted Forgetting: \textbf{SatImp}~\citep{yang2025exploring}.

\subsubsection{Baselines}

Since \method{} is designed to perform denoising after noisy perturbation, we compare it against two baselines. \textit{\textbf{(i)}} \textsc{\textbf{Clean}}: a noise-free, centralized single-copy unlearning setting, corresponding to the standard unlearning framework with full access to the data and model and without any noise injection. The \textsc{\textbf{Clean}} baseline serves as an approximate upper bound on unlearning performance. \textit{\textbf{(ii)}} \textsc{\textbf{Noised}} (Appendix~\ref{sec:appendix-mpu-vs-noiseonly}): a single-copy baseline in which noise is directly injected before model publication, and the server adopts the client’s returned update without any denoising. \textsc{\textbf{Noised}} provides a lower-bound reference for isolating the effect of denoising.

\subsection{Experimental Results}

\begin{table}[t]
    \caption{
    Performance comparison of different unlearning algorithms using the \textbf{Llama-3.2-1B} model on the TOFU benchmark (Split99). Results are reported under three settings: \textsc{\textbf{Clean}}, a noise-free baseline; \textsc{\textbf{Noised}}, a single-copy noise baseline with the same noise magnitude but without denoising; and \textsc{\textbf{\method}}, using $m{=}2$ copies with noise level $\kappa{=}0.01$. Higher values indicate better performance for Forget Quality, Forget Truth Ratio, and Model Utility, while values of $\text{PrivLeak}$ closer to $0$ are preferred.
    }
    \centering
    \begin{adjustbox}{width=\linewidth}
    \begin{tabular}{ccccccccccccc}
        \toprule
        \multirow{2}[2]{*}{\textbf{\shortstack{Unlearning\\Algorithms}}}
        & \multicolumn{3}{c}{\textbf{Forget Quality} $\uparrow$}
        & \multicolumn{3}{c}{\textbf{Forget Truth Ratio} $\uparrow$}
        & \multicolumn{3}{c}{\textbf{Model Utility} $\uparrow$}
        & \multicolumn{3}{c}{\textbf{$\text{PrivLeak}$}} \\
        \cmidrule(lr){2-4}\cmidrule(lr){5-7}\cmidrule(lr){8-10}\cmidrule(lr){11-13}
        & \textsc{Clean} & \textsc{Noised} & \textsc{MPU}
        & \textsc{Clean} & \textsc{Noised} & \textsc{MPU}
        & \textsc{Clean} & \textsc{Noised} & \textsc{MPU}
        & \textsc{Clean} & \textsc{Noised} & \textsc{MPU} \\
        \midrule
        
        \textsc{GradAscent~\textcolor{gray}{[ACL 2023]}} &6.58e-5&2.81e-8& \colorbox{mpucolor1!15}{\textbf{5.41e-2}} &0.355&0.246& \colorbox{mpucolor1!15}{\textbf{0.468}} &0.000 & 0.000 & \colorbox{mpucolor1!15}{\textbf{2.31e-4}} &65.8& \colorbox{mpucolor1!15}{\textbf{58.9}} & 69.6 \\
        
        \textsc{GradDiff~\textcolor{gray}{[PMLR 2022]}} & \colorbox{mpucolor1!15}{0.405} &0.266& \colorbox{mpucolor1!15}{\textbf{0.405}} &0.535&0.533& \colorbox{mpucolor1!15}{\textbf{0.547}} & 0.461 & 0.461 & \colorbox{mpucolor1!15}{\textbf{0.464}} &77.1& \colorbox{mpucolor1!15}{\textbf{73.3}} & 77.2 \\

        \textsc{DPO~\textcolor{gray}{[NeurIPS 2023]}} &0.165&0.165& \colorbox{mpucolor1!15}{\textbf{0.266}} & 0.637 &0.620& \colorbox{mpucolor1!15}{\textbf{0.641}} & 0.591 & \colorbox{mpucolor1!15}{\textbf{0.595}} & 0.591 &-25.5& \colorbox{mpucolor1!15}{\textbf{-19.8}} & -28.9 \\
        
        \textsc{NPO~\textcolor{gray}{[COLM 2024]}} & \colorbox{mpucolor1!15}{0.919} &0.766& \colorbox{mpucolor1!15}{\textbf{0.919}} &0.624& \colorbox{mpucolor1!15}{\textbf{0.640}} &0.628& 0.599 & \colorbox{mpucolor1!15}{\textbf{0.600}} & 0.597 &30.6&32.9& \colorbox{mpucolor1!15}{\textbf{28.2}} \\
        
        \textsc{SimNPO~\textcolor{gray}{[NeurIPS 2025]}} &5.41e-2 & 5.41e-2& \colorbox{mpucolor1!15}{\textbf{9.71e-2}} & \colorbox{mpucolor1!15}{\textbf{0.526}} &0.522& 0.525 & \colorbox{mpucolor1!15}{0.598} &0.592& \colorbox{mpucolor1!15}{\textbf{0.598}} & \colorbox{mpucolor1!15}{\textbf{-68.4}} &-70.2& -71.8 \\
        
        \textsc{UnDIAL~\textcolor{gray}{[NAACL 2025]}} & \colorbox{mpucolor1!15}{1.43e-2} & \colorbox{mpucolor1!15}{1.43e-2} & \colorbox{mpucolor1!15}{\textbf{1.43e-2}} & \colorbox{mpucolor1!15}{\textbf{0.530}} & 0.527 & 0.529 & 0.613 & 0.614 & \colorbox{mpucolor1!15}{\textbf{0.615}} & \colorbox{mpucolor1!15}{\textbf{-76.4}} &-77.4& -78.0 \\
        
        \textsc{SatImp~\textcolor{gray}{[ICML 2025]}} &3.02e-3& \colorbox{mpucolor1!15}{6.76e-3} & \colorbox{mpucolor1!15}{\textbf{6.76e-3}} &0.474&0.470& \colorbox{mpucolor1!15}{\textbf{0.476}} & 0.600 & 0.597 & \colorbox{mpucolor1!15}{\textbf{0.601}} & \colorbox{mpucolor1!15}{-98.9} & -99.1 & \colorbox{mpucolor1!15}{\textbf{-98.9}} \\
        \bottomrule
    \end{tabular}
    \end{adjustbox}
    \vspace{-2mm}
     \label{tab:pum_main_comparison}
\end{table}

For the complete experiment results and analysis, please refer to Appendix~\ref{Supplementary Experiments} and \ref{app:pum_supp_ablations}. We provide a brief analysis of the comparison of the \textbf{Llama-3.2-1B} model on the TOFU benchmark here. 

\paragraph{Privacy}
Privacy measures whether sensitive data in the forget set can still be retrieved from the model. As shown in Table~\ref{tab:pum_main_comparison}, we find that \method{} under a low noise level consistently outperforms single-copy noised/noise-free unlearning in Forget Quality (FQ). Specifically, for unlearning algorithms with high FQ (\textbf{GradDiff}, \textbf{NPO}), \method{} matches the FQ of noise-free unlearning while substantially outperforming single-copy noisy unlearning, with $0.405$ vs.\ $0.266$ for GradDiff, and $0.919$ vs.\ $0.766$ for NPO. For other low-FQ unlearning algorithms (\textbf{GradAscent}, \textbf{SimNPO}, and \textbf{DPO}), we observe that the single-copy noisy and noise-free frameworks yield similar FQ scores, whereas \method{} improves over both baselines significantly: $0.054$ vs.\ near zero for GradAscent, $0.097$ vs.\ $0.054$ for SimNPO, and $0.266$ vs.\ $0.165$ for DPO. The increase in FQ relative to noise-free unlearning can be attributed to the multi-copy stability effect. For more details, please refer to Appendix~\ref{sec:mpu-stability}.

For Forget Truth Ratio (FTR) and PrivLeak (PL), we do not observe significant differences across the compared settings. This suggests that, under these two privacy-related metrics, all unlearning algorithms except \textbf{GradAscent} exhibit stable behavior. Comparing across unlearning algorithms, \textbf{NPO} performs best overall, achieving the highest average FQ together with the smallest absolute value of PL. Conversely, \textbf{GradAscent} exhibits almost zero FQ and MU under both the noise-free and noised unlearning settings, indicating a complete breakdown of the model.

\paragraph{Utility}
Utility measures model performance on general tasks, reflecting whether general capability is preserved after unlearning. Table~\ref{tab:pum_main_comparison} shows that single-copy noise, noise-free unlearning, and \method{} achieve very similar MU under each unlearning algorithm, with variations below $0.01$. This indicates that utility preservation is not particularly sensitive to noise injection in our setup.

\paragraph{Memorization}
Memorization measures the extent to which the model retains information from the training data. From Figure~\ref{fig:QA}, the Forget QA Probability of \method{} is higher than the single-copy noise-free/no-denoise frameworks, except for the \textbf{SimNPO} algorithm (while only $0.002$ lower compared with the noise-free one). The Forget QA ROUGE results show that \method{} and the noise-free framework attain similar scores, while both outperform the no-denoise framework. Moreover, \method{} achieves the best ROUGE for the unlearning algorithms like \textbf{GradAscent}, \textbf{DPO}, and \textbf{SatImp}, whereas the noise-free framework yields the best ROUGE for \textbf{NPO}, \textbf{SimNPO}, and \textbf{UnDIAL}. Overall, these memorization results suggest that direct noise injection can introduce undesirable memorization artifacts, while \method{} mitigates this effect via harmonic denoising aggregation. Beyond the low-noise comparison in Table~\ref{tab:pum_main_comparison}, we further evaluate \method{} with extensive experiments on the number of published noisy copies $m$, which measures the computational overhead, and noise level $\kappa$.
Overall, \method{} is robust under moderate choices of $m$, and $\kappa$ (Appendix~\ref{app:pum_supp_ablations} and Table~\ref{tab:pum_kappa},~\ref{tab:pum_mval}), indicating that \method is tolerable for large noise levels, and few copies is the best choice for resource consideration.

\begin{figure}[t]
    \centering
    \includegraphics[width=1.0\linewidth]{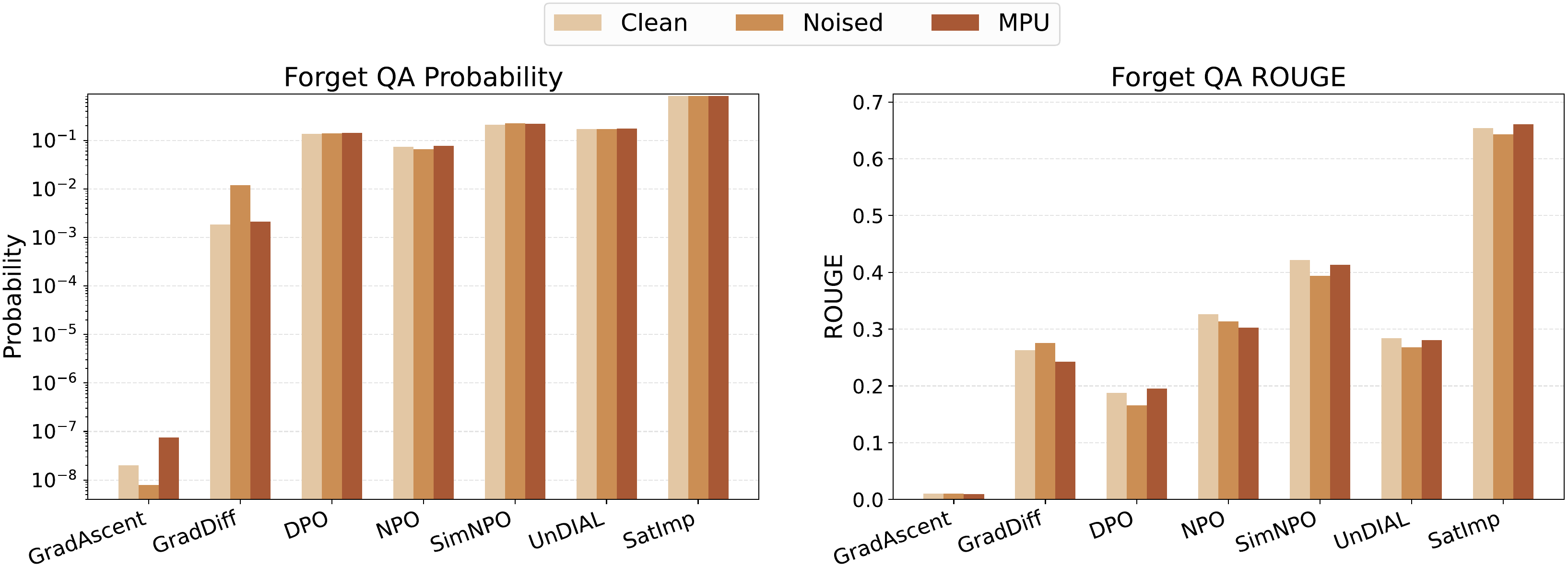}
    \caption{Performance comparison of different unlearning algorithms using the \textbf{Llama-3.2-1B} model. Results are reported under three settings: \textit{\textbf{Clean}}, a noise-free baseline; \textit{\textbf{Noised}}, a single-copy noise baseline with the same noise magnitude but without denoising; and \method, using $m{=}2$ copies with noise level $\kappa{=}0.01$. Higher values indicate better performance for Forget QA Probability and ROUGE.}
    \label{fig:QA}
\end{figure}

\section{Conclusion}
\label{sec:conclusion}
We propose \method{}, a server-client privacy-preserving unlearning framework under a dual non-disclosure constraint: the server does not reveal its exact model parameters, and the client does not share its data. \method{} is the first, to our knowledge, to enable unlearning under this strict setting without additional assumptions (\eg distributional constraints or surrogate models/data), while achieving performance comparable to, and sometimes better than, no-privacy noise-free unlearning. We provide theoretical guarantees that harmonic aggregation eliminates first-order noise error, multi-copy unlearning improves stability, and the reparameterization preserves both functionality and optimization trajectories, while ensuring the breaking of the security key is NP-hard. Empirically, \method{} consistently outperforms privacy baselines and can even surpass the noise-free baseline due to multi-copy stability. Although local computation grows linearly with the copy number, two copies suffice in all our experiments. Future work may design improved communication protocols and further reduce computational cost.

\newpage

{
\small
\bibliography{reference}

@inproceedings{cao2015towards,
  title={Towards making systems forget with machine unlearning},
  author={Cao, Yinzhi and Yang, Junfeng},
  booktitle={2015 IEEE symposium on security and privacy},
  pages={463--480},
  year={2015},
  organization={IEEE}
}

@article{ginart2019making,
  title={Making ai forget you: Data deletion in machine learning},
  author={Ginart, Antonio and Guan, Melody and Valiant, Gregory and Zou, James Y},
  journal={Advances in neural information processing systems},
  volume={32},
  year={2019}
}

@inproceedings{bourtoule2021machine,
  title={Machine unlearning},
  author={Bourtoule, Lucas and Chandrasekaran, Varun and Choquette-Choo, Christopher A and Jia, Hengrui and Travers, Adelin and Zhang, Baiwu and Lie, David and Papernot, Nicolas},
  booktitle={2021 IEEE symposium on security and privacy (SP)},
  pages={141--159},
  year={2021},
  organization={IEEE}
}

@inproceedings{koh2017understanding,
  title={Understanding black-box predictions via influence functions},
  author={Koh, Pang Wei and Liang, Percy},
  booktitle={International conference on machine learning},
  pages={1885--1894},
  year={2017},
  organization={PMLR}
}

@inproceedings{golatkar2020eternal,
  title={Eternal sunshine of the spotless net: Selective forgetting in deep networks},
  author={Golatkar, Aditya and Achille, Alessandro and Soatto, Stefano},
  booktitle={Proceedings of the IEEE/CVF conference on computer vision and pattern recognition},
  pages={9304--9312},
  year={2020}
}

@article{liu2020federated,
  title={Federated unlearning},
  author={Liu, Gaoyang and Ma, Xiaoqiang and Yang, Yang and Wang, Chen and Liu, Jiangchuan},
  journal={arXiv preprint arXiv:2012.13891},
  year={2020}
}

@inproceedings{liu2022right,
  title={The right to be forgotten in federated learning: An efficient realization with rapid retraining},
  author={Liu, Yi and Xu, Lei and Yuan, Xingliang and Wang, Cong and Li, Bo},
  booktitle={IEEE INFOCOM 2022-IEEE conference on computer communications},
  pages={1749--1758},
  year={2022},
  organization={IEEE}
}

@inproceedings{dwork2006calibrating,
  title={Calibrating noise to sensitivity in private data analysis},
  author={Dwork, Cynthia and McSherry, Frank and Nissim, Kobbi and Smith, Adam},
  booktitle={Theory of cryptography conference},
  pages={265--284},
  year={2006},
  organization={Springer}
}

@inproceedings{Mironov2017,
  author    = {Ilya Mironov},
  title     = {R{\'e}nyi Differential Privacy},
  booktitle = {IEEE Computer Security Foundations Symposium (CSF)},
  year      = {2017}
}

@article{Su2021RoPE,
  title={Roformer: Enhanced transformer with rotary position embedding},
  author={Su, Jianlin and Ahmed, Murtadha and Lu, Yu and Pan, Shengfeng and Bo, Wen and Liu, Yunfeng},
  journal={Neurocomputing},
  volume={568},
  pages={127063},
  year={2024},
  publisher={Elsevier}
}

@article{fan2024simplicity,
  title={Simplicity Prevails: Rethinking Negative Preference Optimization for LLM Unlearning},
  author={Fan, Chongyu and Liu, Jiancheng and Lin, Licong and Jia, Jinghan and Zhang, Ruiqi and Mei, Song and Liu, Sijia},
  journal={arXiv preprint arXiv:2410.07163},
  year={2024}
}

@article{zhou2024emergence,
  title={On the Emergence of Cross-Task Linearity in the Pretraining-Finetuning Paradigm},
  author={Zhou, Zhanpeng and Chen, Zijun and Chen, Yilan and Zhang, Bo and Yan, Junchi},
  journal={arXiv preprint arXiv:2402.03660},
  year={2024}
}

@inproceedings{wortsman2022model,
  title={Model soups: averaging weights of multiple fine-tuned models improves accuracy without increasing inference time},
  author={Wortsman, Mitchell and Ilharco, Gabriel and Gadre, Samir Ya and Roelofs, Rebecca and Gontijo-Lopes, Raphael and Morcos, Ari S and Namkoong, Hongseok and Farhadi, Ali and Carmon, Yair and Kornblith, Simon and others},
  booktitle={International conference on machine learning},
  pages={23965--23998},
  year={2022},
  organization={PMLR}
}

@article{ilharco2022editing,
  title={Editing models with task arithmetic},
  author={Ilharco, Gabriel and Ribeiro, Marco Tulio and Wortsman, Mitchell and Gururangan, Suchin and Schmidt, Ludwig and Hajishirzi, Hannaneh and Farhadi, Ali},
  journal={arXiv preprint arXiv:2212.04089},
  year={2022}
}

@article{zhou2024metagpt,
  title={Metagpt: Merging large language models using model exclusive task arithmetic},
  author={Zhou, Yuyan and Song, Liang and Wang, Bingning and Chen, Weipeng},
  journal={arXiv preprint arXiv:2406.11385},
  year={2024}
}

@article{openunlearning2025,
  title={{OpenUnlearning}: Accelerating {LLM} Unlearning via Unified Benchmarking of Methods and Metrics},
  author={Dorna, Vineeth and Mekala, Anmol and Zhao, Wenlong and McCallum, Andrew and Lipton, Zachary C and Kolter, J Zico and Maini, Pratyush},
  journal={arXiv preprint arXiv:2506.12618},
  year={2025},
}

@article{grattafiori2024llama,
  title={The llama 3 herd of models},
  author={Grattafiori, Aaron and Dubey, Abhimanyu and Jauhri, Abhinav and Pandey, Abhinav and Kadian, Abhishek and Al-Dahle, Ahmad and Letman, Aiesha and Mathur, Akhil and Schelten, Alan and Vaughan, Alex and others},
  journal={arXiv preprint arXiv:2407.21783},
  year={2024}
}

@article{wang2024llm,
  title={{LLM} unlearning via loss adjustment with only forget data},
  author={Wang, Yaxuan and Wei, Jiaheng and Liu, Chris Yuhao and Pang, Jinlong and Liu, Quan and Shah, Ankit Parag and Bao, Yujia and Liu, Yang and Wei, Wei},
  journal={arXiv preprint arXiv:2410.11143},
  year={2024}
}

@article{maini2024tofu,
  title={Tofu: A task of fictitious unlearning for llms},
  author={Maini, Pratyush and Feng, Zhili and Schwarzschild, Avi and Lipton, Zachary C and Kolter, J Zico},
  journal={arXiv preprint arXiv:2401.06121},
  year={2024}
}

@article{zhang2024negative,
  title={Negative preference optimization: From catastrophic collapse to effective unlearning},
  author={Zhang, Ruiqi and Lin, Licong and Bai, Yu and Mei, Song},
  journal={arXiv preprint arXiv:2404.05868},
  year={2024}
}

@article{qiu2025survey,
  title={A Survey on Unlearning in Large Language Models},
  author={Qiu, Ruichen and Tan, Jiajun and Pu, Jiayue and Wang, Honglin and Gao, Xiao-Shan and Sun, Fei},
  journal={arXiv preprint arXiv:2510.25117},
  year={2025}
}

@inproceedings{jang2023knowledge,
  title={Knowledge unlearning for mitigating privacy risks in language models},
  author={Jang, Joel and Yoon, Dongkeun and Yang, Sohee and Cha, Sungmin and Lee, Moontae and Logeswaran, Lajanugen and Seo, Minjoon},
  booktitle={Proceedings of the 61st Annual Meeting of the Association for Computational Linguistics (Volume 1: Long Papers)},
  pages={14389--14408},
  year={2023}
}

@article{yao2024large,
  title={Large language model unlearning},
  author={Yao, Yuanshun and Xu, Xiaojun and Liu, Yang},
  journal={Advances in Neural Information Processing Systems},
  volume={37},
  pages={105425--105475},
  year={2024}
}

@inproceedings{mekala2025alternate,
  title={Alternate preference optimization for unlearning factual knowledge in large language models},
  author={Mekala, Anmol Reddy and Dorna, Vineeth and Dubey, Shreya and Lalwani, Abhishek and Koleczek, David and Rungta, Mukund and Hasan, Sadid A and Lobo, Elita AA},
  booktitle={Proceedings of the 31st International Conference on Computational Linguistics},
  pages={3732--3752},
  year={2025}
}

@article{liu2025rethinking,
  title={Rethinking machine unlearning for large language models},
  author={Liu, Sijia and Yao, Yuanshun and Jia, Jinghan and Casper, Stephen and Baracaldo, Nathalie and Hase, Peter and Yao, Yuguang and Liu, Chris Yuhao and Xu, Xiaojun and Li, Hang and others},
  journal={Nature Machine Intelligence},
  pages={1--14},
  year={2025},
  publisher={Nature Publishing Group UK London}
}

@inproceedings{wang2025balancing,
  title={Balancing Forget Quality and Model Utility: A Reverse KL-Divergence Knowledge Distillation Approach for Better Unlearning in LLMs},
  author={Wang, Bichen and Zi, Yuzhe and Sun, Yixin and Zhao, Yanyan and Qin, Bing},
  booktitle={Proceedings of the 2025 Conference of the Nations of the Americas Chapter of the Association for Computational Linguistics: Human Language Technologies (Volume 1: Long Papers)},
  pages={1306--1321},
  year={2025}
}

@article{shen2025lunar,
  title={Lunar: Llm unlearning via neural activation redirection},
  author={Shen, William F and Qiu, Xinchi and Kurmanji, Meghdad and Iacob, Alex and Sani, Lorenzo and Chen, Yihong and Cancedda, Nicola and Lane, Nicholas D},
  journal={arXiv preprint arXiv:2502.07218},
  year={2025}
}

@inproceedings{jin2025unlearning,
  title={Unlearning as multi-task optimization: A normalized gradient difference approach with an adaptive learning rate},
  author={Jin, Xiaomeng and Bu, Zhiqi and Vinzamuri, Bhanukiran and Ramakrishna, Anil and Chang, Kai-Wei and Cevher, Volkan and Hong, Mingyi},
  booktitle={Proceedings of the 2025 Conference of the Nations of the Americas Chapter of the Association for Computational Linguistics: Human Language Technologies (Volume 1: Long Papers)},
  pages={11278--11294},
  year={2025}
}

@inproceedings{pan2025multi,
  title={Multi-Objective Large Language Model Unlearning},
  author={Pan, Zibin and Zhang, Shuwen and Zheng, Yuesheng and Li, Chi and Cheng, Yuheng and Zhao, Junhua},
  booktitle={ICASSP 2025-2025 IEEE International Conference on Acoustics, Speech and Signal Processing (ICASSP)},
  pages={1--5},
  year={2025},
  organization={IEEE}
}

@article{zhang2025oblivionis,
  title={Oblivionis: A Lightweight Learning and Unlearning Framework for Federated Large Language Models},
  author={Zhang, Fuyao and Yan, Xinyu and Wu, Tiantong and Li, Wenjie and Chen, Tianxiang and Cao, Yang and Yan, Ran and Huang, Longtao and Lim, Wei Yang Bryan and Yang, Qiang},
  journal={arXiv preprint arXiv:2508.08875},
  year={2025}
}

@inproceedings{cho-etal-2014-properties,
    title = "On the Properties of Neural Machine Translation: Encoder{--}Decoder Approaches",
    author = {Cho, Kyunghyun  and
      van Merri{\"e}nboer, Bart  and
      Bahdanau, Dzmitry  and
      Bengio, Yoshua},
    editor = "Wu, Dekai  and
      Carpuat, Marine  and
      Carreras, Xavier  and
      Vecchi, Eva Maria",
    booktitle = "Proceedings of {SSST}-8, Eighth Workshop on Syntax, Semantics and Structure in Statistical Translation",
    month = oct,
    year = "2014",
    address = "Doha, Qatar",
    publisher = "Association for Computational Linguistics",
    url = "https://aclanthology.org/W14-4012/",
    doi = "10.3115/v1/W14-4012",
    pages = "103--111"
}

@inproceedings{lin2004rouge,
  title={Rouge: A package for automatic evaluation of summaries},
  author={Lin, Chin-Yew},
  booktitle={Text summarization branches out},
  pages={74--81},
  year={2004}
}

@article{kuurkova1994functionally,
  title={Functionally equivalent feedforward neural networks},
  author={K{\r{u}}rkov{\'a}, V{\v{e}}ra and Kainen, Paul C},
  journal={Neural Computation},
  volume={6},
  number={3},
  pages={543--558},
  year={1994},
  publisher={MIT Press One Rogers Street, Cambridge, MA 02142-1209, USA journals-info~…}
}

@article{yang2025exploring,
  title={Exploring Criteria of Loss Reweighting to Enhance LLM Unlearning},
  author={Yang, Puning and Wang, Qizhou and Huang, Zhuo and Liu, Tongliang and Zhang, Chengqi and Han, Bo},
  journal={arXiv preprint arXiv:2505.11953},
  year={2025}
}

@inproceedings{liu2022continual,
  title={Continual learning and private unlearning},
  author={Liu, Bo and Liu, Qiang and Stone, Peter},
  booktitle={Conference on Lifelong Learning Agents},
  pages={243--254},
  year={2022},
  organization={PMLR}
}

@article{rafailov2023direct,
  title={Direct preference optimization: Your language model is secretly a reward model},
  author={Rafailov, Rafael and Sharma, Archit and Mitchell, Eric and Manning, Christopher D and Ermon, Stefano and Finn, Chelsea},
  journal={Advances in neural information processing systems},
  volume={36},
  pages={53728--53741},
  year={2023}
}

@inproceedings{dong2025undial,
  title={Undial: Self-distillation with adjusted logits for robust unlearning in large language models},
  author={Dong, Yijiang River and Lin, Hongzhou and Belkin, Mikhail and Huerta, Ramon and Vuli{\'c}, Ivan},
  booktitle={Proceedings of the 2025 Conference of the Nations of the Americas Chapter of the Association for Computational Linguistics: Human Language Technologies (Volume 1: Long Papers)},
  pages={8827--8840},
  year={2025}
}

@article{qwen2024qwen2,
  title={Qwen2. 5 technical report},
  author={Yang, An and Yang, Baosong and Zhang, B and Hui, B and Zheng, B and Yu, B and Li, Chengpeng and Liu, D and Huang, F and Wei, H and others},
  journal={arXiv preprint},
  year={2024}
}

@inproceedings{wang2026m,
  title={M-Loss: Quantifying Model Merging Compatibility with Limited Unlabeled Data},
  author={Wang, Tiantong and Duan, Yiyang and Chen, Haoyu and Wu, Tiantong and Lim, Wei Yang Bryan},
  booktitle={Proceedings of the AAAI Conference on Artificial Intelligence},
  volume={40},
  number={31},
  pages={26471--26479},
  year={2026}
}

@inproceedings{dwork2006ourgaussian,
  title={Our data, ourselves: Privacy via distributed noise generation},
  author={Dwork, Cynthia and Kenthapadi, Krishnaram and McSherry, Frank and Mironov, Ilya and Naor, Moni},
  booktitle={Annual international conference on the theory and applications of cryptographic techniques},
  pages={486--503},
  year={2006},
  organization={Springer}
}

@article{shi2024muse,
  title={Muse: Machine unlearning six-way evaluation for language models},
  author={Shi, Weijia and Lee, Jaechan and Huang, Yangsibo and Malladi, Sadhika and Zhao, Jieyu and Holtzman, Ari and Liu, Daogao and Zettlemoyer, Luke and Smith, Noah A and Zhang, Chiyuan},
  journal={arXiv preprint arXiv:2407.06460},
  year={2024}
}

@article{hu2022lora,
  title={Lora: Low-rank adaptation of large language models.},
  author={Hu, Edward J and Shen, Yelong and Wallis, Phillip and Allen-Zhu, Zeyuan and Li, Yuanzhi and Wang, Shean and Wang, Liang and Chen, Weizhu and others},
  journal={Iclr},
  volume={1},
  number={2},
  pages={3},
  year={2022}
}
\bibliographystyle{plainnat}
}

\appendix
\newpage
\section*{Appendix Overview}
\label{app:overview}

This appendix provides supplementary materials for \method{} that expand the main paper
along two axes: \textit{\textbf{(i)}} mathematical derivations and analysis supporting the
core design (noise construction, denoising aggregation, and function-preserving
reparameterizations), and \textit{\textbf{(ii)}} experimental details, evaluation, additional results, and ablations.

To keep the appendix navigable, we organize it into two parts.

\paragraph{Part A: Mathematical Details and Analysis}
This part provides the mathematical foundations behind
\method{}, including the structured zero-sum noise, harmonic denoising aggregation,
Transformer reparameterization symmetries, and a formal comparison against a naive
single-copy ``noise-only'' baseline.
\begin{itemize}
    \item \textbf{Appendix~\ref{app:mpu-noise}}: Structured zero-sum noise construction and statistics, including marginal variance, cross-covariance, and the stacked covariance form.
    \item \textbf{Appendix~\ref{app:mpu_dp_combination}}: How the DP mechanism can be integrated in MPU to ensure a formal privacy guarantee.
    \item \textbf{Appendix~\ref{app:mpu-aggregation}}: Harmonic denoising aggregation, including first-order exact cancellation, second-order remainder bounds (Appendix~\ref{app:mpu-second-order}), and the uniqueness/optimality of harmonic weights (Appendix~\ref{app:mpu-optimality}).
    \item \textbf{Appendix~\ref{app:pum-reparam}}: Function-preserving reparameterizations for Transformers, including attention head basis transforms (Appendix~\ref{app:mpu-attn}), RoPE-aware commutant restrictions (Appendix~\ref{app:mpu-rope}), Feed-Forward Neural Network hidden-channel permutations (Appendix~\ref{app:mpu-ffn}), and linearity of the reparameterization map (Appendix~\ref{app:mpu-linear-T}).
    \item \textbf{Appendix~\ref{app:mpu-reparam-trajectory}}: Discussion of whether reparameterization affects client-side optimization trajectories and loss smoothness (equivariance and Euclidean smoothness invariance).
    \item \textbf{Appendix~\ref{app:security_hardness}}: Discuss the hardness level of breaking the reparameterization. Provide a theoretical security guarantee.
    \item \textbf{Appendix~\ref{app:overhead}}: Discuss the overhead of MPU. Showing that MPU only adds a negligible computation burden.
    \item \textbf{Appendix~\ref{sec:appendix-mpu-vs-noiseonly}}: Comparison of update error between MPU and single-copy noisy unlearning, including one-round bias and variance analysis (Appendix~\ref{app:mpu-compare-1round}), multi-round implications (Appendix~\ref{app:mpu-compare-multiround}), and an SNR interpretation (Appendix~\ref{app:mpu-snr}).
    \item \textbf{Appendix~\ref{sec:mpu-stability}}: Method discussion on stability enhancement via multi-copy learning.
\end{itemize}

\paragraph{Part B: Implementation Details and Supplementary Experiments}
This part provides benchmark descriptions, evaluation metrics,
unlearning algorithm formulations, implementation details, and extensive supplementary experiments and ablations.
\begin{itemize}
    \item \textbf{Appendix~\ref{Benchmarks}}: Benchmark description.
    \item \textbf{Appendix~\ref{app:Metrics}}: Evaluation metrics, including memorization (Appendix~\ref{metrics:mem}), privacy (Appendix~\ref{metrics:pri}), and utility (Appendix~\ref{metrics:utl}).
    \item \textbf{Appendix~\ref{Unlearning}}: Unlearning algorithms instantiated within \method{}.
    \item \textbf{Appendix~\ref{Experimental}}: Implementation details, including testbed configuration (Appendix~\ref{Testbed}), default hyperparameters (Appendix~\ref{Hyperparameters}), and the prompt template used for training and evaluation (Appendix~\ref{Prompt}).
    \item \textbf{Appendix~\ref{Supplementary Experiments}}: Additional experimental results.
    \item \textbf{Appendix~\ref{app:pum_supp_ablations}}: Detailed supplementary analysis and ablations, including the effect of copy number $m$ (Appendix~\ref{app:pum_copy_number}), noise level $\kappa$ (Appendix~\ref{app:pum_noise_level}), round--epoch allocation (Appendix~\ref{app:pum_round_epoch_alloc}), the no-denoise ablation (Appendix~\ref{app:pum_nodenoise}), robustness to larger forget splits (Appendix~\ref{app:pum_forget_splits}), and scaling across model sizes (Appendix~\ref{app:pum_model_scaling}).
\end{itemize}

\clearpage
\section{Mathematical Details and Analysis}
\label{app:mpu}\label{app:pum}

This appendix provides the mathematical details supporting the method described in Section~\ref{sec:method}, including: \textit{\textbf{(i)}} the structured zero-sum noise construction (Section~\ref{sec:pum-noise}), \textit{\textbf{(ii)}} the harmonic denoising aggregation and its associated error terms (Section~\ref{sec:pum-aggregation}), and \textit{\textbf{(iii)}} the function-preserving reparameterization family used to obfuscate the original parameter space (Section~\ref{sec:pum-reparam}). We additionally present a comparison with a naive single-copy ``noise-only'' baseline, which injects noise but performs no denoising.

\subsection{Structured Zero-Sum Noise Construction and Statistics}
\label{app:mpu-noise}

Fix a communication round $r$ and a parameter block (\eg a layer) $\ell$ with dimension
$d_\ell$. Recall the construction in Eq.~\eqref{eq:pum-noise-construct}: draw i.i.d.
$z_{k,\ell}^{(r)} \sim \mathcal{N}(0, \sigma_\ell^2 I_{d_\ell})$ for $k \in [m]$, define
$\bar z_\ell^{(r)} = \frac{1}{m}\sum_{k=1}^m z_{k,\ell}^{(r)}$, and set
\begin{equation}
\epsilon^{0,(r)}_{k,\ell}
\;=\;
\sqrt{\frac{m}{m-1}}\Big(z_{k,\ell}^{(r)} - \bar z_\ell^{(r)}\Big),
\quad
\epsilon^{(r)}_{k,\ell} = \alpha_k\,\epsilon^{0,(r)}_{k,\ell}.
\end{equation}
By construction, $\sum_{k=1}^m \epsilon^{0,(r)}_{k,\ell} \equiv 0$ for every block $\ell$, and hence $\sum_k \epsilon_k^{0,(r)} \equiv 0$ for the stacked vector. 

We define
\begin{equation}
    \epsilon_k^{0,(r)} = \text{stack}_{\ell}(\epsilon^{0,(r)}_{k,\ell}),
\end{equation}
where the stack operation means to stack the block-wise noises together to obtain a model-level noise.

\paragraph{Marginal Distribution and Cross-Covariance of Zero-Sum Base}
Let $z_k := z_{k,\ell}^{(r)} \in \mathbb{R}^{d_\ell}$ and
$\bar z = \frac{1}{m}\sum_{k=1}^m z_k$.
Since $z_k \sim \mathcal{N}(0, \sigma_\ell^2 I_{d_\ell})$ are i.i.d., we have
\begin{equation}
\mathrm{Var}(z_k - \bar z)
=
\Big(1 - \frac{1}{m}\Big)\sigma_\ell^2 I_{d_\ell}.
\end{equation}
The scaling factor $\sqrt{\frac{m}{m-1}}$ therefore restores the desired marginal variance:
\begin{equation}
\mathrm{Var}\!\big(\epsilon^{0,(r)}_{k,\ell}\big)
=
\frac{m}{m-1}\,\mathrm{Var}(z_k - \bar z)
=
\sigma_\ell^2 I_{d_\ell}
\Rightarrow
\epsilon^{0,(r)}_{k,\ell} \sim \mathcal{N}(0, \sigma_\ell^2 I_{d_\ell}) \;\;\text{marginally}.
\end{equation}
For $k \neq j$, using $\mathrm{Cov}(z_k, \bar z) = \frac{1}{m}\sigma_\ell^2 I_{d_\ell}$ and
$\mathrm{Var}(\bar z) = \frac{1}{m}\sigma_\ell^2 I_{d_\ell}$, we obtain
\begin{equation}
\mathrm{Cov}(z_k - \bar z,\; z_j - \bar z)
=
-\frac{1}{m}\sigma_\ell^2 I_{d_\ell},
\end{equation}
and hence, after scaling,
\begin{equation}
\mathrm{Cov}\!\big(\epsilon^{0,(r)}_{k,\ell}, \epsilon^{0,(r)}_{j,\ell}\big)
=
-\frac{1}{m-1}\sigma_\ell^2 I_{d_\ell},
\quad k \neq j.
\end{equation}
After applying the per-copy scaling, this becomes
\begin{equation}
\mathrm{Cov}\!\big(\epsilon^{(r)}_{k,\ell}, \epsilon^{(r)}_{j,\ell}\big)
=
-\frac{\alpha_k \alpha_j}{m-1}\sigma_\ell^2 I_{d_\ell},
\quad k \neq j.
\end{equation}

\paragraph{Scaled-Copy Covariance (Matrix Form)}
Let
$\boldsymbol{\epsilon}_\ell := [\epsilon^{(r)}_{1,\ell};\ldots;\epsilon^{(r)}_{m,\ell}]
\in\mathbb{R}^{m d_\ell}$ be the stacked noise vector for block $\ell$. Then
\begin{equation}
\label{eq:mpu-noise-cov}
\mathrm{Cov}\!\big[\boldsymbol{\epsilon}_{\ell}\big]
=
\sigma_\ell^2\,(D B D)\otimes I_{d_\ell},
\end{equation}
where
\begin{equation}
D = \mathrm{diag}(\alpha_1,\ldots,\alpha_m),
\quad
B = \frac{m}{m-1}I_m - \frac{1}{m-1}\mathbf{1}\mathbf{1}^\top.
\end{equation}
Since $B$ has eigenvalues $\frac{m}{m-1}$ (multiplicity $m-1$) and $0$ (multiplicity $1$),
$\mathrm{Cov}[\boldsymbol{\epsilon}_\ell]$ has rank $(m-1)d_\ell$. This expresses the core
property used by harmonic denoising: the $m$ copy noises live in an $(m-1)$-dimensional
subspace (per block), and the ``missing'' direction is precisely the all-ones direction that harmonic aggregation cancels.

\subsection{Possible Combination of MPU and Differential Privacy}
\label{app:mpu_dp_combination}

While the noise injection serves as an effective defense, certain deployments may require formal, theoretical privacy guarantees. In this section, we provide the preliminaries of Differential Privacy (DP) and detail how MPU's noise injection mechanism can be naturally combined with DP to offer rigorous mathematical privacy guarantees for the server's parameters.

\subsubsection{Preliminaries: Differential Privacy}

DP is a rigorous mathematical framework for privacy-preserving data analysis that quantifies and strictly bounds the privacy risks associated with algorithmic outputs.

Formally, a randomized algorithm $\mathcal{M}$ satisfies $(\epsilon, \delta)$-DP if, for all adjacent datasets $\mathcal{D}$ and $\mathcal{D}'$ (which differ by exactly one data record) and for all sets of possible outputs $\mathcal{S} \subseteq \text{Range}(\mathcal{M})$, the following inequality holds:
\begin{equation}
    \mathbb{P}[\mathcal{M}(\mathcal{D}) \in \mathcal{S}] \le e^\epsilon \mathbb{P}[\mathcal{M}(\mathcal{D}') \in \mathcal{S}] + \delta
\end{equation}
where $\epsilon > 0$ is the privacy budget and $\delta \in [0, 1)$ is a cryptographically small probability of a strict privacy breach. A smaller $\epsilon$ corresponds to a stronger privacy guarantee.

\textbf{DP provides formal privacy guarantee} \\
DP ensures that the probability distribution of the algorithm's output remains almost indistinguishable whether any single individual's data is included in the dataset or not. This provides \textit{plausible deniability} to all individuals in the training data. Because the guarantee formally bounds the worst-case likelihood ratio of the outputs, an adversary cannot reliably infer the presence, absence, or specific content of any individual training sample. This holds true regardless of how much arbitrary side information or computational power the adversary possesses, thereby provably preventing attacks such as membership inference and data reconstruction.

\textbf{The Gaussian Mechanism and Post-Processing} \\
To achieve $(\epsilon, \delta)$-DP for a deterministic vector-valued function $f(\mathcal{D}) \in \mathbb{R}^d$ (e.g., releasing model parameters), the standard approach is the Gaussian mechanism. This involves adding independent noise drawn from $\mathcal{N}(0, \sigma_{DP}^2)$ to each coordinate. The required noise scale $\sigma_{DP}$ is determined by the $L_2$-sensitivity of the function, defined as $\Delta f = \max_{\mathcal{D}, \mathcal{D}'} \|f(\mathcal{D}) - f(\mathcal{D}')\|_2$. Specifically, returning $f(\mathcal{D}) + \mathcal{N}(0, \sigma_{DP}^2 I)$ guarantees $(\epsilon, \delta)$-DP if:
\begin{equation}
    \sigma_{DP} \ge \frac{\Delta f \sqrt{2 \ln(1.25/\delta)}}{\epsilon}
\end{equation}
A crucial property of DP is its immunity to \textit{post-processing}: if a randomized output $\mathcal{M}(\mathcal{D})$ satisfies DP, then for any arbitrary data-independent function $g$, the composition $g(\mathcal{M}(\mathcal{D}))$ also satisfies the exact same DP guarantees.

\subsubsection{Combining MPU with Computational Differential Privacy}

In the standard DP framework, privacy guarantees are typically established against computationally unbounded adversaries (i.e., information-theoretic DP). In the context of MPU, relying on information-theoretic DP would be problematic: because the zero-sum noise vectors are algebraically constrained ($\sum_{k=1}^m \epsilon_k^{0,(r)} \equiv 0$), an unbounded adversary could hypothetically break the reparameterizations $T_{k,r}$, align the $m$ copies, average them, and perfectly eliminate the noise, thereby bypassing the privacy guarantee. To prevent this, one would traditionally be forced to inject an additional, independent shared Gaussian noise that cannot be canceled, which inherently degrades the unlearning utility.

However, as formally established in Appendix~\ref{app:security_hardness}, attempting to reverse the reparameterizations to align the mismatched coordinate spaces is \textbf{NP-hard} (equivalent to the Multi-Dimensional Assignment Problem and the Generalized Orthogonal Procrustes Problem). Therefore, against any realistic, probabilistic polynomial-time (PPT) adversary, we can safely establish our guarantees under the framework of \textbf{Computational DP (CDP)}. CDP provides the same robust mathematical privacy guarantees as standard DP, but relaxes the adversary assumption to computationally bounded entities. Because the computationally bounded adversary cannot solve the NP-hard alignment problem, the zero-sum structural dependency between the copies remains cryptographically inaccessible. Consequently, we can use the zero-sum noise itself to provide the formal privacy guarantee.

The integration of MPU with CDP operates as follows:

\paragraph{Bounding Sensitivity and Calibrating Zero-Sum Noise.}
The server first bounds the global $L_2$-sensitivity $\Delta \theta$ of its model parameters $\theta_{r-1}$ via standard DP-training techniques (e.g., DP-SGD, gradient clipping). To achieve a target computational privacy budget $(\epsilon_{srv}, \delta_{srv})$, the server ensures that the marginal variance of each $\epsilon_k^{(r)}$ is $\sigma_{DP}^2 I_d$, where $\sigma_{DP}$ is determined strictly according to the Gaussian mechanism:
\begin{equation}
    \sigma_{DP} = \frac{\Delta \theta \sqrt{2 \ln(1.25/\delta_{srv})}}{\epsilon_{srv}}
\end{equation}
The server scales the base zero-sum noise vectors (Eq. 4) such that their marginal distributions match $\mathcal{N}(0, \sigma_{DP}^2 I_d)$, and adds them to the base parameters: $\tilde{\theta}_k^{(r)} = \theta_{r-1} + \epsilon_k^{(r)}$.

\paragraph{Reparameterization and Computational DP Guarantee}
The server then applies the data-independent reparameterization $T_{k,r} \leftarrow \textsc{SampleReparam}(t_r, k)$ to yield the published models distributed to the clients:
\begin{equation}
    \theta_{pub}^{(k,r)} = T_{k,r}(\tilde{\theta}_k^{(r)})
\end{equation}
To a PPT adversary, the transformed copies $\{ \theta_{pub}^{(k,r)} \}_{k=1}^m$ are computationally indistinguishable from copies perturbed by truly independent Gaussian noise, because discovering the zero-sum correlation requires solving an NP-hard alignment problem. Since independent Gaussian noise of scale $\sigma_{DP}$ satisfies $(\epsilon_{srv}, \delta_{srv})$-DP, the published copies in MPU satisfy $(\epsilon_{srv}, \delta_{srv})$-CDP.

\paragraph{Perfect Privacy-Utility Decoupling (Zero First-Order Utility Cost)}
This formulation yields a profound advantage for the MPU framework: it completely bypasses the fundamental privacy-utility trade-off inherent in standard DP. During the Post-Process update aggregation (Section 3.4), the server uses its secret keys $t_r$ to invert the reparameterizations in $\mathcal{O}(d)$ time and applies harmonic denoising. Because \textit{all} the injected DP noise is constructed with the zero-sum constraint ($\sum_{k=1}^m \epsilon_k^{0,(r)} \equiv 0$), the server's harmonic aggregation perfectly eliminates the first-order noise error. 

Consequently, the server can inject an arbitrarily massive noise scale $\sigma_{DP}$ to achieve an exceptionally small privacy budget $\epsilon_{srv}$ (extremely strong privacy) without suffering the catastrophic utility degradation that typically plagues DP frameworks. The formal privacy guarantee is satisfied during client-side exposure, while the noise is mathematically canceled upon server-side aggregation.

\paragraph{Remark: Client-Side Differential Privacy.}
While the above details the protection of the server's model, MPU also naturally complements client-side DP. If the client wishes to formally protect its local forget set, it can perform local unlearning using standard DP-SGD. As proved in Appendix~\ref{app:mpu-optimality}, the server's harmonic aggregation inherently provides a variance reduction effect for any independent client-side noise. Thus, MPU not only guarantees server privacy without first-order cost, but also actively suppresses the utility degradation associated with the client's local DP mechanism.

\subsection{Harmonic Denoising Aggregation: Cancellation, Remainder, and Streaming}
\label{app:mpu-aggregation}

We \textit{\textbf{(i)}} formalize the first-order cancellation argument, \textit{\textbf{(ii)}} derive a clean second-order remainder bound under a local smoothness assumption, \textit{\textbf{(iii)}} establish the uniqueness and optimality of harmonic weights for zero-sum cancellation, and \textit{\textbf{(iv)}} justify the streaming implementation used by \method{}
(Section~\ref{sec:pum-streaming}).

\paragraph{Local Linearization Model}
Let $\Delta^\star(\theta)$ denote the ideal (noise-free) unlearning displacement produced by the chosen unlearning trainer when initialized at parameters $\theta$. For analytical clarity, we assume fixed client data and fixed algorithmic randomness.
Define the Jacobian of $\Delta^\star$ at the current iterate as
\begin{equation}
    J
\;=\;
\left.
\frac{\partial \Delta^\star(\theta)}{\partial \theta}
\right|_{\theta_{r-1}}.
\end{equation}
After inverting the reparameterization (Section~\ref{sec:pum-aggregation}), we model each aligned update via a first-order expansion in the injected perturbation:
\begin{equation}
\label{eq:mpu-firstorder}
\widehat{\Delta}^{(k,r)}
=
\Delta^\star(\theta_{r-1})
+
J\,\epsilon_k^{(r)}
+
\rho_k^{(r)},
\quad
\epsilon_k^{(r)}=\alpha_k\,\epsilon_k^{0,(r)},
\end{equation}
where the remainder term $\rho_k^{(r)}$ captures second-order (and higher-order) effects, as well as any deviation from modeling the client-side unlearning routine as a deterministic map.

\subsubsection{First-Order Exact Cancellation via Harmonic Weights}

The server aggregates the aligned updates using harmonic weights
$w_k \propto \alpha_k^{-1}$:
\begin{equation}
\bar{\Delta}^{(r)}
=
\sum_{k=1}^m w_k\,\widehat{\Delta}^{(k,r)},
\quad
w_k
:=
\frac{\alpha_k^{-1}}{\sum_{j=1}^m \alpha_j^{-1}}.
\end{equation}
Substituting the first-order model in Eq.~\eqref{eq:mpu-firstorder} yields
\begin{equation}
\bar{\Delta}^{(r)}
=
\Delta^\star(\theta_{r-1})
+
J \sum_{k=1}^m w_k \epsilon_k^{(r)}
+
\sum_{k=1}^m w_k \rho_k^{(r)}.
\end{equation}
Since
\begin{equation}
w_k \epsilon_k^{(r)}
=
\frac{\alpha_k^{-1}}{S_0}\,\alpha_k \epsilon_k^{0,(r)}
=
\frac{1}{S_0}\,\epsilon_k^{0,(r)},
\quad
S_0 := \sum_{j=1}^m \alpha_j^{-1},
\end{equation}
we obtain
\begin{equation}
\sum_{k=1}^m w_k \epsilon_k^{(r)}
=
\frac{1}{S_0} \sum_{k=1}^m \epsilon_k^{0,(r)}
=
0,
\end{equation}
where the last equality follows from the zero-sum property in
Eq.~\eqref{eq:pum-noise-construct}.

Therefore,
\begin{equation}
\bar{\Delta}^{(r)}
=
\Delta^\star(\theta_{r-1})
+
\sum_{k=1}^m w_k \rho_k^{(r)},
\end{equation}
which formally establishes that the injected noise is canceled \emph{exactly to first order}.

\subsubsection{Second-Order Remainder under Lipschitz Jacobian}
\label{app:mpu-second-order}

A standard approach to control the remainder term $\rho_k^{(r)}$ is to assume local
smoothness of the Jacobian.

\paragraph{Assumption (Local Lipschitz Jacobian)}
Assume that $\Delta^\star(\theta)$ is Fr\'echet differentiable in a neighborhood of
$\theta_{r-1}$, and that its Jacobian $J(\theta)$ is $L_J$-Lipschitz in this neighborhood:
\begin{equation}
\|J(\theta) - J(\theta')\|
\;\le\;
L_J \, \|\theta - \theta'\|.
\end{equation}
\paragraph{Remainder Bound}
Let $\delta_k^{(r)} := \epsilon_k^{(r)}$.
By the integral remainder form of Taylor's theorem, we have
\begin{equation}
\begin{aligned}
\rho_k^{(r)}
&=
\Delta^\star(\theta_{r-1} + \delta_k^{(r)})
-
\Delta^\star(\theta_{r-1})
-
J\,\delta_k^{(r)} \\
&=
\int_0^1
\Big(
J(\theta_{r-1} + t\,\delta_k^{(r)}) - J(\theta_{r-1})
\Big)\,
\delta_k^{(r)}\,dt.
\end{aligned}
\end{equation}
Taking norms and applying the Lipschitz condition yields
\begin{equation}
\|\rho_k^{(r)}\|
\;\le\;
\int_0^1 L_J\, t \, \|\delta_k^{(r)}\|^2 \, dt
=
\frac{L_J}{2}\,\|\epsilon_k^{(r)}\|^2.
\end{equation}
Consequently,
\begin{equation}
\big\|
\bar{\Delta}^{(r)} - \Delta^\star(\theta_{r-1})
\big\|
=
\Big\|
\sum_{k=1}^m w_k \rho_k^{(r)}
\Big\|
\;\le\;
\frac{L_J}{2}
\sum_{k=1}^m
w_k\,\|\epsilon_k^{(r)}\|^2,
\label{eq:mpu-second-order-bound}
\end{equation}
which matches the $O(\|\epsilon\|^2)$ behavior stated in
Section~\ref{sec:pum-aggregation}.

\subsubsection{Uniqueness and Optimality of Harmonic Weights}
\label{app:mpu-optimality}

In this section, we show that harmonic weights are the unique linear weights that cancel
the first-order perturbation for \emph{all} zero-sum noise realizations.

\begin{proposition}[Uniqueness of Harmonic Weights for Zero-Sum Cancellation]
\label{prop:mpu-unique-weights}
Consider a linear estimator of the form
$\sum_{k=1}^m w_k \widehat{\Delta}^{(k,r)}$ with $\sum_{k=1}^m w_k = 1$.
Assume the linear response model
\begin{equation}
\widehat{\Delta}^{(k,r)}
=
\Delta^\star(\theta_{r-1})
+
J\,\alpha_k\,\epsilon_k^{0,(r)}
+
\rho_k^{(r)},
\end{equation}
where the base noises satisfy $\sum_{k=1}^m \epsilon_k^{0,(r)} \equiv 0$.
If the first-order term cancels for \emph{all} such zero-sum realizations, \ie
\begin{equation}
\sum_{k=1}^m w_k \alpha_k \,\epsilon_k^{0,(r)} \equiv 0
\quad
\text{for all } \{\epsilon_k^{0,(r)}\}_{k=1}^m \text{ with }
\sum_{k=1}^m \epsilon_k^{0,(r)} = 0,
\end{equation}
then necessarily
\begin{equation}
w_k
=
\frac{\alpha_k^{-1}}{\sum_{j=1}^m \alpha_j^{-1}},
\quad k=1,\dots,m.
\end{equation}
\end{proposition}
\begin{proof}
Cancellation for all zero-sum realizations is equivalent to requiring that the vector
$(w_k \alpha_k)_{k=1}^m$ is orthogonal to the entire subspace
$\{(\epsilon_k) \in \mathbb{R}^m : \sum_{k=1}^m \epsilon_k = 0\}$.
This subspace is spanned by pairwise difference vectors of the form $e_i - e_j$.
Hence, orthogonality implies $w_i \alpha_i = w_j \alpha_j$ for all $i,j$, i.e.,
$w_k \alpha_k = c$ for some constant $c$.
Imposing the normalization $\sum_{k=1}^m w_k = 1$ yields
$c = 1 / \sum_{j=1}^m \alpha_j^{-1}$, which gives the harmonic form.
\end{proof}

\paragraph{Variance Reduction of Client-Side Randomness}
Suppose the aligned updates additionally include zero-mean stochasticity from local
training,
\begin{equation}
\widehat{\Delta}^{(k,r)}
=
\Delta^\star(\theta_{r-1})
+
J\,\epsilon_k^{(r)}
+
\eta_k^{(r)}
+
\rho_k^{(r)},
\end{equation}
where $\mathbb{E}[\eta_k^{(r)}] = 0$ and
$\mathrm{Cov}(\eta_k^{(r)}) = \Sigma_\eta$, independently across $k$.
Ignoring the $O(\|\epsilon\|^2)$ remainder terms for clarity, the covariance of the aggregated update satisfies
\begin{equation}
\mathrm{Cov}(\bar{\Delta}^{(r)})
=
\sum_{k=1}^m w_k^2 \,\Sigma_\eta
=
\Bigg(
\frac{\sum_{k=1}^m \alpha_k^{-2}}
     {(\sum_{k=1}^m \alpha_k^{-1})^2}
\Bigg)\Sigma_\eta.
\end{equation}
When the scales $\{\alpha_k\}$ are of comparable magnitude, the prefactor scales as $\asymp 1/m$, formalizing the $1/m$ variance-reduction intuition discussed in Section~\ref{sec:mpu-stability}.

\subsection{Function-Preserving Reparameterizations}
\label{app:pum-reparam}

This section expands Section~\ref{sec:pum-reparam} by providing full Transformer-level
details. We describe the reparameterizations used by \method{} from a group-theoretic
perspective and present explicit forward and inverse formulas together with invariance
proofs, including the RoPE-commutation constraint.

\paragraph{Parameter-Symmetry Group Viewpoint}
Let $\Theta$ denote the parameter tuple of a Transformer block, and let
$f_\Theta:\mathbb{R}^{L\times d_{\mathrm{model}}}\to\mathbb{R}^{L\times d_{\mathrm{model}}}$
denote the forward function induced by $\Theta$ on a length-$L$ sequence.
Define the set of \emph{function-preserving reparameterizations} as
\begin{equation}
\mathcal{G}
\;=\;
\Big\{\,T:\Theta\mapsto T(\Theta)\;\Big|\;
T \text{ is invertible and } f_{T(\Theta)}(X)=f_\Theta(X)\ \forall X\,\Big\}.
\end{equation}
Under composition, $\mathcal{G}$ forms a group, and each $T\in\mathcal{G}$ constitutes a
parameter-symmetry of the model.

In \method, we take $L=d_{\mathrm{ff}}$ and sample $T_{k,r}$ \emph{data-independently}
from a structured subset $\mathcal{T}\subseteq\mathcal{G}$ constructed from:
\textit{\textbf{(i)}} discrete channel permutations within FFN submodules, and
\textit{\textbf{(ii)}} continuous orthogonal basis changes within attention head subspaces.
For RoPE-based models, we further restrict the attention basis changes to lie in the
commutant of the RoPE operators, ensuring functional invariance.

\subsubsection{Attention Head Reparameterization}
\label{app:mpu-attn}
\label{app:sa-orth}

Let $d_{\mathrm{model}}$ denote the model width, $d_h$ the head dimension,
$H_Q$ the number of query heads, and $H_{KV}$ the number of key/value heads
(GQA/MQA are covered by $H_Q \ge H_{KV}$).
We use the following parameters:
\begin{subequations}
\begin{align}
W_Q&\in\mathbb{R}^{d_{\mathrm{model}}\times(H_Q d_h)},\\
W_K, W_V&\in\mathbb{R}^{d_{\mathrm{model}}\times(H_{KV} d_h)},\\
W_O&\in\mathbb{R}^{(H_Q d_h)\times d_{\mathrm{model}}}.
\end{align}
\end{subequations}

with parameter tuple 
\begin{equation}
\Theta = (W_Q, W_K, W_V, W_O)
\end{equation}

\paragraph{Group Element and Lifted Action}
Let $\pi : [H_Q] \to [H_{KV}]$ denote the fixed assignment mapping each query head to its
corresponding key/value head.
Sample per-KV orthogonal blocks $S_j \in O(d_h)$ for $j = 1,\ldots,H_{KV}$, and define
\begin{subequations}
\begin{align}
S_{KV} &:= \operatorname{diag}(S_1,\ldots,S_{H_{KV}})\in O(H_{KV}d_h),\\
U(\pi,S_{KV}) &:= \operatorname{blkdiag}\!\big(S_{\pi(1)},\ldots,S_{\pi(H_Q)}\big)\in O(H_Q d_h).
\end{align}
\end{subequations}
Both $S_{KV}$ and $U(\pi, S_{KV})$ are orthogonal matrices.

\paragraph{Forward and Inverse Formulas}
The reparameterization acts on the attention weights as
\begin{subequations}
\label{eq:mpu-attn-forward}
\begin{align}
W_Q' &= W_Q\,U(\pi, S_{KV}),
&
W_K' &= W_K\,S_{KV}, \\
W_V' &= W_V\,S_{KV},
&
W_O' &= U(\pi, S_{KV})^\top\,W_O. \\
T(\Theta) &= (W_Q',W_K',W_V',W_O')
\end{align}
\end{subequations}

If biases are present, apply the same right-multiplication:
$b_Q' = b_Q U(\pi,S_{KV})$, $b_K' = b_K S_{KV}$, and $b_V' = b_V S_{KV}$.
The inverse transformation is given by transposition:
\begin{subequations}
\label{eq:mpu-attn-inverse}\label{eq:pum-attn-inverse}
\begin{align}
W_Q &= W_Q'\,U(\pi, S_{KV})^\top,
&
W_K &= W_K'\,S_{KV}^\top, \\
W_V &= W_V'\,S_{KV}^\top,
&
W_O &= U(\pi, S_{KV})\,W_O'.
\end{align}
\end{subequations}

\paragraph{Function Invariance (without RoPE)}
Let $X \in \mathbb{R}^{L \times d_{\mathrm{model}}}$ and define
$Q = X W_Q$, $K = X W_K$, and $V = X W_V$, with head-wise partitions
$Q = [Q_1 \mid \cdots \mid Q_{H_Q}]$,
$K = [K_1 \mid \cdots \mid K_{H_{KV}}]$, and
$V = [V_1 \mid \cdots \mid V_{H_{KV}}]$,
where each $Q_i, K_j, V_j \in \mathbb{R}^{L \times d_h}$.
For GQA attention, the outputs are given by
\begin{equation}
A_i
=
\mathrm{softmax}\!\Big(\tfrac{1}{\sqrt{d_h}} Q_i K_{\pi(i)}^\top \Big),
\quad
O_i = A_i V_{\pi(i)},
\quad
Y = [O_1 \mid \cdots \mid O_{H_Q}]\, W_O .
\end{equation}

\textbf{Claim}
Under the reparameterization in Eq.~\eqref{eq:mpu-attn-forward}, the realized attention
function is invariant, i.e., $Y' = Y$.

\begin{proof}
From Eq.~\eqref{eq:mpu-attn-forward}, we have
$Q' = Q\,U(\pi,S_{KV})$ and $K' = K\,S_{KV}$, $V' = V\,S_{KV}$.
Consequently, $Q'_i = Q_i S_{\pi(i)}$ and $K'_{\pi(i)} = K_{\pi(i)} S_{\pi(i)}$.
Since $S_{\pi(i)}^\top S_{\pi(i)} = I$, it follows that
\begin{equation}
Q'_i K'_{\pi(i)}{}^\top
=
\left(Q_i S_{\pi\left(i\right)}\right)\left(S_{\pi\left(i\right)}^\top K_{\pi\left(i\right)}^\top\right)
=
Q_i K_{\pi(i)}^\top,
\end{equation}
and hence the attention logits, and therefore $A'_i$, are unchanged.

Moreover,
\begin{equation}
O'_i
=
A'_i V'_{\pi(i)}
=
A_i \left(V_{\pi\left(i\right)} S_{\pi\left(i\right)}\right)
=
(O_i) S_{\pi(i)}.
\end{equation}
Concatenation yields
$\left[O'_1 \mid \cdots \mid O'_{H_Q}\right]
=
\left[O_1 \mid \cdots \mid O_{H_Q}\right]\, U\left(\pi,S_{KV}\right)$.
Finally, using $W_O' = U(\pi,S_{KV})^\top W_O$, we obtain
\begin{equation}
Y'
=
\left[O'\right] W_O'
=
\left(\left[O\right] U\right)\left(U^\top W_O\right)
=
[O] W_O
=
Y,
\end{equation}
which completes the proof.
\end{proof}

\subsubsection{RoPE-Aware Specialization}
\label{app:mpu-rope}

Llama-type decoders apply rotary positional embeddings (RoPE) to queries and keys via
right-multiplication by a position-dependent orthogonal operator
$\Phi(p) \in \mathbb{R}^{d_h \times d_h}$:
\begin{equation}
\widetilde Q_i[p,:] = Q_i[p,:]\,\Phi(p)^\top,
\quad
\widetilde K_j[p,:] = K_j[p,:]\,\Phi(p)^\top.
\end{equation}
RoPE acts as independent $2\times 2$ rotations on $d_h/2$ disjoint planes,
\begin{equation}
\Phi(p)
=
\operatorname{blkdiag}\!\big(R(\omega_1 p),\ldots,R(\omega_{d_h/2} p)\big),
\quad
R(\varphi)
=
\begin{bmatrix}
\cos\varphi & -\sin\varphi \\
\sin\varphi & \cos\varphi
\end{bmatrix}
\in \mathrm{SO}(2).
\end{equation}
To preserve the attention logits, we require the head-basis transform to commute with all
$\Phi(p)$, \ie
\begin{equation}
S_j \Phi(p) = \Phi(p) S_j,
\quad \forall\, p.
\end{equation}

\paragraph{Commutant under Distinct Frequencies}
When the RoPE frequencies $\{\omega_r\}$ are distinct across the $d_h/2$ planes, the
orthogonal transformations that commute with all $\Phi(p)$ are exactly the per-plane
rotations:
\begin{equation}
S_j
=
\operatorname{blkdiag}\!\left(R\left(\varphi_{j,1}\right),\ldots,R\left(\varphi_{j,d_h/2}\right)\right)
\;\in\;
\mathrm{SO}(2)^{\,d_h/2},
\quad
j = 1,\ldots,H_{KV}.
\label{eq:mpu-rope-commuting}
\end{equation}
Reflections in $O(2)^{d_h/2} \setminus \mathrm{SO}(2)^{d_h/2}$ generally fail to commute with $\Phi(p)$,
except at degenerate angles.
The rotation angles $\{\varphi_{j,r}\}$ may be sampled deterministically from the round
seed and indices $(t_r, k, \text{layer}, j, r)$.

\paragraph{Function Invariance with RoPE}
Under the commutation condition
$S_{\pi(i)} \Phi(p) = \Phi(p) S_{\pi(i)}$, the proof of function invariance follows
identically to the non-RoPE case.
In particular, commutativity implies
\begin{equation}
S_{\pi(i)} \Phi(p)^\top \Phi(p') S_{\pi(i)}^\top
=
\Phi(p)^\top \Phi(p'),
\end{equation}
inside the attention logits, so the attention weights remain unchanged.
The output invariance then follows as before via cancellation with
$W_O' = U(\pi,S_{KV})^\top W_O$.

\subsubsection{Feed-Forward Blocks: Hidden-Channel Permutations}
\label{app:mpu-ffn}
\label{app:ffn-perm}

\paragraph{Standard Two-Layer Feed-Forward Neural Network}
Consider an FFN with parameters $\Theta = (W_1,b_1,W_2,b_2)$ and an element-wise nonlinearity $\sigma$:
\begin{equation}
f_\Theta(x)=W_2\,\sigma(W_1x+b_1)+b_2.
\end{equation}
Let $P\in\mathbb{R}^{d_{\mathrm{ff}}\times d_{\mathrm{ff}}}$ be a randomly drawn permutation matrix
($P^{-1}=P^\top$) acting on hidden channels. Define the reparameterization
\begin{equation}
W_1' = P W_1,\quad
b_1' = P b_1,\quad
W_2' = W_2 P^\top,\quad
b_2' = b_2.
\end{equation}
so that 
\begin{equation}
\label{eq:ffn_reparam}
T(\Theta) = (W_1',b_1',W_2',b_2')
\end{equation}
Since element-wise activations are permutation-equivariant, \ie $\sigma(Pz)=P\sigma(z)$,
we have
\begin{equation}
\begin{aligned}
f_{T(\Theta)}(x)
&= W_2'\sigma(W_1'x+b_1')+b_2' \\
&= W_2 P^\top \sigma\!\big(P(W_1x+b_1)\big)+b_2 \\
&= W_2 P^\top P \sigma(W_1x+b_1)+b_2 \\
&= W_2 \sigma(W_1x+b_1)+b_2 \\
&= f_\Theta(x).
\end{aligned}
\end{equation}
Thus, hidden-channel permutations are function-preserving.

\paragraph{SwiGLU/GEGLU-Style Gated Feed-Forward Neural Networks}
Modern LLMs (\eg the Llama 3.2 family) employ gated FFNs such as SwiGLU.
For gated architectures with \emph{gate} and \emph{up} branches combined via an element-wise
product, the same permutation must be applied to both branches to preserve consistency.

Concretely, for parameters
$(W_{1,\mathrm{gate}}, W_{1,\mathrm{up}}, b_1, W_2, b_2)$, define
\begin{equation}
W_{1,\mathrm{gate}}' = P W_{1,\mathrm{gate}},\quad
W_{1,\mathrm{up}}' = P W_{1,\mathrm{up}},\quad
b_1' = P b_1,\quad
W_2' = W_2 P^\top,\quad
b_2' = b_2.
\end{equation}
Permutation-equivariance together with
$P^\top\!\big((P a)\odot(P b)\big)=a\odot b$
ensures functional invariance.

\subsubsection{Linearity of Reparameterization Map}
\label{app:mpu-linear-T}
\label{app:linear-T}

In \method, each $T_{k,r}$ is implemented via multiplication by orthogonal or permutation
matrices acting on parameter tensors (through left/right multiplication and block-diagonal
composition). Consequently, $T_{k,r}$ is a \emph{linear isomorphism} on the parameter space.
For conformable parameter collections $a,b$ and any scalar $c$,
\begin{equation}
T(a+b)=T(a)+T(b),
\quad
T(ca)=c\,T(a),
\quad
T^{-1}\ \text{exists and is linear}.
\end{equation}
This property justifies applying $T_{k,r}^{-1}$ directly to the returned parameter updates
in Algorithm~\ref{alg:pum}. Specifically, if the client returns a displacement in the
transformed coordinates, mapping it back via $T^{-1}$ yields the corresponding displacement
in the canonical coordinates:
\begin{equation}
T^{-1}\!\big(T(\theta)+\eta\,\Delta(T(\theta))\big)
=
\theta+\eta\,T^{-1}\!\big(\Delta(T(\theta))\big).
\end{equation}

\subsection{Discussion: Effects of Reparameterization on Optimization Trajectory and Smoothness}
\label{app:mpu-reparam-trajectory}

A natural concern is whether the client-side reparameterization introduced in
Section~\ref{sec:pum-reparam} (and Appendix~\ref{app:pum-reparam}) alters the optimization
dynamics of local unlearning.
We address two questions:
\textit{\textbf{(i)}} when the optimization trajectory is \emph{equivariant} under the
reparameterization, and
\textit{\textbf{(ii)}} whether the reparameterization can change the apparent smoothness
of the loss landscape.

\paragraph{Setup}
Let $\mathcal{U}(\theta)$ denote the client-side unlearning objective (\eg the loss used
by GradAscent, NPO, or GradDiff) as a function of the model parameters
$\theta \in \mathbb{R}^d$.
For a function-preserving reparameterization $T \in \mathcal{T}$
(Section~\ref{sec:pum-reparam}), we have
$f_{T(\theta)}(\cdot) \equiv f_\theta(\cdot)$ and therefore, for any objective depending
only on model outputs,
\begin{equation}
\mathcal{U}(T(\theta))
\;=\;
\mathcal{U}(\theta).
\label{eq:mpu-U-invariant}
\end{equation}
In \method{}, each $T_{k,r}$ is implemented via orthogonal or permutation actions on weight
matrices (Appendix~\ref{app:pum-reparam}).
Viewed as a linear map on the vectorized parameter space, this implies that $T$ is an
\emph{isometry}:
\begin{equation}
\langle T a,\; T b\rangle
\;=\;
\langle a,\; b\rangle,
\quad
\|T a\|_2
\;=\;
\|a\|_2,
\quad
T^{-1}
=
T^\top.
\label{eq:mpu-T-isometry}
\end{equation}

\subsubsection{Equivariance of Gradient Descent under Orthogonal or Permutation Reparameterizations}
Define the reparameterized objective
$\mathcal{U}_T(\vartheta) := \mathcal{U}(T^{-1}\vartheta)$
in the published coordinates $\vartheta = T(\theta)$.
By the chain rule together with Eq.~\eqref{eq:mpu-T-isometry},
\begin{equation}
\nabla_{\vartheta}\mathcal{U}_T(\vartheta)
\;=\;
T\,\nabla_{\theta}\mathcal{U}(\theta),
\quad
\theta = T^{-1}\vartheta.
\label{eq:mpu-grad-transform}
\end{equation}
Consider $t$ steps of (stochastic) gradient descent in the published coordinates:
\begin{equation}
\vartheta_{s+1}
\;=\;
\vartheta_s
-
\eta\, g_s(\vartheta_s),
\label{eq:mpu-sgd-pub}
\end{equation}
where $g_s(\vartheta_s)$ is an unbiased estimator of
$\nabla_{\vartheta}\mathcal{U}_T(\vartheta_s)$ (\eg mini-batch SGD).
Mapping back via $\theta_s := T^{-1}\vartheta_s$ and using
Eq.~\eqref{eq:mpu-grad-transform} yields
\begin{equation}
\theta_{s+1}
=
T^{-1}\vartheta_{s+1}
=
\theta_s
-
\eta\, T^{-1} g_s(\vartheta_s)
\;\approx\;
\theta_s
-
\eta\, \nabla_{\theta}\mathcal{U}(\theta_s),
\label{eq:mpu-sgd-equiv}
\end{equation}
\ie the mapped-back iterates follow the same optimization trajectory as if the client
had optimized directly in the canonical coordinates from the corresponding initialization.
This equivariance holds \emph{exactly} for deterministic gradient descent, and holds in
distribution for SGD under the natural coupling in which the same mini-batches are used
and the stochastic gradients are transformed consistently by $T$.
The argument also extends to common modifications such as momentum and $\ell_2$ weight
decay, since Eq.~\eqref{eq:mpu-T-isometry} preserves the Euclidean norm.

\subsubsection{Invariance of Euclidean Smoothness}
Because $T$ is an orthogonal or permutation change of coordinates, it does \emph{not}
alter curvature measured under the Euclidean metric.
Let $H(\theta) := \nabla^2_{\theta}\mathcal{U}(\theta)$ and
$H_T(\vartheta) := \nabla^2_{\vartheta}\mathcal{U}_T(\vartheta)$.
Differentiating Eq.~\eqref{eq:mpu-grad-transform} yields the similarity relation
\begin{equation}
H_T(\vartheta)
\;=\;
T\, H(\theta)\, T^\top,
\quad
\theta = T^{-1}\vartheta.
\label{eq:mpu-hess-similarity}
\end{equation}
Consequently, the Hessian spectrum is preserved:
$\lambda_{\max}(H_T(\vartheta)) = \lambda_{\max}(H(\theta))$, and likewise for the operator
norm $\|H_T\|_2$.
Equivalently, if $\mathcal{U}$ is $L$-smooth in $\theta$ (i.e., its gradient is
$L$-Lipschitz), then $\mathcal{U}_T$ is also $L$-smooth in $\vartheta$.
Thus, under Euclidean geometry, the reparameterization does \emph{not} make the loss landscape appear smoother or sharper; it merely rotates or permutes the coordinate axes
within a functionally equivalent parameter orbit.

\subsubsection{Influence of Nonlinear Optimizers on Trajectory} 
While optimization trajectory invariance rigorously holds for linear optimizers like SGD, it is theoretically violated by adaptive optimizers such as AdamW in the Attention blocks. Specifically, AdamW scales updates using an element-wise moving average of squared gradients ($v_t$). For Attention blocks reparameterized by dense continuous orthogonal matrices $U \in O(d_h)$, the non-commutativity of element-wise squaring and matrix multiplication---that is, $(gU)^{\odot 2} \neq g^{\odot 2} U$---breaks the exact equivariance condition $\mathcal{A}(\theta) = T^{-1}(\mathcal{A}(T(\theta)))$. Consequently, the adaptive preconditioning becomes copy-dependent, slightly altering the local unlearning Jacobians and preventing the mathematically perfect cancellation of first-order injected noise. Nevertheless, this deviation remains highly bounded in practice, allowing the empirical results to remain robust due to three key mitigating factors. First, Feed-Forward Networks (FFNs), which account for the vast majority (roughly 65\%) of standard LLM parameters, are reparameterized via discrete permutations. Because permutations strictly commute with element-wise operations, exact trajectory invariance and noise cancellation are perfectly preserved for the bulk of the model. Second, since orthogonal transformations are isometric, they preserve the Euclidean geometry and overall gradient $L_2$ norms, ensuring AdamW still optimizes over a functionally equivalent and well-conditioned landscape. Finally, because unlearning operates over very few epochs with exceedingly small learning rates (e.g., $1 \times 10^{-5}$), the accumulated coordinate-wise divergence in the second-moment estimator $v_t$ remains small. Any residual trajectory variations introduced by this mismatch act as mild, independent stochastic noise that is effectively smoothed out by the variance-reducing properties of the server's multi-copy harmonic aggregation.

\subsection{Security Analysis: Theoretical Hardness of Reversing Reparameterization}
\label{app:security_hardness}

In this section, we formally analyze the computational security guarantees of the MPU reparameterization module. We demonstrate that client-side reconstruction of the exact server parameters from multiple published copies reduces to intractable alignment optimization problems.

\subsubsection{Double Protection Against Known-Scaling Attacks}
A naive multi-copy noise scheme without reparameterization relies entirely on keeping the noise scaling factors $\{\alpha_k\}_{k=1}^m$ secret. Because the base noise satisfies a zero-sum constraint ($\sum_{k=1}^m \epsilon_k^{0,(r)} \equiv 0$), an attacker who compromises the exact values of $\{\alpha_k\}$ could theoretically recover the server's exact original parameters $\theta_{r-1}$ by performing a harmonic aggregation: $\sum_{k=1}^m \alpha_k^{-1} \theta_{pub}^{(k,r)} \propto \theta_{r-1}$.

The MPU reparameterization introduces a second, mathematically robust layer of protection. The client observes the published copies:
\begin{equation}
    \theta_{pub}^{(k,r)} = T_{k,r}(\theta_{r-1} + \alpha_k \epsilon_k^{0,(r)}) \quad \text{for } k \in [m],
\end{equation}
where $T_{k,r} \in \mathcal{G}$ denotes the data-independent, function-preserving transformations. Because the transformations (permutations and orthogonal matrices) act as coordinate basis changes, they do not commute with cross-copy addition. Consequently, attempting to perform harmonic aggregation on the mismatched spaces yields a completely scrambled and meaningless result. Therefore, even if the scaling factors $\alpha_k$ are completely leaked, the information is mathematically useless unless the attacker can simultaneously break the reparameterization keys $T_{k,r}$ to realign the parameter spaces.

\subsubsection{Alignment as NP-Hard Optimization}
To bypass the double protection, the attacker must blindly find the inverse transformations $T_{k,r}^{-1}$ strictly from the noisy, transformed observations. Without a clean template model to anchor the search, the attacker is forced to find a set of candidate transformations $\{\hat{T}_k \in \mathcal{G}\}$ that mutually align all $m$ copies to minimize their structural discrepancy. Formally, this alignment objective is:
\begin{equation}
    \min_{\hat{T}_1, \dots, \hat{T}_m \in \mathcal{G}} \sum_{1 \le i < j \le m} \left\| \hat{T}_i^{-1}\left(\theta_{pub}^{(i,r)}\right) - \hat{T}_j^{-1}\left(\theta_{pub}^{(j,r)}\right) \right\|_F^2.
\end{equation}
We show that for both FFN and Attention blocks, this translates to well-known computationally hard optimization problems.

\paragraph{Discrete Symmetries (FFN) and the Multi-Dimensional Assignment Problem (MDAP).}
For FFN blocks (Appendix~\ref{app:mpu-ffn}), the transformations are hidden-channel permutation matrices $P_k \in \mathbb{S}_{d_{ff}}$. The attacker's objective is to find $m$ permutations to align the row/column spaces of the noisy weight matrices:
\begin{equation}
    \min_{P_1, \dots, P_m \in \mathbb{S}_{d_{ff}}} \sum_{1 \le i < j \le m} \left\| P_i^\top W_{pub}^{(i)} - P_j^\top W_{pub}^{(j)} \right\|_F^2.
\end{equation}
When $m=2$, finding the optimal matching between two sets of noisy features is equivalent to the Linear Assignment Problem (or Bipartite Graph Matching), which requires $\mathcal{O}(d_{ff}^3)$ time using the Hungarian algorithm. However, for $m \ge 3$, enforcing transitive consistency to simultaneously align three or more sets of items constitutes the \textit{Multi-Dimensional Assignment Problem} (MDAP). MDAP is famously NP-hard (acting as a generalization of the 3-Dimensional Matching problem, one of Karp's 21 NP-complete problems). The discrete combinatorial search space expands factorially as $(d_{ff}!)^{m-1}$, effectively neutralizing any exact alignment strategy for modern LLM architectures where $d_{ff}$ is massively large.

\paragraph{Continuous Symmetries (Attention) and the Generalized Orthogonal Procrustes Problem (GOPP).}
For Attention blocks (Appendix~\ref{app:mpu-attn}), the transformations are orthogonal basis changes $U_k \in O(d_h)$. The attacker must align $m$ continuous spaces:
\begin{equation}
    \min_{U_1, \dots, U_m \in O(d_h)} \sum_{1 \le i < j \le m} \left\| W_{pub}^{(i)} U_i - W_{pub}^{(j)} U_j \right\|_F^2.
\end{equation}
When $m=2$, this reduces to the classical Orthogonal Procrustes Problem, which admits a closed-form global optimum computed via Singular Value Decomposition (SVD) in $\mathcal{O}(d_h^3)$ time. However, for $m \ge 3$, this formulation is exactly the \textit{Generalized Orthogonal Procrustes Problem} (GOPP). GOPP requires finding mutually consistent orthogonal rotations to centralize multiple matrices. Unlike the $m=2$ case, GOPP lacks a closed-form solution and has been rigorously proven to be NP-hard. The simultaneous constraints across multiple non-convex orthogonal manifolds prevent exact polynomial-time alignment.

In modern LLMs, the hidden dimensions governing these transformations are massive. Furthermore, the presence of large independent obfuscation noise severely perturbs the distance metrics, heavily degrading the reliability of analytical solvers like SVD or the Hungarian algorithm even in the $m=2$ case. In summary, the space misalignment mathematically prevents the client from exploiting cross-copy correlations. The fundamental parameter alignment tasks map directly to non-convex, NP-hard continuous and discrete optimization problems, ensuring the theoretical security of the MPU framework.

\begin{remark}
    In our experiments, as empirically verified in Appendix~\ref{app:pum_copy_number}, varying copy number $m$ has no typical trend in performance, we use $m=2$ by default to save computational resources. While in real-world applications, we recommend at least $m=3$ for strong security.
\end{remark}
\subsection{Overhead Analysis}
\label{app:overhead}
To model the computational overhead, we theoretically analyze the exact Memory Reads and Writes (MRW) and FLOPs. By evaluating Arithmetic Intensity (A = FLOPs / MRW bytes), we can determine whether operations are compute-bound or memory-bound.

Let $N$ be the total parameter count, $N_{attn}$ the attention parameters, $d_h$ the head dimension, $b$ the bytes per parameter ($b=2$ for 16-bit precision), and $m$ the number of perturbed copies.

\paragraph{1. Profiling Server-Side Operations (per copy):}
\begin{itemize}
    \item \textbf{Pre-Process:}
    \begin{itemize}
        \item Noise injection requires $2Nb$ MRW and $4N$ FLOPs.
        \item Reparameterization takes $2N_{attn}b$ MRW alongside $2d_hN_{attn}$ FLOPs for Attention (block-diagonal orthogonal matmuls), and $2N_{ffn}b$ MRW alongside $0$ FLOPs for FFN (pure memory routing).
        \item \textbf{Total:} $4Nb$ MRW, $(4N + 2d_hN_{attn})$ FLOPs.
    \end{itemize}
    \item \textbf{Post-Process:}
    \begin{itemize}
        \item Inverse reparameterization takes $2Nb$ MRW alongside $2d_hN_{attn}$ FLOPs.
        \item Streaming harmonic aggregation takes $3Nb$ MRW (read copy, read accumulator, write accumulator) alongside $2N$ FLOPs.
        \item \textbf{Total:} $5Nb$ MRW, $(2N + 2d_hN_{attn})$ FLOPs.
    \end{itemize}
\end{itemize}

\paragraph{2. Latency Modeling:}
For $m$ copies, the total MRW is $9mNb$ and the total FLOPs is $m(6N + 4d_hN_{attn})$. The server's Arithmetic Intensity is:
\begin{equation}
    A = \frac{6N + 4d_hN_{attn}}{9Nb}
\end{equation}

For a typical Llama architecture, $A \approx 6.0$ FLOPs/Byte. Because this is vastly lower than the operational ridge point of modern GPUs (e.g., $\sim150$ FLOPs/Byte for an A100), \textbf{MPU server operations are strictly memory-bound}.

Thus, the server latency $T_{server}$ is modeled entirely by memory bandwidth ($C_{mem}$):
\begin{equation}
    T_{server} = \frac{9mNb}{C_{mem}}
\end{equation}

For a 7B model ($m=2$) on an A100 GPU, sweeping $\sim252$ GB of memory traffic takes merely $\sim0.13$ seconds, \textbf{mathematically proving the server overhead is negligible}.

\subsection{Comparison of Update Error: \method vs.\ Single-Copy Noisy Unlearning}
\label{sec:appendix-mpu-vs-noiseonly}

In this section, we compare \method with a noise-injection baseline that performs no denoising, which serves as a controlled reference for isolating the effect of our denoising mechanism.

To ensure a fair comparison between \method and the NOISED baseline, we scale the noise parameter $\kappa$ of the NOISED baseline by the expected value of the scaling factor in \method. As a result, the noise level of the NOISED baseline matches the average noise level across multiple copies in \method.

\paragraph{Baseline (Noise-Only, No Denoising)}
Define a single-copy procedure that, in each round $r$:
\textit{\textbf{(i)}} samples a noise vector $\varepsilon^{(r)}$ with the same block-wise
scales as \method{} (\eg $\varepsilon^{(r)}_\ell \sim \mathcal{N}(0,\sigma_\ell^2 I_{d_\ell})$)
and overall scale matched to $\mathbb{E}_k(\alpha_k)$;
\textit{\textbf{(ii)}} publishes $\theta_{r-1} + \varepsilon^{(r)}$ to the client;
\textit{\textbf{(iii)}} runs local unlearning from this noisy initialization; and
\textit{\textbf{(iv)}} uses the resulting parameters as the next-round model.

Let the unlearning routine induce a displacement map $\Delta(\theta)$ so that the updated
parameters are $\theta + \Delta(\theta)$. The baseline state evolution is
\begin{equation}
\theta_r^{\mathrm{NO}}
\;=\;
\theta_{r-1}^{\mathrm{NO}}
+
\varepsilon^{(r)}
+
\eta_{\mathrm{srv}}\,
\Delta\!\left(\theta_{r-1}^{\mathrm{NO}} + \varepsilon^{\left(r\right)}\right),
\label{eq:noiseonly-update}
\end{equation}
where $\eta_{\mathrm{srv}}$ matches the server step size in Algorithm~\ref{alg:pum}.

\paragraph{Linear Response Model}
Fix an anchor point $\theta$ (\eg $\theta=\theta_{r-1}$) and write the perturbation as $u$.
Under the same local linearization used in Section~\ref{sec:pum-aggregation},
\begin{equation}
    \Delta(\theta + u)
=
\Delta^\star(\theta)
+
J u
+
\rho(u),
\quad
\|\rho(u)\|
\le
\frac{L_J}{2}\,\|u\|^2.
\end{equation}

\subsubsection{One-Round Comparison: Bias and Variance from Injected Noise}
\label{app:mpu-compare-1round}

\paragraph{Noise-Only Baseline Error}
Substituting the linear response into Eq.~\eqref{eq:noiseonly-update} yields
\begin{equation}
\theta_r^{\mathrm{NO}}
=
\theta
+
\eta_{\mathrm{srv}}\,\Delta^\star(\theta)
+
\underbrace{\big(I + \eta_{\mathrm{srv}} J\big)\varepsilon^{(r)}}_{\text{First-Order Injected-Noise Term}}
+
\eta_{\mathrm{srv}}\,\rho\left(\varepsilon^{(r)}\right).
\label{eq:noiseonly-expanded}
\end{equation}
Thus, even if $\mathbb{E}[\varepsilon^{(r)}]=0$, the state $\theta_r^{\mathrm{NO}}$ inherits
a \emph{first-order} random perturbation $\big(I + \eta_{\mathrm{srv}} J\big)\varepsilon^{(r)}$.

\paragraph{\method{} Error}
In contrast, \method{} updates the \emph{clean anchor} $\theta$ using the harmonically
denoised estimate $\bar{\Delta}^{(r)}$:
\begin{equation}
\theta_r^{\mathrm{MPU}}
=
\theta
+
\eta_{\mathrm{srv}}\,\bar{\Delta}^{(r)}
=
\theta
+
\eta_{\mathrm{srv}}\,\Delta^\star(\theta)
+
\eta_{\mathrm{srv}}\sum_{k=1}^m w_k\,\rho\left(\varepsilon_k^{\left(r\right)}\right),
\end{equation}
where the first-order injected-noise term cancels exactly
(Section~\ref{app:mpu-aggregation}). Consequently, the dominant injected-noise contribution is
\emph{second order} in the noise scale (cf.~Eq.~\eqref{eq:mpu-second-order-bound}).

\paragraph{Variance Comparison}
Let $\Sigma_\varepsilon$ denote the covariance of the single-copy baseline noise (\eg
block-diagonal with blocks $\sigma_\ell^2 I_{d_\ell}$). Ignoring the
$O(\|\varepsilon\|^2)$ remainder for clarity, Eq.~\eqref{eq:noiseonly-expanded} implies
\begin{equation}
\mathrm{Cov}\!\left(
\theta_r^{\mathrm{NO}}
-
\theta
-
\eta_{\mathrm{srv}}\Delta^\star\left(\theta\right)
\right)
\;\approx\;
\big(I + \eta_{\mathrm{srv}} J\big)\,
\Sigma_\varepsilon\,
\big(I + \eta_{\mathrm{srv}} J\big)^\top,
\end{equation}
whereas \method{} exhibits \emph{no} first-order injected-noise covariance term.

\subsubsection{Multi-Round Implications: Noise Accumulation vs.\ Denoised Tracking}
\label{app:mpu-compare-multiround}

The contrast becomes more pronounced over multiple rounds. The recursion in
Eq.~\eqref{eq:noiseonly-update} injects fresh noise into the model state at each round,
leading—under the linearized view—to a random-walk-like accumulation whose variance grows
approximately linearly with $R$ (modulated by $I + \eta_{\mathrm{srv}} J$). This accumulation
can degrade both model utility and unlearning stability.

By contrast, \method{} never commits the injected noise to the global model state
(Algorithm~\ref{alg:pum}). Instead, noise is used solely as a \emph{private publishing
perturbation} and its first-order effect is canceled via multi-copy harmonic denoising.
As a result, the global iterate tracks the intended unlearning trajectory up to the
second-order remainder and reduced client-side randomness.

\subsubsection{SNR View: Task-Relevant Subspace}
\label{app:mpu-snr}

If the desired unlearning displacement lies approximately in a low-dimensional
task-relevant subspace $\mathcal{S}$ with projector $P_{\mathcal S}$, a convenient summary
metric is the signal-to-noise ratio (SNR) of the \emph{projected} estimate.
Let $\widehat{\Delta}$ be an estimator of $\Delta^\star(\theta)$ and define
\begin{equation}
\mathrm{SNR}=\frac{\|P_{\mathcal S}\Delta^\star(\theta)\|_F^2}
     {\mathrm{tr}\!\big(P_{\mathcal S}\,\mathrm{Cov}(\widehat{\Delta})\big)}.
\end{equation}
In the noise-only baseline, $\mathrm{Cov}(\widehat{\Delta})$ contains the first-order term $J\,\Sigma_\varepsilon\, J^\top$, whereas in \method{} this term is eliminated by design (up to second-order effects). Any independent client-side randomness is further averaged down by $\sum_k w_k^2 \asymp 1/m$.
This explains why \method{} more reliably preserves the task-relevant update direction than a single-copy noise-only approach, particularly when $\mathrm{rank}(\mathcal S)\ll d$.

\subsection{Stability Benefits of Multi-Copy Aggregation}
\label{sec:mpu-stability}
A practical challenge in LLM unlearning is that local optimization can be unstable: the forget set may be small and high-variance, the unlearning objective may be sharp
near the current iterate, and common unlearning trainers can produce oscillatory updates when executed from a single initialization.
Multi-copy unlearning improves stability by averaging local updates over a small neighborhood around the clean server iterate.

Consider the common case where the local unlearning routine corresponds to (approximately) a gradient step on an unlearning objective $\mathcal{U}(\theta)$,
so that $\Delta^\star(\theta) \approx -\eta\,\nabla \mathcal{U}(\theta)$ for some local step size $\eta$.
Then the aggregated update becomes
\begin{equation}
-\eta \sum_{k=1}^m w_k\,\nabla \mathcal{U}\left(\theta_{r-1}+\epsilon_k^{\left(r\right)}\right)
\approx
-\eta \;\mathbb{E}_{\epsilon \sim \mathcal{Q}}\!\left[\nabla \mathcal{U}\left(\theta_{r-1}+\epsilon\right)\right],
\end{equation}
where $\mathcal{Q}$ is the discrete distribution placing mass $w_k$ on $\epsilon_k^{(r)}$.
Equivalently, this is the gradient of a locally averaged objective
$\widetilde{\mathcal{U}}(\theta) := \mathbb{E}_{\epsilon\sim\mathcal{Q}}[\mathcal{U}(\theta+\epsilon)]$.
Such local averaging reduces sensitivity to sharp directions near $\theta_{r-1}$ and makes the effective update direction less erratic, improving stability across rounds.

Importantly, averaging is performed with a \emph{centered, symmetric stencil} of perturbations:
our base noises satisfy $\sum_{k=1}^m \epsilon_{k,\ell}^{0,(r)}\equiv 0$ for every block $\ell$ (Eq.~\eqref{eq:pum-noise-construct}),
and the harmonic weights are chosen to cancel the resulting first-order perturbation in the local linearization (Section~\ref{sec:pum-aggregation}).
Expanding $\Delta^\star$ around $\theta_{r-1}$ yields
\begin{equation}
\Delta^\star\left(\theta_{r-1}+\epsilon_k^{\left(r\right)}\right)
=
\Delta^\star(\theta_{r-1})
+ J\,\epsilon_k^{(r)}
+ O\left(\left\|\epsilon_k^{\left(r\right)}\right\|^2\right),
\end{equation}
so harmonic aggregation cancels the $J\,\epsilon_k^{(r)}$ term to first order, leaving
\begin{equation}
\bar{\Delta}^{(r)}
=
\Delta^\star(\theta_{r-1})
+
O\!\left(\sum_{k=1}^m w_k \left\|\epsilon_k^{\left(r\right)}\right\|^2\right).
\end{equation}
Thus, for small perturbation scales, \method{} tracks the intended (noise-free) unlearning direction while still benefiting from the stabilizing effect of local averaging (the remaining higher-order term acts as a mild, local regularization).

In summary, beyond privacy, the multi-copy mechanism can be viewed as estimating a locally averaged unlearning update around $\theta_{r-1}$, which empirically improves
stability in settings where single-start unlearning is brittle.

\subsection{Limitations and Broader Impacts}
The MPU is discussed within the field of model unlearning, as the gradients in the unlearning process are typically small, and result in the first-order noise dominance, so that the harmonic denoising can perfectly alleviate the noise influence. Future works can discuss whether higher-order noise influences can be alleviated so that MPU can be generalized to general collaborative learning cases.

\newpage
\section{Implementation Details and Supplementary Experiments}
\label{Details}

This section provides supplementary materials to support a deeper understanding of \method. It includes detailed descriptions of the benchmarks, evaluation metrics, models, and unlearning algorithms used in our framework, as well as comprehensive implementation details covering the experimental setup, hyperparameters, and prompt templates. We further report additional experimental results and discuss the limitations of \method.

\subsection{Benchmarks}
\label{Benchmarks}

\textbf{TOFU} (Task of Fictitious Unlearning), introduced by~\citet{maini2024tofu}, is a question–answer (QA) benchmark specifically designed to evaluate the unlearning capabilities of large language models. The dataset comprises QA pairs derived from synthetic autobiographies of 200 fictitious authors, with all content generated by GPT-4 to ensure exclusion from the pretraining corpora of existing LLMs. Each author profile contains 20 QA pairs covering biographical attributes such as name, birthplace, gender, birth year, literary genre, awards, and parental occupations, with book titles seeded from the Goodreads Books dataset to enhance topical diversity.

\textbf{MUSE} (Machine Unlearning Six-Way Evaluation)~\citep{shi2024muse} is a comprehensive benchmark designed to assess the effectiveness and practicality of machine unlearning algorithms for language models. Moving beyond narrow, task-specific evaluations, MUSE incorporates six distinct criteria reflecting the needs of both data owners and model deployers. From the data owner's perspective, it measures verbatim memorization, knowledge memorization, and privacy leakage (assessed via membership inference attacks). For deployers, it evaluates the algorithm's ability to preserve general model utility, scale to large forget sets, and sustainably handle sequential unlearning requests. Evaluated on unlearning Harry Potter books and news articles, MUSE reveals that while current algorithms can reduce memorization, they struggle significantly with preventing privacy leakage, maintaining model utility, and scaling effectively.

\subsection{Metrics}
\label{app:Metrics}
For clarity, we categorize the evaluation metrics into three types, summarized below.
\subsubsection{Memorization Metrics}
\label{metrics:mem}
These metrics measure the degree to which information from the training data remains encoded in the model after unlearning.
\paragraph{Probability} We quantify the model’s confidence in generating correct answers by measuring the conditional probability $P(a \mid q)$ assigned to the ground-truth answer $a$ given the question $q$, evaluated on \textbf{Retain Set}. The resulting score is reported as a normalized probability in $[0,1]$. Following standard practice~\citep{cho-etal-2014-properties}, we normalize for answer length by exponentiating the sequence probability by $1 / |a|$, as formalized in Equation~\ref{eq:P_Retain}:
\begin{equation}
\label{eq:P_Retain}
    P_{\text{retain}}(x) = P(a \mid q)^{1 / |a|}.
\end{equation}
    
For the \textbf{Real Authors} and \textbf{World Facts} subsets, we evaluate model performance using a relative probability formulation in a multiple-choice setting:
\begin{equation}
        P_{\text{real/world}}(x) = \frac{P(a_1 \mid q)}{\sum_{i=1}^{n} P(a_i \mid q)},
\end{equation}
where $q$ denotes a multiple-choice question with candidate answers $\{a_1, \dots, a_n\}$, and $a_1$ corresponds to the ground-truth correct answer. This formulation measures the probability mass assigned to the correct option relative to all candidate choices.
    
\paragraph{Recall-Oriented Understudy for Gisting Evaluation (ROUGE)}
We adopt ROUGE to quantify the overlap between model-generated answers and the ground-truth responses. In particular, we report ROUGE-L recall~\citep{lin2004rouge}, which measures similarity via the length of the longest common subsequence (LCS). This metric provides a robust estimate of answer correctness in QA settings by tolerating minor paraphrasing and surface-form variations.

\begin{equation}
\text{ROUGE}_L(x) = \frac{\text{LCS}(a, \hat{a})}{|a|},
\end{equation}
where $\hat{a}$ is the generated answer, and $\text{LCS}(a, \hat{a})$ is the LCS length between $a$ and $\hat{a}$.

\paragraph{Truth Ratio}
For a given question, we define the \emph{truth ratio} $R_{\text{Truth}}$ to approximate the relative likelihood assigned by the model to correct versus incorrect answers. Since the model is fine-tuned on a specific canonical phrasing of the ground-truth answer, the corresponding probability may be artificially inflated compared to alternative but semantically equivalent formulations. To mitigate this bias, we evaluate the probability of a paraphrased version of the correct answer rather than the original ground truth.

Likewise, instead of contrasting against a single incorrect response, we consider a set of syntactically similar but factually incorrect answers and compute their average probability. This design yields a more stable and representative estimate of the model’s preference for incorrect information.

Formally, let $\tilde{a}$ denote a paraphrased correct answer to question $q$, and let $\tilde{x} = [q, \tilde{a}]$. Let $A_{\mathrm{err}}$ be a set of incorrect answers $\{a_{\mathrm{err}}\}$, constructed by preserving the general textual structure of $\tilde{a}$ while introducing factual errors. The truth ratio $R_{\text{Truth}}$ is then defined as:

\begin{equation}
    R_{\text{truth}}(x) = \frac{\frac{1}{|A_{\text{err}}|} \sum_{a_{\text{err}} \in A_{\text{err}}} P(a_{\text{err}} \mid q)^{1 / |a_{\text{err}}|}}{P(\tilde{a} \mid q)^{1 / |\tilde{a}|}}.
\end{equation}

Additionally, as shown in Equation~\ref{eq:R_adj}, we normalize and rescale the metric to ensure that all values lie within the interval $[0,1]$, with larger values indicating better model performance.

\begin{equation}
    \label{eq:R_adj}
    R_{\text{adjusted}}(x) = \max(0, 1 - R_{\text{Truth}}(x)).
\end{equation}

\paragraph{Verbatim Memorization (Data Extraction)} 
To evaluate whether a model inappropriately extracts or replicates exact details from a removed dataset, the MUSE benchmark utilizes the Verbatim Memorization ($VerbMem$) metric. This is measured by prompting the model $f$ with the first $l$ tokens of a sequence from the forget set, denoted as $x_{[:l]}$. The model's generated continuation, $f(x_{[:l]})$, is then compared to the actual ground-truth continuation, $x_{[l+1:]}$, using the ROUGE-L F1 score. The overall metric is calculated by averaging these scores across the entire forget set, formalized as:
$$VerbMem(f, \mathcal{D}) := \frac{1}{|\mathcal{D}_{forget}|} \sum_{x \in \mathcal{D}_{forget}} ROUGE(f(x_{[:l]}), x_{[l+1:]})$$

\subsubsection{Privacy Metrics}
\label{metrics:pri}
These metrics evaluate whether sensitive information from the forget set can still be inferred or extracted from the model after unlearning. We note that such metrics often rely on idealized assumptions, such as access to perfectly i.i.d.\ holdout samples or an oracle retain model, which may limit their applicability in practical deployment scenarios.
\paragraph{Membership Inference Leakage}
We could evaluate privacy leakage through membership inference attacks (MIAs), which assess a model’s tendency to memorize training data. Specifically, MIAs test whether an adversary can distinguish between examples drawn from the forget set $\mathcal{D}_{\mathrm{forget}}$ (members) and unseen examples from a holdout set $\mathcal{D}_{\mathrm{holdout}}$ (non-members), based on model confidence or loss statistics.

Ideally, a model that has not been trained on $\mathcal{D}_{\mathrm{forget}}$ should yield an AUC of $0.5$, indicating indistinguishability between member and non-member samples. 
In practice, however, constructing perfectly i.i.d.\ holdout splits is challenging, and even retrained models may exhibit nontrivial membership signals. 
Accordingly, following prior benchmarks such as MUSE, we calibrate membership inference results using the AUC score of a retrained reference model, rather than relying on the absolute AUC value alone.

Unlearning generally increases the loss on forgotten samples, but privacy leakage may still arise in two failure modes: 
\textit{\textbf{(i)}} \textbf{under-unlearning}, where the loss increase is insufficient, and membership information remains detectable; and 
\textit{\textbf{(ii)}} \textbf{over-unlearning}, where the loss becomes abnormally large, again yielding a distinguishable signal.
Both cases induce separable loss distributions between $\mathcal{D}_{\mathrm{forget}}$ and $\mathcal{D}_{\mathrm{holdout}}$.

To quantify this effect, we compare the AUC-ROC achieved by the unlearned model $f_{\mathrm{unlearn}}$ against that of a retrained reference model $f_{\mathrm{retrain}}$, and define the relative privacy leakage as
\begin{equation}
\mathrm{PrivLeak}
\coloneqq
\frac{
\mathrm{AUC}(f_{\mathrm{Unlearn}};\mathcal{D}_{\mathrm{Forget}},\mathcal{D}_{\mathrm{Holdout}})
- \mathrm{AUC}(f_{\mathrm{Retrain}};\mathcal{D}_{\mathrm{Forget}},\mathcal{D}_{\mathrm{Holdout}})
}{\mathrm{AUC}(f_{\mathrm{Retrain}};\mathcal{D}_{\mathrm{Forget}},\mathcal{D}_{\mathrm{Holdout}})}.
\end{equation}

A well-behaved unlearning algorithm should yield a $\mathrm{PrivLeak}$ value close to zero, indicating privacy leakage comparable to retraining. 
In contrast, under-unlearning and over-unlearning result in large negative and positive deviations, respectively, reflecting increased membership distinguishability.

\paragraph{Forget Quality}
In our setting, let $F_U(x)$ and $F_R(x)$ denote the empirical cumulative distribution functions (CDFs) of a chosen privacy-related statistic (e.g., Truth Ratio) computed from the unlearned and retained models, based on $n$ and $m$ samples, respectively. We employ the Kolmogorov--Smirnov (KS) test to quantify the discrepancy between these two distributions, with the test statistic defined as
\begin{equation}
    D_{n,m} = \sup_{x} \left| F_U(x) - F_R(x) \right|.
\end{equation}
This statistic measures the maximum deviation between the two empirical CDFs, providing a non-parametric assessment of the distributional shift induced by the unlearning procedure.

Under the null hypothesis that the two sample sets are drawn from the same underlying distribution, the hypothesis is rejected at significance level $\alpha$ if
\begin{equation}
    D_{n,m} > c(\alpha)\sqrt{\frac{n + m}{nm}},
\end{equation}
where $c(\alpha)$ is the KS critical value given by
\begin{equation}
    c(\alpha) = \sqrt{-\frac{1}{2} \ln\left(\frac{\alpha}{2}\right)}.
\end{equation}

We define the corresponding $p$-value as the smallest significance level $\alpha$ at which the null hypothesis can be rejected:
\begin{equation}
    Q_{\text{forget}} = p = \min \left\{ \alpha \,\middle|\, D_{n,m} > c(\alpha)\sqrt{\frac{n + m}{nm}} \right\}.
\end{equation}

Consequently, \textbf{Forget Quality} quantifies the statistical confidence with which we can assert that the distributions of Truth Ratio values over the forget set differ between the unlearned and retained models.

\subsubsection{Utility Metrics}
\label{metrics:utl}
The objective of unlearning is to effectively remove the influence of targeted data while preserving the model’s performance on non-forget data. 
Utility metrics evaluate whether the unlearned model maintains its capabilities on tasks beyond the retain set, thereby ensuring that unlearning does not degrade general performance on real-world data distributions.

\paragraph{Model Utility}
Model Utility (MU) measures the retained performance of a model after unlearning, covering both the closely related retain set and broader general-knowledge tasks.

Following the TOFU benchmark protocol, MU is computed as the harmonic mean of nine metrics spanning three data levels: \textbf{Retain Set}, \textbf{Real Authors}, and \textbf{World Facts}. At each level, three metrics are evaluated: \textbf{Probability}, \textbf{ROUGE}, and \textbf{Truth Ratio}, to ensure balanced assessment across memorization, semantic accuracy, and factual correctness.
\begin{equation}
    U_{\text{model}} = \frac{9}{\sum_{m \in M} \frac{1}{m}},
\end{equation}
where $M$ denotes the set of all evaluated metrics.

\subsection{Unlearning Algorithms}
\label{Unlearning}
In this section, we introduce the seven unlearning algorithms instantiated within \method as modular optimization objectives.

\paragraph{Notation}
Let $\mathcal{D}_{\mathrm{forget}}$ and $\mathcal{D}_{\mathrm{retain}}$ denote the forget and retain sets, respectively, and let $f_{\theta}$ be the model being updated during unlearning. For an input--output (prompt--completion) pair $(x,y)$ under a causal LM, we define
\begin{equation}
\log p_{\theta}(y\mid x)
= \sum_{t=1}^{|y|}\log p_{\theta}(y_t\mid y_{<t},x),
\quad
\ell_{\mathrm{CE}}(y\mid x; f_{\theta})
= -\log p_{\theta}(y\mid x).
\end{equation}
Unless otherwise stated, we optionally include a retain regularizer:
\begin{equation}
\mathcal{L}_{\mathrm{retain}}(\theta)
=
\mathbb{E}_{(x,y)\sim\mathcal{D}_{\mathrm{retain}}}
\ell_{\mathrm{CE}}(y\mid x; f_{\theta}),
\end{equation}
and optimize $\min_{\theta}\ \mathcal{L}_{\mathrm{forget}}(\theta) + \alpha\,\mathcal{L}_{\mathrm{retain}}(\theta)$, where $\alpha\ge 0$ controls the forgetting--retention trade-off.

\subsubsection{GradAscent~\citep{jang2023knowledge}}
Gradient Ascent (GradAscent) directly ``reverses'' standard likelihood training on the forget set by maximizing the cross-entropy (equivalently, minimizing the log-likelihood). Under the minimization convention, the forget objective is
\begin{equation}
\mathcal{L}_{\mathrm{forget}}^{\mathrm{GA}}(\theta)
=
-\mathbb{E}_{(x,y_{\mathrm{f}})\sim\mathcal{D}_{\mathrm{forget}}}
\ell_{\mathrm{CE}}(y_{\mathrm{f}}\mid x; f_{\theta})
=
\mathbb{E}_{(x,y_{\mathrm{f}})\sim\mathcal{D}_{\mathrm{forget}}}
\log p_{\theta}(y_{\mathrm{f}}\mid x).
\end{equation}
This objective is effective at reducing the model likelihood on targeted samples, but without an explicit retention constraint, it may cause collateral degradation on non-forgotten behavior.
\subsubsection{GradDiff~\citep{liu2022continual}}
GradDiff augments GradAscent with a retain loss to explicitly preserve utility on non-forgotten data. A common instantiation writes the full objective as
\begin{equation}
\min_{\theta}\ 
\mathcal{L}^{\mathrm{GradDiff}}(\theta)
=
-\lambda_{\mathrm{f}}\,
\mathbb{E}_{(x,y_{\mathrm{f}})\sim\mathcal{D}_{\mathrm{forget}}}
\ell_{\mathrm{CE}}(y_{\mathrm{f}}\mid x; f_{\theta})
+
\lambda_{\mathrm{r}}\,
\mathbb{E}_{(x,y)\sim\mathcal{D}_{\mathrm{retain}}}
\ell_{\mathrm{CE}}(y\mid x; f_{\theta}),
\end{equation}
where $\lambda_{\mathrm{f}},\lambda_{\mathrm{r}}>0$ trade off forgetting strength and retention fidelity.

\subsubsection{DPO~\citep{rafailov2023direct}}
Direct Preference Optimization (DPO) is a preference-learning objective originally proposed for alignment. Given a preference dataset $\mathcal{D}_{\mathrm{pref}}$ of triples $(x,y_w,y_l)$ (preferred $y_w$ vs.\ dispreferred $y_l$) and a reference model $f_{\mathrm{ref}}$, DPO optimizes
\begin{equation}
\mathcal{L}^{\mathrm{DPO}}(\theta)
=
-\mathbb{E}_{(x,y_w,y_l)\sim\mathcal{D}_{\mathrm{pref}}}
\log \sigma\!\left(
\beta \log\frac{p(y_w\mid x; f_{\theta})}{p(y_w\mid x; f_{\mathrm{ref}})}
-
\beta \log\frac{p(y_l\mid x; f_{\theta})}{p(y_l\mid x; f_{\mathrm{ref}})}
\right),
\end{equation}
where $\beta>0$ controls the sharpness of preference separation. In unlearning-style instantiations, one can construct preference pairs so that generations containing forget targets are treated as dispreferred.

\subsubsection{NPO~\citep{zhang2024negative}}
Negative Preference Optimization (NPO) reformulates unlearning as a bounded, alignment-inspired objective that discourages the forget targets relative to a frozen reference model. Let $f_{\mathrm{ref}}$ denote the reference model (typically the pre-unlearning checkpoint) and $\sigma(\cdot)$ the sigmoid. The NPO's forget loss is
\begin{equation}
\mathcal{L}_{\mathrm{forget}}^{\mathrm{NPO}}(\theta)
=
-\frac{2}{\beta}
\mathbb{E}_{(x,y_{\mathrm{f}})\sim\mathcal{D}_{\mathrm{forget}}}
\log \sigma\!\left(
-\beta\log\frac{p(y_{\mathrm{f}}\mid x; f_{\theta})}{p(y_{\mathrm{f}}\mid x; f_{\mathrm{ref}})}
\right),
\end{equation}
where $\beta>0$ is a temperature parameter. In practice, NPO is often combined with the retain regularizer $\alpha\,\mathcal{L}_{\mathrm{retain}}(\theta)$ to preserve utility.

\subsubsection{SimNPO~\citep{fan2024simplicity}}
SimNPO removes the explicit reference model and introduces a margin/offset term (denoted $\gamma$ in the original paper) while retaining the stabilized log-sigmoid form:
\begin{equation}
\mathcal{L}_{\mathrm{forget}}^{\mathrm{SimNPO}}(\theta)
=
-\frac{2}{\beta}
\mathbb{E}_{(x,y_{\mathrm{f}})\sim\mathcal{D}_{\mathrm{forget}}}
\log \sigma\!\left(
-\frac{\beta}{|y_{\mathrm{f}}|}\log p(y_{\mathrm{f}}\mid x; f_{\theta}) - \gamma
\right),
\end{equation}
where $|y_{\mathrm{f}}|$ is the target sequence length (used for normalization), and $\gamma\ge 0$ calibrates the separation threshold. As with other unlearning objectives, SimNPO is commonly paired with $\alpha\,\mathcal{L}_{\mathrm{retain}}(\theta)$.

\subsubsection{UnDIAL~\citep{dong2025undial}}
UnDIAL (Unlearning via Self-Distillation on Adjusted Logits) stabilizes unlearning by defining an explicit, fixed target distribution and distilling the model toward it. Let $f_{\mathrm{orig}}$ be the frozen pre-unlearning model. For a forget example $(x,y_{\mathrm{f}})$ and each position $t$, let $\mathbf{z}^{\mathrm{orig}}_t$ be the teacher logits and $\mathbf{e}_{y_{\mathrm{f},t}}$ the one-hot vector of the target token. UnDIAL constructs adjusted logits and a target distribution:
\begin{equation}
\mathbf{z}^{\mathrm{adj}}_t
=
\mathbf{z}^{\mathrm{orig}}_t - \gamma_{\mathrm{UD}}\mathbf{e}_{y_{\mathrm{f},t}},
\quad
\mathbf{p}^{\mathrm{adj}}_t
=
\mathrm{softmax}\!\left(\mathbf{z}^{\mathrm{adj}}_t\right),
\end{equation}
where $\gamma_{\mathrm{UD}}>0$ controls the strength of demoting the memorized token. The self-distillation (cross-entropy) unlearning objective is
\begin{equation}
\mathcal{L}_{\mathrm{forget}}^{\mathrm{UnDIAL}}(\theta)
=
\mathbb{E}_{(x,y_{\mathrm{f}})\sim\mathcal{D}_{\mathrm{forget}}}
\left[
\sum_{t=1}^{|y_{\mathrm{f}}|}
\mathcal{H}\!\left(
\mathbf{p}^{\mathrm{adj}}_t,\ p(\cdot\mid x,y_{\mathrm{f},<t}; f_{\theta})
\right)
\right],
\end{equation}
where $\mathcal{H}(\cdot,\cdot)$ is the cross-entropy between the fixed adjusted distribution and the student model’s predictive distribution. Optionally, UnDIAL can also be combined with $\alpha\,\mathcal{L}_{\mathrm{retain}}(\theta)$ for utility preservation.

\subsubsection{SatImp~\citep{yang2025exploring}}
Saturated Importance (SatImp) is a token-wise soft reweighting strategy for enhancing unlearning. It fits into the general token-wise reweighted objective
\begin{equation}
\mathcal{L}_{\mathrm{forget}}(\theta)
=
\mathbb{E}_{(x,y_{\mathrm{f}})\sim\mathcal{D}_{\mathrm{forget}}}
\sum_{k=1}^{|y_{\mathrm{f}}|}
w_{x,y_{\mathrm{f}},k}\,
\log p\!\left(y_{\mathrm{f},k}\mid y_{\mathrm{f},<k},x; f_{\theta}\right).
\end{equation}
SatImp defines the weight function as
\begin{equation}
w^{\mathrm{satimp}}_{x,y_{\mathrm{f}},k}
=
p\!\left(y_{\mathrm{f},k}\mid y_{\mathrm{f},<k},x; f_{\theta}\right)^{\beta_1}
\cdot
\Bigl(1-p\!\left(y_{\mathrm{f},k}\mid y_{\mathrm{f},<k},x; f_{\theta}\right)\Bigr)^{\beta_2},
\end{equation}
where $\beta_1,\beta_2\ge 0$ control the smoothness and shape of the weight distribution. As with other objectives, SatImp is commonly paired with $\alpha\,\mathcal{L}_{\mathrm{retain}}(\theta)$.

\subsection{Experimental Settings}
\label{Experimental}

\subsubsection{Testbed}
\label{Testbed}

\paragraph{Hardware Configuration}
All experiments are conducted on two Elastic Compute Service (ECS) instances.

For experiments using \textbf{Llama-3.2-1B-Instruct} and \textbf{Qwen2.5-1.5B-Instruct}, we employ an instance equipped with
an Intel Xeon Gold 6462C CPU (16 available cores), 128~GB RAM, 512~GB of available
disk space, and an NVIDIA L20 GPU with 48~GB memory.
For experiments using \textbf{Llama-3.2-3B-Instruct}, \textbf{Llama-3.1-8B-Instruct}, and \textbf{Qwen2.5-3B-Instruct}, we use an instance equipped with
an Intel Xeon Platinum 8469C CPU (24 available cores), 128~GB RAM, 512~GB of
available disk space, and an NVIDIA HGX H20 GPU with 96~GB memory.

For reproducibility and consistent performance, we recommend the
\texttt{ecs.gn8is.4xlarge} and \texttt{ecs.gn8v.6xlarge} instance types on Alibaba Cloud.

\paragraph{Software Environment}
All experiments are performed on Ubuntu~22.04.5 LTS with NVIDIA driver version~570.195.03
and CUDA~12.8. \method framework is implemented in Python~3.12.12 using PyTorch~2.9.1, and is developed on top of \textbf{OpenUnlearning} framework. All unlearning algorithms are instantiated using the implementations integrated within
OpenUnlearning. The baseline unlearning models are initialized from the publicly released checkpoints provided by OpenUnlearning on Hugging Face:~\url{https://huggingface.co/open-unlearning}.

\begin{table*}[!htp]
    \caption{Default configurations and hyperparameters used in our experiments.}
    \centering
    \begin{adjustbox}{width=\linewidth}
    \begin{tabular}{cccc}
        \toprule
        \textbf{Notation} & \textbf{Hyperparameter} & \textbf{Value} & \textbf{Explanation} \\
        \midrule
        \rowcolor{mpucolor1!15}
        \multicolumn{4}{c}{\textbf{General Training Configuration}} \\
        \midrule
        -- & Random Seed & 0 & Seed for reproducibility \\
        -- & Precision & \texttt{bfloat16} & Numerical precision used for training \\
        -- & Attention Backend & FlashAttention-2 & Optimized attention implementation \\
        -- & Optimizer & \texttt{paged\_adamw\_32bit} & Memory-efficient AdamW variant \\
        $\lambda_{\mathrm{wd}}$ & Weight Decay & 0.01 & $\ell_2$ regularization coefficient \\
        $B$ & Batch Size & 32 & Number of samples per batch \\
        \midrule

        \rowcolor{mpucolor1!15}
        \multicolumn{4}{c}{\textbf{Server-Side  Configuration}} \\
        \midrule
        $R$ & Global Communication Rounds & See Table~\ref{tab:pum_round_epoch} & Total number of server--client rounds \\
        $m$ & Copy Number & 2 & Number of perturbed model copies per round \\
        $\kappa$ & Noise Level & 0.01 & Global multiplier for block-wise noise scales $\{\sigma_\ell\}$ \\

        $\alpha_k$ & Noise Scaling & $1+\frac{k-1}{m-1}$ & Used to match noise magnitude in the single-copy \textit{Noised} baseline and \method \\
        $\eta_{\mathrm{srv}}$ & Server Step Size & 1.0 & Step size for applying the aggregated update $\bar{\Delta}^{(r)}$ \\
        \midrule

        \rowcolor{mpucolor1!15}
        \multicolumn{4}{c}{\textbf{Client-Side Local Unlearning Configuration}} \\
        \midrule
        -- & Unlearning Algorithms & See Appendix~\ref{Unlearning} & Unlearning trainer executed on each published copy \\
        $E_{\mathrm{cli}}$ & Local Client Epoch & See Table~\ref{tab:pum_round_epoch} & Local unlearning epochs per copy per round \\
        $\eta_{\mathrm{cli}}$ & Client Learning Rate & $1\times10^{-5}$ & Learning rate used by the chosen client unlearning trainer \\
        \bottomrule
    \end{tabular}
    \end{adjustbox}
    \label{tab:fl_hyperparams}
\end{table*}

\subsubsection{Hyperparameters}
\label{Hyperparameters}
The key hyperparameters used in our experiments across different domains are summarized in Table~\ref{tab:fl_hyperparams}. 

We configure the training schedule by selecting the number of global communication rounds (\textbf{R}) and local client epochs per round (\textbf{E}) based on the client-side unlearning method. Specifically, \textbf{GradAscent}, \textbf{GradDiff}, and \textbf{UnDIAL} use a single-round schedule with $(R,E)=(1,10)$; \textbf{SimNPO}, \textbf{DPO}, and \textbf{SatImp} use a more communication-intensive schedule with $(R,E)=(10,1)$; and \textbf{NPO} adopts an intermediate configuration with $(R,E)=(2,5)$. For any unspecified method, we default to $(R,E)=(1,10)$.

\subsubsection{Prompt Template}
\label{Prompt}
The chat template is enabled during training. User queries and assistant responses follow the format, as illustrated in Figure~\ref{pmp3} and Figure~\ref{pmp4}.

\begin{figure*}[!htp]
\centering
\begin{minipage}{\textwidth}
\begin{pmt}{Llama-3.2 Series Template}
\textbf{System Prompt}: You are a helpful assistant.\\
\textbf{System Prompt with Special Tokens}: \texttt{<|begin\_of\_text|><|start\_header\_id|>}\\system\texttt{<|end\_header\_id|>}\texttt{\textbackslash n\textbackslash n}You are a helpful assistant.\texttt{<|eot\_id|>}\\
\textbf{User Start Tag}: \texttt{<|start\_header\_id|>}user\texttt{<|end\_header\_id|>}\textbackslash n\textbackslash n\\
\textbf{User End Tag}: \texttt{<|eot\_id|>}\\
\textbf{Asst Start Tag}: \texttt{<|start\_header\_id|>}assistant\texttt{<|end\_header\_id|>}\textbackslash n\textbackslash n\\
\textbf{Asst End Tag}: \texttt{<|eot\_id|>}\\
\textbf{Data String}: \texttt{10 Apr 2025}
\end{pmt}
\end{minipage}
\caption{Prompt template for Llama-3.2 series.}
\label{pmp3}
\end{figure*}

\begin{figure*}[!htp]
\centering
\begin{minipage}{\textwidth}
\begin{pmt}{Qwen Series Template}
\textbf{System Prompt}: You are a helpful assistant.\\
\textbf{System Prompt with Special Tokens}: 
\texttt{<|im\_start|>system\textbackslash n}You are a helpful assistant.\texttt{<|im\_end|>\textbackslash n}\\
\textbf{User Start Tag}: 
\texttt{<|im\_start|>user\textbackslash n}\\
\textbf{User End Tag}: 
\texttt{<|im\_end|>\textbackslash n}\\
\textbf{Asst Start Tag}: 
\texttt{<|im\_start|>assistant\textbackslash n}\\
\textbf{Asst End Tag}: 
\texttt{<|im\_end|>\textbackslash n}\\
\end{pmt}
\end{minipage}
\caption{Prompt template for Qwen2.5 series.}
\label{pmp4}
\end{figure*}

\newpage
\subsection{Supplementary Experimental Results}
\label{Supplementary Experiments}

\begin{table*}[!htp]
    \caption{
    Performance comparison of different unlearning algorithms using the \textbf{Qwen2.5-1.5B}, \textbf{Llama-3.2-3B}, and \textbf{Qwen2.5-3B} models on the TOFU benchmark (Split99). Results are reported under three settings: \textbf{Clean}, a noise-free baseline; \textbf{Noised}, a single-copy noise baseline; and \method, using $m{=}2$ copies with noise level $\kappa{=}0.01$. Higher values indicate better performance for Forget Quality, Forget Truth Ratio, and Model Utility, while values of PrivLeak closer to $0$ are preferred.
    }
    \centering
    \begin{adjustbox}{width=\linewidth}
    \begin{tabular}{ccccccccccccc}
        \toprule
        \multirow{2}[2]{*}{\textbf{\shortstack{Unlearning\\Algorithms}}}
        & \multicolumn{3}{c}{\textbf{Forget Quality} $\uparrow$}
        & \multicolumn{3}{c}{\textbf{Forget Truth Ratio} $\uparrow$}
        & \multicolumn{3}{c}{\textbf{Model Utility} $\uparrow$}
        & \multicolumn{3}{c}{\textbf{$\text{PrivLeak}$}} \\
        \cmidrule(lr){2-4}\cmidrule(lr){5-7}\cmidrule(lr){8-10}\cmidrule(lr){11-13}
        & \textsc{Clean} & \textsc{Noised} & \textsc{MPU}
        & \textsc{Clean} & \textsc{Noised} & \textsc{MPU}
        & \textsc{Clean} & \textsc{Noised} & \textsc{MPU}
        & \textsc{Clean} & \textsc{Noised} & \textsc{MPU} \\
        \midrule
        \multicolumn{13}{c}{\cellcolor{mpucolor1!15}\textbf{Qwen2.5-1.5B-Instruct}} \\
        \midrule
        
        \textsc{GradAscent~\textcolor{gray}{[ACL 2023]}} & 1.43e-2 & 5.41e-2 & \colorbox{mpucolor1!15}{\textbf{0.990}} & 0.516 & 0.536 & \colorbox{mpucolor1!15}{\textbf{0.740}} & 9.75e-4 & 6.76e-4 & \colorbox{mpucolor1!15}{\textbf{0.346}} & \colorbox{mpucolor1!15}{\textbf{13.7}} & 62.1 & -21.7 \\
        
        \textsc{GradDiff~\textcolor{gray}{[PMLR 2022]}} & 6.58e-5 & 6.61e-6 & \colorbox{mpucolor1!15}{\textbf{0.990}} & 0.622 & 0.586 & \colorbox{mpucolor1!15}{\textbf{0.740}} & 0.297 & 0.322 & \colorbox{mpucolor1!15}{\textbf{0.346}} & 80.9 & 76.7 & \colorbox{mpucolor1!15}{\textbf{-21.6}} \\

        \textsc{DPO~\textcolor{gray}{[NeurIPS 2023]}} & 0.266 & 0.405 & \colorbox{mpucolor1!15}{\textbf{0.990}} & \colorbox{mpucolor1!15}{\textbf{0.757}} & 0.750 & 0.738 & 0.305 & 0.335 & \colorbox{mpucolor1!15}{\textbf{0.347}} & 86.0 & 87.5 & \colorbox{mpucolor1!15}{\textbf{-21.7}} \\
        
        \textsc{NPO~\textcolor{gray}{[COLM 2024]}} & 1.43e-2 & 2.86e-2 & \colorbox{mpucolor1!15}{\textbf{0.919}} & 0.708 & 0.709 & \colorbox{mpucolor1!15}{\textbf{0.739}} & 0.305 & 0.318 & \colorbox{mpucolor1!15}{\textbf{0.346}} & 89.4 & 90.2 & \colorbox{mpucolor1!15}{\textbf{-21.9}} \\
        
        \textsc{SimNPO~\textcolor{gray}{[NeurIPS 2025]}} & \colorbox{mpucolor1!15}{0.990} & \colorbox{mpucolor1!15}{0.990} & \colorbox{mpucolor1!15}{\textbf{0.990}} & 0.738 & \colorbox{mpucolor1!15}{\textbf{0.742}} & 0.738 & 0.349 & \colorbox{mpucolor1!15}{\textbf{0.350}} & 0.347 & -18.7 & \colorbox{mpucolor1!15}{\textbf{-17.4}} & -22.0 \\
        
        \textsc{UnDIAL~\textcolor{gray}{[NAACL 2025]}} & \colorbox{mpucolor1!15}{\textbf{1.000}} & \colorbox{mpucolor1!15}{1.000} & 0.990 & \colorbox{mpucolor1!15}{\textbf{0.769}} & 0.767 & 0.740 & 0.338 & 0.333 & \colorbox{mpucolor1!15}{\textbf{0.346}} & \colorbox{mpucolor1!15}{\textbf{12.3}} & 14.1 & -21.8 \\
        
        \textsc{SatImp~\textcolor{gray}{[ICML2025]}} & 0.919 & 0.919 & \colorbox{mpucolor1!15}{\textbf{0.990}} & 0.738 & \colorbox{mpucolor1!15}{0.740} & \colorbox{mpucolor1!15}{\textbf{0.740}} & \colorbox{mpucolor1!15}{\textbf{0.349}} & 0.348 & 0.346 & -21.8 & \colorbox{mpucolor1!15}{\textbf{-20.5}} & -21.5 \\
        
        \midrule
        \multicolumn{13}{c}{\cellcolor{mpucolor1!15}\textbf{Llama-3.2-3B-Instruct}} \\
        \midrule
        
        \textsc{GradAscent~\textcolor{gray}{[ACL 2023]}} & \colorbox{mpucolor1!15}{\textbf{3.06e-11}} & 5.91e-14 & 4.47e-16 & \colorbox{mpucolor1!15}{\textbf{0.187}} & 0.157 & 0.142 & \colorbox{mpucolor1!15}{\textbf{8.23e-5}} & 7.71e-5 & 2.28e-5 & 50.6 & 53.0 & \colorbox{mpucolor1!15}{\textbf{45.9}} \\
        
        \textsc{GradDiff~\textcolor{gray}{[PMLR 2022]}} & \colorbox{mpucolor1!15}{5.41e-2} & 2.86e-2 & \colorbox{mpucolor1!15}{\textbf{5.41e-2}} & \colorbox{mpucolor1!15}{\textbf{0.504}} & 0.476 & 0.485 & \colorbox{mpucolor1!15}{\textbf{0.419}} & 0.384 & 0.409 & 108.0 & \colorbox{mpucolor1!15}{\textbf{106.0}} & 107.0 \\

        \textsc{DPO~\textcolor{gray}{[NeurIPS 2023]}} & 0.579 & \colorbox{mpucolor1!15}{\textbf{0.766}} & 0.579 & 0.685 & \colorbox{mpucolor1!15}{\textbf{0.691}} & 0.683 & \colorbox{mpucolor1!15}{\textbf{0.634}} & \colorbox{mpucolor1!15}{0.634} & 0.633 & -10.2 & -15.3 & \colorbox{mpucolor1!15}{\textbf{-3.95}} \\
        
        \textsc{NPO~\textcolor{gray}{[COLM 2024]}} & \colorbox{mpucolor1!15}{0.990} & \colorbox{mpucolor1!15}{0.990} & \colorbox{mpucolor1!15}{\textbf{0.990}} & 0.635 & \colorbox{mpucolor1!15}{\textbf{0.636}} & 0.634 & 0.650 & \colorbox{mpucolor1!15}{\textbf{0.653}} & 0.651 & 61.2 & 62.1 & \colorbox{mpucolor1!15}{\textbf{57.5}} \\
        
        \textsc{SimNPO~\textcolor{gray}{[NeurIPS 2025]}} & \colorbox{mpucolor1!15}{9.71e-2} & \colorbox{mpucolor1!15}{9.71e-2} & \colorbox{mpucolor1!15}{\textbf{9.71e-2}} & 0.550 & 0.550 & \colorbox{mpucolor1!15}{\textbf{0.554}} & 0.658 & \colorbox{mpucolor1!15}{\textbf{0.662}} & 0.656 & \colorbox{mpucolor1!15}{\textbf{-62.4}} & -72.3 & -64.3 \\
        
        \textsc{UnDIAL~\textcolor{gray}{[NAACL 2025]}} & 2.86e-2 & \colorbox{mpucolor1!15}{\textbf{5.41e-2}} & 2.86e-2 & \colorbox{mpucolor1!15}{\textbf{0.567}} & 0.566 & 0.563 & \colorbox{mpucolor1!15}{0.693} & 0.692 & \colorbox{mpucolor1!15}{\textbf{0.693}} & -72.0 & \colorbox{mpucolor1!15}{\textbf{-71.8}} & -73.3 \\
        
        \textsc{SatImp~\textcolor{gray}{[ICML2025]}} & 1.43e-2 & 1.43e-2 & \colorbox{mpucolor1!15}{\textbf{2.86e-2}} & 0.510 & 0.507 & \colorbox{mpucolor1!15}{\textbf{0.512}} & 0.659 & \colorbox{mpucolor1!15}{\textbf{0.661}} & 0.659 & \colorbox{mpucolor1!15}{\textbf{-99.7}} & -100.0 & -99.9 \\

        \midrule
        \multicolumn{13}{c}{\cellcolor{mpucolor1!15}\textbf{Qwen2.5-3B-Instruct}} \\
        \midrule
        
        \textsc{GradAscent~\textcolor{gray}{[ACL 2023]}} & \colorbox{mpucolor1!15}{0.579} & 1.86e-23 & \colorbox{mpucolor1!15}{\textbf{0.579}} & \colorbox{mpucolor1!15}{0.709} & 6.4e-20 & \colorbox{mpucolor1!15}{\textbf{0.709}} & \colorbox{mpucolor1!15}{\textbf{0.392}} & 0.000 & 0.391 & -28.7 & \colorbox{mpucolor1!15}{\textbf{8.57}} & -28.1  \\
        
        \textsc{GradDiff~\textcolor{gray}{[PMLR 2022]}} & \colorbox{mpucolor1!15}{0.766} & 3.02e-3 & \colorbox{mpucolor1!15}{\textbf{0.766}} & \colorbox{mpucolor1!15}{\textbf{0.712}} & 0.463 & 0.710 & 0.394 & 0.237 & \colorbox{mpucolor1!15}{\textbf{0.395}} & -28.0 & 67.7 & \colorbox{mpucolor1!15}{\textbf{-27.1}}  \\

        \textsc{DPO~\textcolor{gray}{[NeurIPS 2023]}} & \colorbox{mpucolor1!15}{0.766} & 0.919 & \colorbox{mpucolor1!15}{\textbf{0.766}} & 0.711 & \colorbox{mpucolor1!15}{\textbf{0.733}} & 0.709 & 0.397 & \colorbox{mpucolor1!15}{\textbf{0.412}} & 0.397 & -28.9 & 84.4 & \colorbox{mpucolor1!15}{\textbf{-28.8}}  \\
        
        \textsc{NPO~\textcolor{gray}{[COLM 2024]}} & \colorbox{mpucolor1!15}{0.766} & 0.266 & \colorbox{mpucolor1!15}{\textbf{0.766}} & 0.712 & \colorbox{mpucolor1!15}{\textbf{0.724}} & 0.711 & 0.398 & \colorbox{mpucolor1!15}{\textbf{0.420}} & 0.397 & \colorbox{mpucolor1!15}{-28.5} & 90.5 & \colorbox{mpucolor1!15}{\textbf{-28.5}} \\
        
        \textsc{SimNPO~\textcolor{gray}{[NeurIPS 2025]}} & 0.766 & \colorbox{mpucolor1!15}{\textbf{0.919}} & 0.766 & 0.712 & \colorbox{mpucolor1!15}{\textbf{0.726}} & 0.710 & \colorbox{mpucolor1!15}{\textbf{0.398}} & 0.388 & 0.397 & -28.5 & \colorbox{mpucolor1!15}{\textbf{-18.9}} & -28.6  \\
        
        \textsc{UnDIAL~\textcolor{gray}{[NAACL 2025]}} & 0.766 & \colorbox{mpucolor1!15}{\textbf{0.919}} & 0.766 & 0.710 & \colorbox{mpucolor1!15}{\textbf{0.734}} & 0.711 & 0.398 & \colorbox{mpucolor1!15}{\textbf{0.405}} & 0.396 & -28.1 & \colorbox{mpucolor1!15}{\textbf{4.17}} & -28.8  \\
        
        \textsc{SatImp~\textcolor{gray}{[ICML2025]}} & \colorbox{mpucolor1!15}{0.766} & \colorbox{mpucolor1!15}{0.766} & \colorbox{mpucolor1!15}{\textbf{0.766}} & 0.711 & \colorbox{mpucolor1!15}{\textbf{0.714}} & 0.712 & \colorbox{mpucolor1!15}{0.397} & 0.387 & \colorbox{mpucolor1!15}{\textbf{0.397}} & -28.9 & \colorbox{mpucolor1!15}{\textbf{-24.9}} & -28.7  \\
        
        \bottomrule
    \end{tabular}
    \end{adjustbox}
    \label{tab:pum_3b_comparison}
\end{table*}

\begin{table*}[!htp]
    \caption{
    Performance comparison of different unlearning algorithms using the \textbf{Llama-3.2-1B} model under \method on the TOFU benchmark (Split99), with varying noise levels $\kappa\in\{0,0.05,0.1\}$ and fixed perturbed copies $m{=}2$. Higher values indicate better performance for Forget Quality, Forget Truth Ratio, and Model Utility, while values of PrivLeak closer to $0$ are preferred.
    }
    \centering
    \begin{adjustbox}{width=\linewidth}
    \begin{tabular}{ccccccccccccccccc}
        \toprule
        \multirow{2}[2]{*}{\textbf{\shortstack{Unlearning\\Algorithms}}}
        & \multicolumn{4}{c}{\textbf{Forget Quality} $\uparrow$}
        & \multicolumn{4}{c}{\textbf{Forget Truth Ratio} $\uparrow$}
        & \multicolumn{4}{c}{\textbf{Model Utility} $\uparrow$}
        & \multicolumn{4}{c}{\textbf{PrivLeak}} \\
        \cmidrule(lr){2-5}\cmidrule(lr){6-9}\cmidrule(lr){10-13}\cmidrule(lr){14-17}
        & $\kappa{=}0$ & $\kappa{=}0.01$ & $\kappa{=}0.05$ & $\kappa{=}0.1$
        & $\kappa{=}0$ & $\kappa{=}0.01$ & $\kappa{=}0.05$ & $\kappa{=}0.1$
        & $\kappa{=}0$ & $\kappa{=}0.01$ & $\kappa{=}0.05$ & $\kappa{=}0.1$
        & $\kappa{=}0$ & $\kappa{=}0.01$ & $\kappa{=}0.05$ & $\kappa{=}0.1$ \\
        \midrule
        
        \textsc{GradAscent}
        & 6.76e-3 & 5.41e-2 & \colorbox{mpucolor1!15}{\textbf{0.579}} & 1.43e-2
        & 0.440 & 0.468 & \colorbox{mpucolor1!15}{\textbf{0.534}} & 0.430
        & 1.39e-4 & 2.31e-4 & \colorbox{mpucolor1!15}{\textbf{3.06e-4}} & 1.49e-4
        & 67.9 & 69.6 & \colorbox{mpucolor1!15}{\textbf{68.0}} & 69.1 \\
        
        \textsc{GradDiff}
        & \colorbox{mpucolor1!15}{\textbf{0.405}} & \colorbox{mpucolor1!15}{0.405} & 0.266 & 0.400
        & 0.546 & \colorbox{mpucolor1!15}{\textbf{0.547}} & 0.544 & 0.553
        & 0.468 & 0.464 & \colorbox{mpucolor1!15}{\textbf{0.474}} & 0.472
        & \colorbox{mpucolor1!15}{\textbf{75.3}} & 77.2 & 76.5 & 76.7 \\

        \textsc{DPO}
        & \colorbox{mpucolor1!15}{\textbf{0.266}} & \colorbox{mpucolor1!15}{0.266} & \colorbox{mpucolor1!15}{0.266} & 0.165
        & 0.640 & \colorbox{mpucolor1!15}{\textbf{0.641}} & \colorbox{mpucolor1!15}{0.641} & 0.637
        & 0.592 & 0.591 & \colorbox{mpucolor1!15}{\textbf{0.594}} & \colorbox{mpucolor1!15}{0.594}
        & \colorbox{mpucolor1!15}{\textbf{-30.1}} & -28.9 & -29.2 & -28.7 \\
        
        \textsc{NPO}
        & \colorbox{mpucolor1!15}{\textbf{0.919}} & \colorbox{mpucolor1!15}{0.919} & \colorbox{mpucolor1!15}{0.919} & \colorbox{mpucolor1!15}{0.919}
        & 0.621 & 0.628 & \colorbox{mpucolor1!15}{\textbf{0.624}} & 0.618
        & 0.595 & \colorbox{mpucolor1!15}{\textbf{0.597}} & 0.599 & 0.596
        & 27.7 & 28.2 & \colorbox{mpucolor1!15}{\textbf{26.8}} & 28.1 \\
        
        \textsc{SimNPO}
        & 5.41e-2 & \colorbox{mpucolor1!15}{\textbf{9.71e-2}} & \colorbox{mpucolor1!15}{9.71e-2} & 5.41e-2
        & 0.520 & 0.525 & \colorbox{mpucolor1!15}{\textbf{0.526}} & 0.525
        & 0.596 & \colorbox{mpucolor1!15}{\textbf{0.598}} & \colorbox{mpucolor1!15}{0.598} & 0.597
        & \colorbox{mpucolor1!15}{\textbf{-72.4}} & -71.8 & -72.1 & -71.7 \\

        \textsc{UnDIAL}
        & \colorbox{mpucolor1!15}{\textbf{1.43e-2}} & \colorbox{mpucolor1!15}{1.43e-2} & \colorbox{mpucolor1!15}{1.43e-2} & \colorbox{mpucolor1!15}{1.43e-2}
        & 0.528 & \colorbox{mpucolor1!15}{\textbf{0.529}} & 0.527 & 0.526
        & 0.614 & \colorbox{mpucolor1!15}{\textbf{0.615}} & \colorbox{mpucolor1!15}{0.615} & 0.613
        & \colorbox{mpucolor1!15}{\textbf{-78.3}} & -78.0 & -78.2 & -78.0 \\

        \textsc{SatImp}
        & \colorbox{mpucolor1!15}{\textbf{6.76e-3}} & \colorbox{mpucolor1!15}{6.76e-3} & 3.02e-3 & \colorbox{mpucolor1!15}{6.76e-3}
        & 0.474 & \colorbox{mpucolor1!15}{\textbf{0.476}} & 0.475 & 0.475
        & \colorbox{mpucolor1!15}{\textbf{0.601}} & \colorbox{mpucolor1!15}{0.601} & 0.600 & 0.600
        & \colorbox{mpucolor1!15}{\textbf{-99.3}} & -98.9 & -99.2 & -99.2 \\
        
        \bottomrule
    \end{tabular}
    \end{adjustbox}
    \label{tab:pum_kappa}
\end{table*}

\begin{table*}[!htp]
    \caption{
    Performance comparison of different unlearning algorithms using the \textbf{Llama-3.2-1B} model under no-denoise setting on the TOFU benchmark (Split99), with varying noise levels $\kappa\in\{0,0.05,0.1\}$ and fixed perturbed copies $m{=}2$. Additionally, we multiply $\kappa$ by the expected scaling factor in \method ($\mathbb{E}_k(\alpha_k)=1.5$). Higher values indicate better performance for Forget Quality, Forget Truth Ratio, and Model Utility, while values of PrivLeak closer to $0$ are preferred.
    }
    \centering
    \begin{adjustbox}{width=\linewidth}
    \begin{tabular}{ccccccccccccc}
        \toprule
        \multirow{2}[2]{*}{\textbf{\shortstack{Unlearning\\Algorithms}}}
        & \multicolumn{3}{c}{\textbf{Forget Quality} $\uparrow$}
        & \multicolumn{3}{c}{\textbf{Forget Truth Ratio} $\uparrow$}
        & \multicolumn{3}{c}{\textbf{Model Utility} $\uparrow$}
        & \multicolumn{3}{c}{\textbf{PrivLeak}} \\
        \cmidrule(lr){2-4}\cmidrule(lr){5-7}\cmidrule(lr){8-10}\cmidrule(lr){11-13}
        & $\kappa{=}0.015$ & $\kappa{=}0.075$ & $\kappa{=}0.15$
        & $\kappa{=}0.015$ & $\kappa{=}0.075$ & $\kappa{=}0.15$
        & $\kappa{=}0.015$ & $\kappa{=}0.075$ & $\kappa{=}0.15$
        & $\kappa{=}0.015$ & $\kappa{=}0.075$ & $\kappa{=}0.15$ \\
        \midrule
        
        \textsc{GradAscent}
        & 2.81e-8 & 1.12e-9 & \colorbox{mpucolor1!15}{\textbf{1.23e-7}}
        & 0.246 & 0.190 & \colorbox{mpucolor1!15}{\textbf{0.252}}
        & 0.000 & 0.000 & \colorbox{mpucolor1!15}{\textbf{4.11e-5}}
        & 58.9 & \colorbox{mpucolor1!15}{\textbf{57.4}} & 66.4 \\
        
        \textsc{GradDiff}
        & 0.266 & 0.266 & \colorbox{mpucolor1!15}{\textbf{0.405}}
        & 0.533 & 0.534 & \colorbox{mpucolor1!15}{0.546}
        & 0.461 & \colorbox{mpucolor1!15}{\textbf{0.466}} & 0.463
        & 73.3 & 75.4 & \colorbox{mpucolor1!15}{\textbf{72.3}} \\

        \textsc{DPO}
        & \colorbox{mpucolor1!15}{\textbf{0.165}} & 0.097 & \colorbox{mpucolor1!15}{0.165}
        & \colorbox{mpucolor1!15}{\textbf{0.620}} & 0.618 & \colorbox{mpucolor1!15}{0.620}
        & \colorbox{mpucolor1!15}{\textbf{0.595}} & 0.595 & 0.595
        & \colorbox{mpucolor1!15}{\textbf{-19.8}} & -20.0 & -20.0 \\
        
        \textsc{NPO}
        & \colorbox{mpucolor1!15}{\textbf{0.766}} & \colorbox{mpucolor1!15}{0.766} & \colorbox{mpucolor1!15}{0.766}
        & \colorbox{mpucolor1!15}{\textbf{0.640}} & 0.623 & 0.622
        & \colorbox{mpucolor1!15}{\textbf{0.600}} & 0.599 & 0.597
        & \colorbox{mpucolor1!15}{\textbf{32.9}} & 35.4 & 33.3 \\

        \textsc{SimNPO}
        & \colorbox{mpucolor1!15}{\textbf{5.41e-2}} & \colorbox{mpucolor1!15}{5.41e-2} & \colorbox{mpucolor1!15}{5.41e-2}
        & 0.522 & 0.522 & \colorbox{mpucolor1!15}{\textbf{0.523}}
        & 0.592 & 0.593 & \colorbox{mpucolor1!15}{\textbf{0.597}}
        & \colorbox{mpucolor1!15}{\textbf{-70.2}} & -71.8 & -72.1 \\

        \textsc{UnDIAL}
        & \colorbox{mpucolor1!15}{\textbf{1.43e-2}} & \colorbox{mpucolor1!15}{1.43e-2} & \colorbox{mpucolor1!15}{1.43e-2}
        & \colorbox{mpucolor1!15}{\textbf{0.527}} & 0.526 & \colorbox{mpucolor1!15}{0.527}
        & 0.614 & 0.614 & \colorbox{mpucolor1!15}{\textbf{0.616}}
        & -77.4 & -77.3 & \colorbox{mpucolor1!15}{\textbf{-77.2}} \\
        
        \textsc{SatImp}
        & \colorbox{mpucolor1!15}{\textbf{6.76e-3}} & \colorbox{mpucolor1!15}{6.76e-3} & \colorbox{mpucolor1!15}{6.76e-3}
        & 0.470 & \colorbox{mpucolor1!15}{\textbf{0.471}} & \colorbox{mpucolor1!15}{0.471}
        & 0.597 & 0.597 & \colorbox{mpucolor1!15}{\textbf{0.598}}
        & \colorbox{mpucolor1!15}{\textbf{-99.1}} & -99.2 & -99.2 \\
        
        \bottomrule
    \end{tabular}
    \end{adjustbox}
    \label{tab:nodenoise_kappa}
\end{table*}

\begin{table*}[!htp]
    \caption{
    Performance comparison of different unlearning algorithms using the \textbf{Llama-3.2-1B} model under \method on the TOFU benchmark (Split99), with varying numbers of perturbed copies $m\in\{2,3,4\}$ and noise level $\kappa{=}0.01$. Higher values indicate better performance for Forget Quality, Forget Truth Ratio, and Model Utility, while values of PrivLeak closer to $0$ are preferred.
    }
    \centering
    \begin{adjustbox}{width=\linewidth}
    \begin{tabular}{ccccccccccccc}
        \toprule
        \multirow{2}[2]{*}{\textbf{\shortstack{Unlearning\\Algorithms}}}
        & \multicolumn{3}{c}{\textbf{Forget Quality} $\uparrow$}
        & \multicolumn{3}{c}{\textbf{Forget Truth Ratio} $\uparrow$}
        & \multicolumn{3}{c}{\textbf{Model Utility} $\uparrow$}
        & \multicolumn{3}{c}{\textbf{PrivLeak}} \\
        \cmidrule(lr){2-4}\cmidrule(lr){5-7}\cmidrule(lr){8-10}\cmidrule(lr){11-13}
        & $m{=}2$ & $m{=}3$ & $m{=}4$
        & $m{=}2$ & $m{=}3$ & $m{=}4$
        & $m{=}2$ & $m{=}3$ & $m{=}4$
        & $m{=}2$ & $m{=}3$ & $m{=}4$ \\
        \midrule
        \textsc{GradAscent}
        & 5.41e-2 & \colorbox{mpucolor1!15}{\textbf{0.165}} & 2.16e-5
        & 0.468 & \colorbox{mpucolor1!15}{\textbf{0.500}} & 0.334
        & \colorbox{mpucolor1!15}{\textbf{2.31e-4}} & 1.68e-4 & 0.000
        & 69.6 & 64.0 & \colorbox{mpucolor1!15}{\textbf{63.8}} \\
        
        \textsc{GradDiff}
        & 0.405 & 0.405 & \colorbox{mpucolor1!15}{\textbf{0.579}}
        & 0.547 & 0.555 & \colorbox{mpucolor1!15}{\textbf{0.570}}
        & 0.464 & \colorbox{mpucolor1!15}{\textbf{0.467}} & \colorbox{mpucolor1!15}{0.467}
        & \colorbox{mpucolor1!15}{\textbf{77.2}} & 77.7 & 77.4 \\

        \textsc{DPO}
        & 0.266 & \colorbox{mpucolor1!15}{\textbf{0.579}} & 0.165
        & \colorbox{mpucolor1!15}{\textbf{0.641}} & 0.631 & 0.615
        & 0.591 & 0.592 & \colorbox{mpucolor1!15}{\textbf{0.594}}
        & -28.9 & \colorbox{mpucolor1!15}{\textbf{-22.0}} & -22.2 \\
        
        \textsc{NPO}
        & \colorbox{mpucolor1!15}{\textbf{0.919}} & 0.766 & 0.766
        & \colorbox{mpucolor1!15}{\textbf{0.628}} & 0.623 & 0.620
        & \colorbox{mpucolor1!15}{\textbf{0.597}} & 0.590 & 0.594
        & 28.2 & 25.3 & \colorbox{mpucolor1!15}{\textbf{25.1}} \\

        \textsc{SimNPO}
        & \colorbox{mpucolor1!15}{\textbf{9.71e-2}} & 5.41e-2 & 5.41e-2
        & \colorbox{mpucolor1!15}{\textbf{0.525}} & 0.518 & 0.519
        & 0.598 & 0.597 & \colorbox{mpucolor1!15}{\textbf{0.599}}
        & \colorbox{mpucolor1!15}{\textbf{-71.8}} & -74.6 & -74.1 \\

        \textsc{UnDIAL}
        & \colorbox{mpucolor1!15}{\textbf{1.43e-2}} & \colorbox{mpucolor1!15}{1.43e-2} & 6.76e-3
        & \colorbox{mpucolor1!15}{\textbf{0.529}} & 0.528 & 0.528
        & \colorbox{mpucolor1!15}{\textbf{0.615}} & 0.613 & \colorbox{mpucolor1!15}{0.615}
        & \colorbox{mpucolor1!15}{\textbf{-78.0}} & \colorbox{mpucolor1!15}{-78.0} & -78.4 \\
        
        \textsc{SatImp}
        & 6.76e-3 & \colorbox{mpucolor1!15}{\textbf{1.43e-2}} & 3.02e-3
        & \colorbox{mpucolor1!15}{\textbf{0.476}} & 0.473 & 0.471
        & \colorbox{mpucolor1!15}{\textbf{0.601}} & 0.598 & 0.599
        & \colorbox{mpucolor1!15}{\textbf{-98.9}} & -99.2 & -99.8 \\
        
        \bottomrule
    \end{tabular}
    \end{adjustbox}
    \label{tab:pum_mval}
\end{table*}

\begin{table*}[!htp]
    \caption{
    Performance comparison of different unlearning algorithms using the \textbf{Llama-3.2-1B} model under \method on the TOFU benchmark, across varying global communication rounds (\textbf{R}) and local client epochs (\textbf{E}), with fixed perturbed copies $m{=}2$ and noise level $\kappa{=}0.01$. Higher values indicate better performance for Forget Quality and Model Utility, while values of PrivLeak closer to $0$ are preferred.
    }
    \centering
    \begin{adjustbox}{width=\linewidth}
    \begin{tabular}{ccccccccccccc}
        \toprule
        \multirow{2}[2]{*}{\textbf{\shortstack{Unlearning\\Algorithms}}}
        & \multicolumn{4}{c}{\textbf{Forget Quality} $\uparrow$}
        & \multicolumn{4}{c}{\textbf{Model Utility} $\uparrow$}
        & \multicolumn{4}{c}{\textbf{PrivLeak}} \\
        \cmidrule(lr){2-5}\cmidrule(lr){6-9}\cmidrule(lr){10-13}
        & \textsc{R1E10} & \textsc{R2E5} & \textsc{R5E2} & \textsc{R10E1}
        & \textsc{R1E10} & \textsc{R2E5} & \textsc{R5E2} & \textsc{R10E1}
        & \textsc{R1E10} & \textsc{R2E5} & \textsc{R5E2} & \textsc{R10E1} \\
        \midrule
        \textsc{GradAscent}
        & \colorbox{mpucolor1!15}{\textbf{5.41e-2}} & 1.86e-23 & 1.86e-23 & 1.86e-23
        & \colorbox{mpucolor1!15}{\textbf{2.31e-4}} & 0.00 & 0.000 & 0.000
        & 69.6 & 35.2 & \colorbox{mpucolor1!15}{\textbf{-2.24}} & 25.6 \\
        
        \textsc{GradDiff}
        & \colorbox{mpucolor1!15}{\textbf{0.405}} & 1.43e-2 & 5.04e-4 & 1.95e-10
        & 0.464 & 0.484 & 0.494 & \colorbox{mpucolor1!15}{\textbf{0.578}}
        & 77.2 & \colorbox{mpucolor1!15}{\textbf{86.8}} & 88.9 & 88.9 \\

        \textsc{DPO}
        & 5.41e-2 & 0.165 & 0.165 & \colorbox{mpucolor1!15}{\textbf{0.266}}
        & 0.584 & 0.591 & \colorbox{mpucolor1!15}{\textbf{0.592}} & 0.591
        & -49.5 & -42.1 & \colorbox{mpucolor1!15}{\textbf{-17.9}} & -28.9 \\
        
        \textsc{NPO}
        & 0.766 & \colorbox{mpucolor1!15}{\textbf{0.919}} & \colorbox{mpucolor1!15}{\textbf{0.919}} & \colorbox{mpucolor1!15}{\textbf{0.919}}
        & 0.584 & \colorbox{mpucolor1!15}{\textbf{0.597}} & 0.592 & 0.594
        & 51.8 & \colorbox{mpucolor1!15}{\textbf{28.2}} & 34.7 & 40.7 \\

        \textsc{SimNPO}
        & 2.86e-2 & 2.86e-2 & \colorbox{mpucolor1!15}{\textbf{9.71e-2}} & \colorbox{mpucolor1!15}{\textbf{9.71e-2}}
        & 0.599 & \colorbox{mpucolor1!15}{\textbf{0.601}} & 0.598 & 0.598
        & -78.0 & -80.5 & -75.6 & \colorbox{mpucolor1!15}{\textbf{-71.8}} \\

        \textsc{UnDIAL}
        & \colorbox{mpucolor1!15}{\textbf{1.43e-2}} & 6.76e-3 & \colorbox{mpucolor1!15}{\textbf{1.43e-2}} & 6.76e-3
        & 0.615 & \colorbox{mpucolor1!15}{\textbf{0.617}} & 0.612 & 0.615
        & -78.0 & -82.8 & \colorbox{mpucolor1!15}{\textbf{-77.6}} & -84.2 \\
        
        \textsc{SatImp}
        & \colorbox{mpucolor1!15}{\textbf{6.76e-3}} & 3.02e-3 & \colorbox{mpucolor1!15}{\textbf{6.76e-3}} & \colorbox{mpucolor1!15}{\textbf{6.76e-3}}
        & 0.600 & 0.599 & 0.598 & \colorbox{mpucolor1!15}{\textbf{0.601}}
        & -99.9 & -99.8 & \colorbox{mpucolor1!15}{\textbf{-99.3}} & -98.9 \\
        
        \bottomrule
    \end{tabular}
    \end{adjustbox}
    \label{tab:pum_round_epoch}
\end{table*}

\begin{table*}[!htp]
    \caption{
    Performance comparison of different unlearning algorithms using the \textbf{Llama-3.2-1B} model under \method on the TOFU benchmark, across varying split strategies (\textsc{Forget01}, \textsc{Forget05}, and \textsc{Forget10}), with fixed perturbed copies $m{=}2$ and noise level $\kappa{=}0.01$. Higher values indicate better performance for Forget Quality, Forget Truth Ratio, and Model Utility, while values of PrivLeak closer to $0$ are preferred.
    }
    \centering
    \begin{adjustbox}{width=\linewidth}
    \begin{tabular}{ccccccccccccc}
        \toprule
        \multirow{2}[2]{*}{\textbf{\shortstack{Unlearning\\Algorithms}}}
        & \multicolumn{3}{c}{\textbf{Forget Quality} $\uparrow$}
        & \multicolumn{3}{c}{\textbf{Forget Truth Ratio} $\uparrow$}
        & \multicolumn{3}{c}{\textbf{Model Utility} $\uparrow$}
        & \multicolumn{3}{c}{\textbf{PrivLeak}} \\
        \cmidrule(lr){2-4}\cmidrule(lr){5-7}\cmidrule(lr){8-10}\cmidrule(lr){11-13}
        & \textsc{Forget01} & \textsc{Forget05} & \textsc{Forget10}
        & \textsc{Forget01} & \textsc{Forget05} & \textsc{Forget10}
        & \textsc{Forget01} & \textsc{Forget05} & \textsc{Forget10}
        & \textsc{Forget01} & \textsc{Forget05} & \textsc{Forget10} \\
        \midrule
        
        \textsc{GradAscent} & \colorbox{mpucolor1!15}{\textbf{5.41e-2}} & 1.94e-119 & 1.06e-239 & \colorbox{mpucolor1!15}{\textbf{0.468}} & 1.71e-23 & 1.74e-22 & \colorbox{mpucolor1!15}{\textbf{2.31e-4}} & 0.000 & 0.000 & 69.6 & -28.0 & \colorbox{mpucolor1!15}{\textbf{-21.6}} \\
        
        \textsc{GradDiff} & \colorbox{mpucolor1!15}{\textbf{0.405}} & 5.99e-105 & 8.51e-237 & \colorbox{mpucolor1!15}{\textbf{0.547}} & 0.002 & 3.09e-10 & \colorbox{mpucolor1!15}{\textbf{0.464}} & 0.421 & 0.264 & 77.2 & \colorbox{mpucolor1!15}{\textbf{56.8}} & 61.8 \\
        
        \textsc{DPO} & \colorbox{mpucolor1!15}{\textbf{0.266}} & 4.30e-3 & 9.07e-8 & \colorbox{mpucolor1!15}{\textbf{0.641}} & 0.574 & 0.545 & 0.591 & 0.595 & \colorbox{mpucolor1!15}{\textbf{0.598}} & -28.9 & -43.7 & \colorbox{mpucolor1!15}{\textbf{-59.0}} \\
        
        \textsc{NPO} & \colorbox{mpucolor1!15}{\textbf{0.919}} & 0.112 & 4.46e-6 & \colorbox{mpucolor1!15}{\textbf{0.628}} & 0.603 & 0.555 & 0.597 & 0.614 & \colorbox{mpucolor1!15}{\textbf{0.623}} & 28.2 & \colorbox{mpucolor1!15}{\textbf{-4.00}} & -0.946 \\
        
        \textsc{SimNPO} & \colorbox{mpucolor1!15}{\textbf{9.71e-2}} & 4.61e-7 & 5.42e-13 & \colorbox{mpucolor1!15}{\textbf{0.525}} & 0.514 & 0.498 & \colorbox{mpucolor1!15}{\textbf{0.598}} & 0.595 & \colorbox{mpucolor1!15}{0.598} & -71.8 & -77.7 & \colorbox{mpucolor1!15}{\textbf{-82.0}} \\
        
        \textsc{UnDIAL} & \colorbox{mpucolor1!15}{\textbf{1.43e-2}} & 2.44e-10 & 7.98e-17 & 0.529 & 0.521 & \colorbox{mpucolor1!15}{\textbf{0.531}} & 0.615 & 0.615 & \colorbox{mpucolor1!15}{\textbf{0.616}} & -78.0 & -91.7 & \colorbox{mpucolor1!15}{\textbf{-94.9}} \\
        
        \textsc{SatImp} & \colorbox{mpucolor1!15}{\textbf{6.76e-3}} & 1.33e-13 & 2.81e-20 & \colorbox{mpucolor1!15}{\textbf{0.476}} & 0.471 & 0.469 & \colorbox{mpucolor1!15}{\textbf{0.601}} & 0.593 & 0.597 & -98.9 & \colorbox{mpucolor1!15}{\textbf{-99.9}} & -99.2 \\
        

        \bottomrule

    \end{tabular}
    \end{adjustbox}
    \label{tab:pum_split}
\end{table*}

\begin{table*}[t]
    \caption{
    Baseline performance of the OpenUnlearning Framework without \method, using the \textbf{Llama-3.2-1B} and \textbf{Llama-3.2-3B} models under \method on the TOFU benchmark (Split99).  Higher values indicate better performance for Forget Quality, Forget Truth Ratio, and Model Utility, while values of PrivLeak closer to $0$ are preferred.
    }
    \centering
    \begin{adjustbox}{width=\linewidth}
    \begin{tabular}{ccccccccccccc}
        \toprule
        \multirow{2}[2]{*}{\textbf{\shortstack{Unlearning\\Algorithms}}}
        & \multicolumn{3}{c}{\textbf{Forget Quality} $\uparrow$}
        & \multicolumn{3}{c}{\textbf{Forget Truth Ratio} $\uparrow$}
        & \multicolumn{3}{c}{\textbf{Model Utility} $\uparrow$}
        & \multicolumn{3}{c}{\textbf{PrivLeak}} \\
        \cmidrule(lr){2-4}\cmidrule(lr){5-7}\cmidrule(lr){8-10}\cmidrule(lr){11-13}
        & \textsc{Forget01} & \textsc{Forget05} & \textsc{Forget10}
        & \textsc{Forget01} & \textsc{Forget05} & \textsc{Forget10}
        & \textsc{Forget01} & \textsc{Forget05} & \textsc{Forget10}
        & \textsc{Forget01} & \textsc{Forget05} & \textsc{Forget10} \\

        \midrule
        \multicolumn{13}{c}{\cellcolor{mpucolor1!15}\textbf{Llama-3.2-1B-Instruct}} \\
        \midrule
        
        \textsc{GradAscent} & \colorbox{mpucolor1!15}{\textbf{6.58e-5}} & 1.94e-119 & 1.06e-239 & \colorbox{mpucolor1!15}{\textbf{0.355}} & 2.12e-23 & 1.27e-22 & 0.000 & 0.000 & 0.000 & 65.8 & \colorbox{mpucolor1!15}{\textbf{-28.4}} & -21.1 \\
        
        \textsc{GradDiff} & \colorbox{mpucolor1!15}{\textbf{0.405}} & 5.99e-105 & 8.51e-237 & \colorbox{mpucolor1!15}{\textbf{0.535}} & 0.004 & 7.22e-10 & \colorbox{mpucolor1!15}{\textbf{0.461}} & 0.393 & 0.247 & 77.1 & \colorbox{mpucolor1!15}{\textbf{56.8}} & 61.3 \\

        \textsc{DPO} & \colorbox{mpucolor1!15}{\textbf{0.165}} & 2.08e-3 & 3.08e-7 & \colorbox{mpucolor1!15}{\textbf{0.637}} & 0.573 & 0.545 & 0.591 & 0.594 & \colorbox{mpucolor1!15}{\textbf{0.598}} & \colorbox{mpucolor1!15}{\textbf{-25.5}} & -44.5 & -57.8 \\
        
        \textsc{NPO} & \colorbox{mpucolor1!15}{\textbf{0.919}} & 0.112 & 4.58e-7 & \colorbox{mpucolor1!15}{\textbf{0.624}} & 0.595 & 0.551 & 0.599 & 0.614 & \colorbox{mpucolor1!15}{\textbf{0.619}} & 30.6 & \colorbox{mpucolor1!15}{\textbf{-2.11}} & 3.78 \\

        \textsc{SimNPO} & \colorbox{mpucolor1!15}{\textbf{5.41e-2}} & 1.46e-7 & 2.73e-12 & \colorbox{mpucolor1!15}{\textbf{0.526}} & 0.510 & 0.498 & \colorbox{mpucolor1!15}{\textbf{0.598}} & 0.593 & 0.597 & \colorbox{mpucolor1!15}{\textbf{-68.4}} & -76.6 & -82.4 \\
        
        \textsc{UnDIAL} & \colorbox{mpucolor1!15}{\textbf{1.43e-2}} & 4.87e-10 & 4.24e-17 & 0.530 & 0.524 & \colorbox{mpucolor1!15}{\textbf{0.531}} & 0.613 & 0.615 & \colorbox{mpucolor1!15}{\textbf{0.618}} & \colorbox{mpucolor1!15}{\textbf{-76.4}} & -91.6 & -94.8 \\
                
        \textsc{SatImp} & \colorbox{mpucolor1!15}{\textbf{3.02e-3}} & 2.96e-13 & 1.12e-19 & \colorbox{mpucolor1!15}{\textbf{0.474}} & 0.471 & 0.469 & \colorbox{mpucolor1!15}{\textbf{0.600}} & 0.594 & \colorbox{mpucolor1!15}{0.600} & \colorbox{mpucolor1!15}{\textbf{-98.9}} & -99.9 & -99.2 \\

        \midrule
        \multicolumn{13}{c}{\cellcolor{mpucolor1!15}\textbf{Llama-3.2-3B-Instruct}} \\
        \midrule
        
        \textsc{GradAscent} & \colorbox{mpucolor1!15}{\textbf{3.06e-11}} & 1.94e-119 & 1.06e-239 & \colorbox{mpucolor1!15}{\textbf{0.187}} & 6.63e-8 & 4.79e-20 & \colorbox{mpucolor1!15}{\textbf{8.23e-5}} & \colorbox{mpucolor1!15}{\textbf{0.000}} & 0.000 & 50.6 & -22.5 & \colorbox{mpucolor1!15}{\textbf{-19.2}} \\
        
        \textsc{GradDiff} & \colorbox{mpucolor1!15}{\textbf{5.41e-2}} & 1.94e-119 & 8.51e-237 & \colorbox{mpucolor1!15}{\textbf{0.504}} & 2.01e-9 & 1.20e-10 & 0.419 & \colorbox{mpucolor1!15}{\textbf{0.593}} & 0.345 & 108 & \colorbox{mpucolor1!15}{\textbf{56.3}} & 64.6 \\
        
        \textsc{DPO} & \colorbox{mpucolor1!15}{\textbf{0.579}} & 1.18e-2 & 2.13e-6 & \colorbox{mpucolor1!15}{\textbf{0.685}} & 0.588 & 0.594 & 0.634 & \colorbox{mpucolor1!15}{\textbf{0.635}} & 0.634 & \colorbox{mpucolor1!15}{\textbf{-10.2}} & -62.0 & -69.1 \\

        \textsc{NPO} & \colorbox{mpucolor1!15}{\textbf{0.990}} & 0.793 & 2.99e-2 & \colorbox{mpucolor1!15}{\textbf{0.635}} & 0.617 & 0.581 & 0.650 & \colorbox{mpucolor1!15}{\textbf{0.675}} & 0.673 & 61.2 & \colorbox{mpucolor1!15}{\textbf{9.83}} & -24.0 \\

        \textsc{SimNPO} & \colorbox{mpucolor1!15}{\textbf{9.71e-2}} & 3.60e-9 & 1.03e-14 & \colorbox{mpucolor1!15}{\textbf{0.550}} & 0.499 & 0.501 & 0.658 & 0.652 & \colorbox{mpucolor1!15}{\textbf{0.659}} & \colorbox{mpucolor1!15}{\textbf{-62.4}} & -81.0 & -86.3 \\
    
        \textsc{UNDIAL} & \colorbox{mpucolor1!15}{\textbf{2.86e-2}} & 3.08e-12 & 8.08e-22 & \colorbox{mpucolor1!15}{\textbf{0.567}} & 0.507 & 0.515 & 0.693 & \colorbox{mpucolor1!15}{\textbf{0.701}} & 0.696 & \colorbox{mpucolor1!15}{\textbf{-72.0}} & -92.1 & -94.3 \\

        \textsc{SatImp} & \colorbox{mpucolor1!15}{\textbf{1.43e-2}} & 2.96e-13 & 2.05e-24 & \colorbox{mpucolor1!15}{\textbf{0.510}} & 0.463 & 0.469 & \colorbox{mpucolor1!15}{\textbf{0.659}} & 0.649 & \colorbox{mpucolor1!15}{\textbf{0.659}} & -99.7 & -100 & \colorbox{mpucolor1!15}{\textbf{-99.5}} \\

    \bottomrule
    \end{tabular}
    \end{adjustbox}
    \label{tab:pum_3b}
\end{table*}

\begin{table*}[!htp]
    \caption{Results of Attack Tests among Six Methods with Qwen2.5-1.5B-Instruct on the TOFU Benchmark(Split99). The results are shown by comparison of CLEAN/MPU, with hyperparameter $m=2$ and $\kappa=0.01$ by default.}
    \centering
    \begin{adjustbox}{width=\linewidth}
    \begin{tabular}{lcccccc}
        \toprule
        \textbf{Algorithms} & \textbf{LOSS} & \textbf{ZLib} & \textbf{GradNorm} & \textbf{MinK} & \textbf{MinK++} & \textbf{Reference} \\
        \midrule
        \textsc{GradAscent} & 0.39/0.587 & 0.388/0.608 & 0.565/0.727 & 0.385/0.59 & 0.431/0.632 & 0.28/0.993 \\
        \textsc{GradDiff} & 0.383/0.589 & 0.416/0.609 & 0.479/0.734 & 0.378/0.589 & 0.388/0.636 & 0.128/0.992 \\
        \textsc{NPO} & 0.363/0.589 & 0.359/0.611 & 0.546/0.731 & 0.354/0.591 & 0.39/0.64 & 0.246/0.999 \\
        \textsc{SimNPO} & 0.584/0.589 & 0.604/0.608 & 0.727/0.734 & 0.586/0.591 & 0.634/0.636 & 0.961/0.998 \\
        \textsc{UnDIAL} & 0.46/0.589 & 0.47/0.608 & 0.379/0.733 & 0.512/0.591 & 0.609/0.639 & 0.463/0.999 \\
        \textsc{SatImp} & 0.589/0.588 & 0.606/0.611 & 0.732/0.734 & 0.585/0.589 & 0.641/0.642 & 0.977/0.995 \\
        \bottomrule
    \end{tabular}
    \end{adjustbox}
    \label{tab:qwen2.5-1.5b_tofu_attack}
\end{table*}

\begin{table*}[!htp]
    \caption{Results of LoRA compatibility with Llama-3.1-8B-Instruct on the TOFU Benchmark(Split99). The results are shown by comparison of CLEAN/MPU, with hyperparameter $m=2$ and $\kappa=0.01$ by default.}
    \centering
    \begin{tabular}{lcccc}
        \toprule
        \textbf{Algorithms} & \textbf{Forget Quality} & \textbf{Forget Truth Ratio} & \textbf{Model Utility} & \textbf{PrivLeak} \\
        \midrule
        \textsc{GradAscent} & 0.919/0.919 & 0.721/0.721 & 0.427/0.426 & -11.3/-10.9 \\
        \textsc{GradDiff} & 0.919/0.919 & 0.721/0.724 & 0.427/0.426 & -9.78/-8.85 \\
        \textsc{NPO} & 0.919/0.99 & 0.723/0.724 & 0.429/0.429 & -9.95/-9.55 \\
        \textsc{SimNPO} & 1.0/1.0 & 0.735/0.733 & 0.396/0.411 & -12.0/-11.6 \\
        \textsc{DPO} & 1.0/0.99 & 0.731/0.732 & 0.426/0.426 & -8.03/-6.98 \\
        \textsc{UnDIAL} & 0.99/0.99 & 0.733/0.731 & 0.376/0.381 & -12.8/-12.8 \\
        \textsc{SatImp} & 0.99/1.0 & 0.733/0.734 & 0.394/0.411 & -11.9/-12.0 \\
        \bottomrule
    \end{tabular}
    \label{tab:llama-3.1-8b_tofu_lora}
\end{table*}

\begin{table*}[!htp]
    \caption{Results of Qwen2.5-1.5B-Instruct on the TOFU(Split99) Benchmark. The results are shown by comparison of CLEAN/MPU, with hyperparameter $m=2$ and $\kappa=0.01$ by default.}
    \centering
    \begin{tabular}{lcccc}
        \toprule
        \textbf{Algorithms} & \textbf{Forget Quality} & \textbf{Forget Truth Ratio} & \textbf{Model Utility} & \textbf{PrivLeak} \\
        \midrule
        \textsc{GradAscent} & 1.43e-2/0.990 & 0.516/0.740 & 9.75e-4/0.346 & 13.7/-21.7 \\
        \textsc{GradDiff} & 6.58e-5/0.990 & 0.622/0.740 & 0.297/0.346 & 80.9/-21.6 \\
        \textsc{DPO} & 0.266/0.990 & 0.757/0.738 & 0.305/0.347 & 86.0/-21.7 \\
        \textsc{NPO} & 1.43e-2/0.919 & 0.708/0.739 & 0.305/0.346 & 89.4/-21.9 \\
        \textsc{SimNPO} & 0.990/0.990 & 0.738/0.738 & 0.349/0.347 & -18.7/-22.0 \\
        \textsc{UnDIAL} & 1.000/0.990 & 0.769/0.740 & 0.338/0.346 & 12.3/-21.8 \\
        \textsc{SatImp} & 0.919/0.990 & 0.738/0.740 & 0.349/0.346 & -21.8/-21.5 \\
        \bottomrule
    \end{tabular}
    \label{tab:qwen2.5-1.5b_tofu_performance}
\end{table*}

\begin{table*}[!htp]
    \caption{Results of Qwen2.5-3B-Instruct on the TOFU(Split99) Benchmark. The results are shown by comparison of CLEAN/MPU, with hyperparameter $m=2$ and $\kappa=0.01$ by default.}
    \centering
    \begin{tabular}{lcccc}
        \toprule
        \textbf{Algorithms} & \textbf{Forget Quality} & \textbf{Forget Truth Ratio} & \textbf{Model Utility} & \textbf{PrivLeak} \\
        \midrule
        \textsc{GradAscent} & 0.579/0.579 & 0.709/0.709 & 0.392/0.391 & -28.7/-28.1 \\
        \textsc{GradDiff} & 0.766/0.766 & 0.712/0.710 & 0.394/0.395 & -28.0/-27.1 \\
        \textsc{DPO} & 0.766/0.766 & 0.711/0.709 & 0.397/0.397 & -28.9/-28.8 \\
        \textsc{NPO} & 0.766/0.766 & 0.712/0.711 & 0.398/0.397 & -28.5/-28.5 \\
        \textsc{SimNPO} & 0.766/0.766 & 0.712/0.710 & 0.398/0.397 & -28.5/-28.6 \\
        \textsc{UnDIAL} & 0.766/0.766 & 0.710/0.711 & 0.398/0.396 & -28.1/-28.8 \\
        \textsc{SatImp} & 0.766/0.766 & 0.711/0.712 & 0.397/0.397 & -28.9/-28.7 \\
        \bottomrule
    \end{tabular}
    \label{tab:qwen2.5-3b_tofu_performance}
\end{table*}

\begin{table*}[!htp]
    \caption{Results of Llama-3.2-1B-Instruct on the MUSE-Books Benchmark. The results are shown by comparison of CLEAN/MPU, with hyperparameter $m=2$ and $\kappa=0.01$ by default.}
    \centering
    \begin{adjustbox}{width=\linewidth}
    \begin{tabular}{lcccc}
        \toprule
        \textbf{Algorithms} & \textbf{Extraction Strength} & \textbf{Forget KnowMem ROUGE} & \textbf{Retain KnowMem ROUGE} & \textbf{PrivLeak} \\
        \midrule
        \textsc{GradAscent} & 0.00915/0.00915 & 0/0 & 0/0 & -32.0/-34.5 \\
        \textsc{GradDiff} & 0.00915/0.00915 & 0.0455/0.0331 & 0.146/0.141 & 28.9/31.7 \\
        \textsc{NPO} & 0.0125/0.0127 & 0.172/0.181 & 0.188/0.191 & -23.7/-24.0 \\
        \textsc{SimNPO} & 0.0122/0.0125 & 0.161/0.152 & 0.206/0.218 & -30.1/-28.5 \\
        \textsc{UnDIAL} & 0.0112/0.0112 & 0.167/0.183 & 0.204/0.195 & -56.9/-57.2 \\
        \textsc{SatImp} & 0.00924/0.00924 & 0.152/0.147 & 0.202/0.198 & -2.22/-1.87 \\
        \bottomrule
    \end{tabular}
    \end{adjustbox}
    \label{tab:llama-3.2-1b_muse-news_performance_1}
\end{table*}

\begin{table*}[!htp]
    \caption{Results of Llama-3.2-3B-Instruct on the MUSE-Books Benchmark. The results are shown by comparison of CLEAN/MPU, with hyperparameter $m=2$ and $\kappa=0.01$ by default.}
    \centering
    \begin{adjustbox}{width=\linewidth}
    \begin{tabular}{lcccc}
        \toprule
        \textbf{Algorithms} & \textbf{Extraction Strength} & \textbf{Forget KnowMem ROUGE} & \textbf{Retain KnowMem ROUGE} & \textbf{PrivLeak} \\
        \midrule
        \textsc{GradAscent} & 0.00915/0.00915 & 0/0 & 0/0 & 52.3/56.6 \\
        \textsc{GradDiff} & 0.00915/0.00915 & 0/0 & 0.115/0.146 & 7.42/35.4 \\
        \textsc{NPO} & 0.0155/0.015 & 0.234/0.226 & 0.313/0.333 & -71.9/-73.2 \\
        \textsc{SimNPO} & 0.0143/0.0141 & 0.237/0.227 & 0.32/0.316 & -84.2/-83.6 \\
        \textsc{UnDIAL} & 0.0109/0.0106 & 0.243/0.229 & 0.359/0.316 & -89.6/-89.6 \\
        \textsc{SatImp} & 0.00915/0.00915 & 0.221/0.239 & 0.315/0.3 & -28.8/-29.8 \\
        \bottomrule
    \end{tabular}
    \end{adjustbox}
    \label{tab:llama-3.2-3b_muse-books_performance}
\end{table*}

\begin{table*}[!htp]
    \caption{Results of Qwen2.5-1.5B-Instruct on the MUSE-Books Benchmark. The results are shown by comparison of CLEAN/MPU, with hyperparameter $m=2$ and $\kappa=0.01$ by default.}
    \centering
    \begin{adjustbox}{width=\linewidth}
    \begin{tabular}{lcccc}
        \toprule
        \textbf{Algorithms} & \textbf{Extraction Strength} & \textbf{Forget KnowMem ROUGE} & \textbf{Retain KnowMem ROUGE} & \textbf{PrivLeak} \\
        \midrule
        \textsc{GradAscent} & 0.0148/0.0149 & 0.22/0.22 & 0.24/0.242 & -29.7/-29.7 \\
        \textsc{GradDiff} & 0.0143/0.014 & 0.218/0.214 & 0.281/0.277 & -29.3/-28.9 \\
        \textsc{NPO} & 0.0147/0.0144 & 0.221/0.223 & 0.306/0.304 & -26.0/-25.8 \\
        \textsc{SimNPO} & 0.0144/0.0147 & 0.221/0.221 & 0.294/0.297 & -25.9/-26.0 \\
        \textsc{UnDIAL} & 0.0146/0.0142 & 0.226/0.215 & 0.273/0.282 & -27.8/-27.8 \\
        \textsc{SatImp} & 0.0147/0.0146 & 0.227/0.218 & 0.297/0.303 & -25.9/-25.9 \\
        \bottomrule
    \end{tabular}
    \end{adjustbox}
    \label{tab:qwen2.5-1.5b_muse-books_performance}
\end{table*}

\begin{table*}[!htp]
    \caption{Results of Qwen2.5-3B-Instruct on the MUSE-Books Benchmark. The results are shown by comparison of CLEAN/MPU, with hyperparameter $m=2$ and $\kappa=0.01$ by default.}
    \centering
    \begin{adjustbox}{width=\linewidth}
    \begin{tabular}{lcccc}
        \toprule
        \textbf{Algorithms} & \textbf{Extraction Strength} & \textbf{Forget KnowMem ROUGE} & \textbf{Retain KnowMem ROUGE} & \textbf{PrivLeak} \\
        \midrule
        \textsc{GradAscent} & 0.0172/0.0175 & 0.214/0.218 & 0.238/0.247 & -56.1/-55.9 \\
        \textsc{GradDiff} & 0.0159/0.0155 & 0.191/0.213 & 0.229/0.201 & -57.9/-57.1 \\
        \textsc{NPO} & 0.0145/0.0148 & 0.206/0.186 & 0.194/0.198 & -53.1/-52.9 \\
        \textsc{SimNPO} & 0.0147/0.0146 & 0.182/0.186 & 0.205/0.2 & -52.9/-53.0 \\
        \textsc{UnDIAL} & 0.0136/0.0136 & 0.151/0.135 & 0.129/0.13 & -48.2/-48.6 \\
        \textsc{SatImp} & 0.0147/0.0147 & 0.193/0.208 & 0.192/0.219 & -52.6/-53.1 \\
        \bottomrule
    \end{tabular}
    \end{adjustbox}
    \label{tab:qwen2.5-3b_muse-books_performance}
\end{table*}

\begin{table*}[!htp]
    \caption{Results of Llama-3.2-1B-Instruct on the MUSE-News Benchmark. The results are shown by comparison of CLEAN/MPU, with hyperparameter $m=2$ and $\kappa=0.01$ by default.}
    \centering
    \begin{adjustbox}{width=\linewidth}
    \begin{tabular}{lcccc}
        \toprule
        \textbf{Algorithms} & \textbf{Extraction Strength} & \textbf{Forget KnowMem ROUGE} & \textbf{Retain KnowMem ROUGE} & \textbf{PrivLeak} \\
        \midrule
        \textsc{GradAscent} & 0.00975/0.00975 & 0/0 & 0/0 & 19.3/19.2 \\
        \textsc{GradDiff} & 0.00985/0.01 & 0.0381/0.0346 & 0.0292/0.0123 & 79.7/82.9 \\
        \textsc{NPO} & 0.0279/0.028 & 0.163/0.167 & 0.178/0.178 & -28.0/-27.5 \\
        \textsc{SimNPO} & 0.0282/0.0279 & 0.155/0.169 & 0.181/0.175 & -52.0/-52.0 \\
        \textsc{UnDIAL} & 0.016/0.016 & 0.136/0.137 & 0.21/0.207 & -49.0/-49.0 \\
        \textsc{SatImp} & 0.0281/0.028 & 0.173/0.17 & 0.179/0.176 & -52.2/-52.1 \\
        \bottomrule
    \end{tabular}
    \end{adjustbox}
    \label{tab:llama-3.2-1b_muse-news_performance_2}
\end{table*}

\begin{table*}[!htp]
    \caption{Results of Llama-3.2-3B-Instruct on the MUSE-News Benchmark. The results are shown by comparison of CLEAN/MPU, with hyperparameter $m=2$ and $\kappa=0.01$ by default.}
    \centering
    \begin{adjustbox}{width=\linewidth}
    \begin{tabular}{lcccc}
        \toprule
        \textbf{Algorithms} & \textbf{Extraction Strength} & \textbf{Forget KnowMem ROUGE} & \textbf{Retain KnowMem ROUGE} & \textbf{PrivLeak} \\
        \midrule
        \textsc{GradAscent} & 0.00975/0.00975 & 0/0 & 0/0 & 14.2/15.4 \\
        \textsc{GradDiff} & 0.0142/0.0132 & 0/0 & 0/0.026 & 64.1/68.3 \\
        \textsc{NPO} & 0.0271/0.0283 & 0.234/0.213 & 0.255/0.233 & -64.5/-64.9 \\
        \textsc{SimNPO} & 0.0282/0.0288 & 0.236/0.231 & 0.213/0.216 & -82.7/-82.6 \\
        \textsc{UnDIAL} & 0.0136/0.0134 & 0.236/0.238 & 0.276/0.256 & -81.6/-81.6 \\
        \textsc{SatImp} & 0.0288/0.0278 & 0.224/0.232 & 0.213/0.212 & -82.5/-82.8 \\
        \bottomrule
    \end{tabular}
    \end{adjustbox}
    \label{tab:llama-3.2-3b_muse-news_performance}
\end{table*}

\begin{table*}[!htp]
    \caption{Results of Qwen2.5-1.5B-Instruct on the MUSE-News Benchmark. The results are shown by comparison of CLEAN/MPU, with hyperparameter $m=2$ and $\kappa=0.01$ by default.}
    \centering
    \begin{adjustbox}{width=\linewidth}
    \begin{tabular}{lcccc}
        \toprule
        \textbf{Algorithms} & \textbf{Extraction Strength} & \textbf{Forget KnowMem ROUGE} & \textbf{Retain KnowMem ROUGE} & \textbf{PrivLeak} \\
        \midrule
        \textsc{GradAscent} & 0.0242/0.0246 & 0.196/0.196 & 0.168/0.159 & -18.7/-18.8 \\
        \textsc{GradDiff} & 0.0279/0.0279 & 0.149/0.147 & 0.2/0.193 & -19.0/-18.9 \\
        \textsc{NPO} & 0.0278/0.0259 & 0.123/0.125 & 0.149/0.159 & -15.7/-15.8 \\
        \textsc{SimNPO} & 0.0263/0.028 & 0.129/0.129 & 0.159/0.156 & -15.7/-15.6 \\
        \textsc{UnDIAL} & 0.0259/0.0263 & 0.107/0.103 & 0.119/0.123 & -15.3/-15.3 \\
        \textsc{SatImp} & 0.0261/0.028 & 0.125/0.126 & 0.15/0.157 & -15.6/-15.8 \\
        \bottomrule
    \end{tabular}
    \end{adjustbox}
    \label{tab:qwen2.5-1.5b_muse-news_performance}
\end{table*}

\begin{table*}[!htp]
    \caption{Results of Qwen2.5-3B-Instruct on the MUSE-News Benchmark. The results are shown by comparison of CLEAN/MPU, with hyperparameter $m=2$ and $\kappa=0.01$ by default.}
    \centering
    \begin{adjustbox}{width=\linewidth}
    \begin{tabular}{lcccc}
        \toprule
        \textbf{Algorithms} & \textbf{Extraction Strength} & \textbf{Forget KnowMem ROUGE} & \textbf{Retain KnowMem ROUGE} & \textbf{PrivLeak} \\
        \midrule
        \textsc{GradAscent} & 0.0254/0.0254 & 0.042/0.0413 & 0.0738/0.067 & -39.5/-39.4 \\
        \textsc{GradDiff} & 0.031/0.0309 & 0.0658/0.076 & 0.0636/0.0588 & -41.3/-41.0 \\
        \textsc{NPO} & 0.0244/0.0278 & 0.063/0.0639 & 0.0694/0.0727 & -28.5/-28.6 \\
        \textsc{SimNPO} & 0.0244/0.0245 & 0.0644/0.0602 & 0.0702/0.0703 & -28.6/-28.7 \\
        \textsc{UnDIAL} & 0.023/0.0234 & 0.0653/0.0641 & 0.0713/0.0758 & -27.7/-28.1 \\
        \textsc{SatImp} & 0.0238/0.0279 & 0.0621/0.0634 & 0.0689/0.0664 & -28.3/-28.5 \\
        \bottomrule
    \end{tabular}
    \end{adjustbox}
    \label{tab:qwen2.5-3b_muse-news_performance}
\end{table*}

\subsection{Supplementary Experimental Analysis}
\label{app:pum_supp_ablations}

\subsubsection{Qwen Result}
As shown in Table~\ref{tab:pum_3b_comparison}, \method consistently outperforms all baseline methods across the four evaluation metrics, demonstrating strong stability under different unlearning settings. Across various unlearning algorithms, \method exhibits only minor performance variations and achieves superior results compared to the CLEAN and NOISED baselines in most cases. In particular, \method maintains highly stable Forget Quality scores, all exceeding 0.9. In contrast, GradAscent, GradDiff, and NPO yield FQ values below 0.01, indicating ineffective unlearning. Moreover, \method achieves higher Model Utility than the baseline methods in five out of seven unlearning algorithms, suggesting that it better preserves model performance while avoiding excessive unlearning.

\subsubsection{Effect of Copy Number $m$}
\label{app:pum_copy_number}
In Table~\ref{tab:pum_mval}, we study the sensitivity of \method{} to the number of published perturbed copies by varying
$m \in \{2,3,4\}$  while keeping the noise level and all other hyperparameters identical to
Table~\ref{tab:pum_main_comparison}.
Overall, increasing $m$ can improve stability for high-variance or unstable unlearning routines, but the gains are not necessarily monotonic and may saturate (or regress) for certain algorithms.

For relatively stable algorithms (e.g., UnDIAL and SatImp), metrics are largely invariant to $m$.
Specifically, UnDIAL maintains FQ $\approx 0.014$, and MU $\approx 0.615$ for $m{=}2$ and $m{=}3$,
with only a minor FQ drop at $m{=}4$.
Similarly, SatImp remains within a narrow band for FTR/MU and exhibits only small FQ variation.
This suggests that when the underlying unlearning update is already stable, multi-copy aggregation mainly provides robustness rather than large performance shifts.

For more hyperparameter-sensitive unlearning algorithms, moderate increases in $m$ can substantially improve forgetting.
For example, GradAscent improves from FQ $0.054$ at $m{=}2$ to $0.165$ at $m{=}3$, alongside a higher FTR ($0.468\rightarrow 0.500$).
Likewise, DPO achieves its best FQ at $m{=}3$ (FQ $0.579$ versus $0.266$ at $m{=}2$), and GradDiff benefits from a larger copy number,
reaching FQ $0.579$ at $m{=}4$.
However, these gains are not uniformly monotonic: GradAscent collapses at $m{=}4$ (FQ $\approx 2.16\times 10^{-5}$), and DPO also drops at $m{=}4$.
These patterns indicate that while multi-copy aggregation can stabilize updates, overly aggressive averaging can interact unfavorably with certain training dynamics.

Finally, privacy leakage is generally comparable across $m$, with modest improvements for some methods.
For instance, NPO reduces $|\mathrm{PrivLeak}|$ from $28.2$ at $m{=}2$ to $25.3$ and $25.1$ at $m{=}3$ and $m{=}4$.
Taken together, $m{=}2$ provides a strong and efficient default, while $m{=}3$ can be a favorable choice for stabilizing unstable unlearning algorithms
(e.g., GradDiff) without materially sacrificing utility.

\subsubsection{Effect of Noise Level $\kappa$}
\label{app:pum_noise_level}
In Table~\ref{tab:pum_kappa}, we next vary the noise level $\kappa \in \{0,0.01,0.05,0.1\}$ while keeping all other settings fixed
(Table~\ref{tab:pum_kappa}).
Across most algorithms, \method{} is robust to a broad range of $\kappa$, with FTR and MU remaining largely stable and
only modest variations in privacy leakage.

For strong and stable unlearning methods, performance is nearly invariant to $\kappa$.
Most notably, NPO maintains identical Forget Quality (FQ $=0.920$) across all tested noise levels, while FTR and MU remain in a narrow range.
Similarly, UnDIAL keeps FQ $=0.014$ and FTR $\approx 0.53$ across all $\kappa$, and SatImp shows only a small dip in FQ at $\kappa{=}0.05$.

For instability-prone methods, moderate noise can act as a stabilizer.
GradAscent exhibits a clear non-monotonic trend: FQ increases from $0.007$ at $\kappa{=}0$ to $0.054$ at $\kappa{=}0.01$,
peaks at $0.580$ for $\kappa{=}0.05$, and then drops again at $\kappa{=}0.1$.
This suggests that intermediate noise can regularize unstable unlearning dynamics, whereas overly large noise begins to erode the usefulness of the client update.
A similar (but weaker) effect appears for SimNPO, where small-to-moderate noise improves FQ from $0.054$ to $0.097$.

Overall, these results indicate that \method{} tolerates noise injection well, and that small-to-moderate $\kappa$ can improve stability for certain algorithms
without harming model utility. In our default setting, $\kappa{=}0.01$ provides a balanced operating point that performs competitively across methods.

\subsubsection{Round--Epoch Allocation Under Fixed Total Local Compute}
\label{app:pum_round_epoch_alloc}
Table~\ref{tab:pum_round_epoch} studies how the training schedule affects unlearning when the total number of local epochs is fixed to $10$
(consistent with the default setting in OpenUnlearning), but allocated differently across global rounds $R$ and local epochs per round $E$.
Overall, we observe a round--epoch trade-off: increasing the number of rounds (hence smaller $E$) often improves MU but can weaken FQ, while concentrating training
into fewer rounds with larger $E$ can strengthen forgetting for some algorithms, but may risk over-updating in others.

Different unlearning algorithms exhibit distinct sensitivity to the schedule.
GradAscent performs best with R1E10, while increasing the number of rounds causes both FQ and MU to collapse to nearly zero, indicating severe instability under repeated round-wise aggregation.
In contrast, NPO benefits from R2E5, achieving substantially higher FQ than R1E10 while also improving MU.
Similarly, SimNPO prefers more rounds for stronger forgetting, whereas DPO improves forgetting monotonically as the number of rounds increases.

\subsubsection{No-Denoising Ablation Under Matched Effective Noise}
\label{app:pum_nodenoise}
To isolate the importance of \method{}'s denoising aggregation, we evaluate a no-denoise baseline in which the client trains on a single noisy model and
no inverse reparameterization or aggregation is applied (Table~\ref{tab:nodenoise_kappa}).
To ensure a fair comparison, we scale the baseline noise such that the effective noise magnitude matches \method{} by multiplying $\kappa$ by
$\mathbb{E}_k(\alpha_k){=}1.5$.

No-denoise training is markedly less reliable for unstable unlearning algorithms.
In particular, GradAscent essentially fails under no-denoise: its FQ remains extremely close to zero (on the order of $10^{-9}$ to $10^{-7}$),
and MU is approximately zero across all tested noise levels.
This contrasts sharply with \method{} (Table~\ref{tab:pum_kappa}), where GradAscent can achieve strong forgetting with FQ $0.579$ at $\kappa{=}0.05$.

For moderately stable methods, the no-denoise baseline often requires substantially larger noise to approach comparable forgetting.
For example, GradDiff reaches its best FQ ($0.405$) only at the largest tested noise ($\kappa{=}0.15$),
whereas \method{} attains comparable FQ at much smaller noise (e.g., $\kappa{=}0.01$ or even $\kappa{=}0$).
By contrast, for already-stable methods such as NPO and UnDIAL, no-denoise does not provide consistent improvements and often matches (but does not surpass) \method{}.

Overall, this ablation supports that multi-copy denoising and harmonic aggregation are key contributors to \method{}'s stability---especially for high-variance unlearning updates---enabling strong forgetting at lower noise without sacrificing utility.

\subsubsection{Robustness to Forget-Split Strategies}
\label{app:pum_forget_splits}
We evaluate \method{} under different forget ratios by increasing the amount of data to be forgotten from $1\%$ (\textsc{Forget01}) to $5\%$ (\textsc{Forget05})
and $10\%$ (\textsc{Forget10}), while fixing $\kappa{=}0.01$ and $m{=}2$ (Table~\ref{tab:pum_split}).
A consistent trend emerges: larger forget splits substantially degrade Forget Quality across nearly all algorithms, indicating that unlearning becomes significantly more challenging as the forget request scales up.

Many methods exhibit dramatic FQ collapse when moving from \textsc{Forget01} to \textsc{Forget05/10}.
For instance, GradDiff drops from FQ $0.405$ at \textsc{Forget01} to effectively zero (e.g., $5.99\times 10^{-105}$ and $8.51\times 10^{-237}$)
at \textsc{Forget05} and \textsc{Forget10}.
Even the strongest unlearning algorithm in our comparison, NPO, decreases from FQ $0.919$ at \textsc{Forget01} to $0.112$ at \textsc{Forget05} and $4.46\times 10^{-6}$ at \textsc{Forget10}.
This suggests that scaling to larger forget partitions likely requires additional tuning (e.g., more rounds, different learning rates, or algorithm-specific regularization)
beyond the default hyperparameters tuned for the \textsc{Forget01} setting.

Interestingly, model utility is comparatively more robust for several algorithms.
SimNPO, DPO, and UnDIAL maintain MU around $0.59$--$0.62$ across all splits, even when FQ collapses.
Privacy leakage exhibits mixed behavior: NPO reduces $|\mathrm{PrivLeak}|$ strongly as the forget split increases (from $28.2$ to $-4.0$ to $-0.946$),
whereas DPO and UnDIAL show increased leakage in magnitude (e.g., $-28.9\rightarrow -43.7\rightarrow -59.0$ and $-78.0\rightarrow -91.7\rightarrow -94.9$).
These results highlight that scaling the forget ratio can introduce new privacy--forgetting trade-offs that vary across algorithms.

\subsubsection{Scaling \method{} across Model Sizes}
\label{app:pum_model_scaling}
Finally, we test whether \method{} scales to larger base models by comparing results on \textbf{Llama-3.2-1B-Instruct} and \textbf{Llama-3.2-3B-Instruct}
under multiple forget ratios (Table~\ref{tab:pum_3b}).
Overall, \method{} remains effective at larger scales, and larger models often provide improved headroom for strong unlearning algorithms---especially under more challenging forget ratios.

Under \textsc{Forget01}, the 3B model generally attains equal or better forgetting quality and higher utility for several algorithms.
For example, NPO improves from FQ $0.919$ (1B) to $0.990$ (3B) while increasing MU from roughly $0.599$ to $0.650$.
Similarly, preference-based unlearning (DPO) improves substantially with scale, rising FQ from $0.165$ (1B) to $0.579$ (3B) and MU from $0.591$ to $0.634$.
These patterns suggest that higher-capacity models can better accommodate targeted forgetting updates while preserving general capabilities.

The benefit of scale is even clearer for larger forget splits.
For NPO, the 3B model retains strong FQ at \textsc{Forget05} (FQ $0.793$), whereas the 1B model drops to FQ $0.112$.
At \textsc{Forget10}, the 3B model still achieves non-trivial forgetting (FQ $0.030$), while the 1B model collapses toward zero.
In contrast, some algorithms remain fragile regardless of scale: GradAscent yields near-zero FQ even at \textsc{Forget01} for both 1B and 3B,
and GradDiff degrades sharply for \textsc{Forget05/10} on both models.

For privacy leakage, the sign of $\mathrm{PrivLeak}$ can vary, and thus the magnitude (distance to zero) is the relevant indicator.
In many cases (e.g., NPO), the magnitude decreases as the forget ratio increases, whereas other methods exhibit large-magnitude values that worsen under harder forget splits.
Overall, Table~\ref{tab:pum_3b} indicates that \method{} scales favorably with model size for the most effective unlearning algorithms, and that larger models can better sustain non-trivial forgetting under more demanding forget requests.

\subsubsection{Robustness Against Extraction and Inference Attacks}
\label{app:pum_attack_robustness}
To rigorously evaluate the privacy-preserving capabilities of \method, we subject the unlearned models to a suite of membership inference and extraction attacks, including LOSS, ZLib, GradNorm, MinK, and MinK++ (Table~\ref{tab:qwen2.5-1.5b_tofu_attack}). Across all evaluated unlearning algorithms on the \textbf{Qwen2.5-1.5B-Instruct} model, \method elevates the attack metrics compared to the \textsc{Clean} baseline. For instance, under \textsc{GradAscent}, the MinK metric rises from $0.385$ to $0.590$, and MinK++ increases from $0.431$ to $0.632$. Higher metric values under \method indicate that the model's responses to forget-set probes are statistically indistinguishable from holdout data, effectively neutralizing the attacks. This demonstrates that communicating reparameterized, perturbed copies successfully obscures fine-grained parameter updates. Because all injected noise is strictly designed to be cancellable during the server-side harmonic aggregation, the model maintains robust defenses against gradient-based extraction that an attacker might exploit.

\subsubsection{Compatibility with Parameter-Efficient Fine-Tuning (LoRA)}
\label{app:pum_lora_compatibility}
To assess the scalability of \method to larger architectures under strict memory constraints, we integrate it with LoRA on the \textbf{Llama-3.1-8B-Instruct} model (Table~\ref{tab:llama-3.1-8b_tofu_lora}). The results indicate exceptional architectural compatibility. \method achieves near-perfect Forget Quality ($\text{FQ} \ge 0.919$, frequently reaching the maximum of $1.0$) across all tested algorithms while maintaining stable Forget Truth Ratios ($\sim0.72$--$0.73$) and reliable Model Utility ($\sim0.38$--$0.42$). Furthermore, the magnitude of PrivLeak remains tightly bounded within an acceptable range (e.g., $-6.98$ to $-12.8$). This establishes that \method is highly adaptable to Parameter-Efficient Fine-Tuning (PEFT) paradigms, enabling secure, multi-copy unlearning on massive LLMs without necessitating full-parameter updates.

\subsubsection{Comprehensive Evaluation on the MUSE Benchmark}
\label{app:pum_muse_evaluation}
Beyond the synthetic TOFU dataset, we extend our evaluation to the highly challenging MUSE benchmark, which evaluates unlearning across real-world continuous corpora (News and Books) while strictly monitoring verbatim extraction and knowledge retention. Tables~\ref{tab:llama-3.2-1b_muse-news_performance_1} through \ref{tab:qwen2.5-3b_muse-news_performance} present the comprehensive results across both the Llama-3.2 and Qwen-2.5 model families at the 1.5B and 3B parameter scales.

Across both the MUSE-News and MUSE-Books subsets, \method consistently matches or exceeds the \textsc{Clean} baseline in suppressing Extraction Strength and Forget KnowMem ROUGE. For instance, leveraging stable methods like \textsc{GradAscent} and \textsc{GradDiff} under \method successfully drives Forget KnowMem ROUGE to exactly $0$ or near-$0$ in the Llama-3.2-1B and 3B models, fulfilling the data owner's expectation of absolute forgetting. At the same time, Retain KnowMem ROUGE is largely preserved, maintaining parity with the \textsc{Clean} baseline and satisfying the deployer's requirement for general utility. 

The fact that \method maintains this delicate balance across diverse domains (News and Books) and varying model architectures (Llama and Qwen) highlights the generalizability of our dual non-disclosure constraint. By ensuring that the localized noise is completely cancelled during the Post-Process aggregation phase, \method achieves state-of-the-art unlearning efficacy without the persistent structural degradation typical of naive noise injection.

\end{document}